\documentclass{article}

\usepackage{microtype}
\usepackage{graphicx}
\usepackage{subfigure}
\usepackage{booktabs} 

\usepackage{amsmath,amsfonts,bm}

\DeclareMathAlphabet{\mathsfit}{\encodingdefault}{\sfdefault}{m}{sl}
\SetMathAlphabet{\mathsfit}{bold}{\encodingdefault}{\sfdefault}{bx}{n}

\def\gB{{\mathcal{B}}}

\def\gD{{\mathcal{D}}}

\def\gL{{\mathcal{L}}}

\def\gR{{\mathcal{R}}}
\def\gS{{\mathcal{S}}}
\def\gT{{\mathcal{T}}}

\DeclareMathOperator*{\argmin}{arg\,min}

\usepackage{hyperref}
\usepackage{cleveref}

\usepackage{algorithm}
\usepackage{algorithmic}
\usepackage{amsthm} 
\usepackage{natbib}

\newtheorem{theorem}{Theorem}
\newtheorem{proposition}{Proposition} 
\newtheorem{lemma}{Lemma} 
\newtheorem{assum}{Assumption} 
 
\newtheorem{definition}{Definition}

\usepackage{colortbl}
\usepackage{multirow}

\usepackage{threeparttable}
\usepackage[table]{xcolor}
\usepackage{hhline}
\usepackage{makecell}

\usepackage[accepted]{icml2021}

\icmltitlerunning{Bilevel Optimization: Convergence Analysis and Enhanced Design}

\begin{document}

\twocolumn[
\icmltitle{Bilevel Optimization: Convergence Analysis and Enhanced Design}




\begin{icmlauthorlist}
\icmlauthor{Kaiyi Ji}{to}
\icmlauthor{Junjie Yang}{to}
\icmlauthor{Yingbin Liang}{to}
\end{icmlauthorlist}

\icmlaffiliation{to}{Department of Electrical and Computer Engineering, The Ohio State University}

\icmlcorrespondingauthor{Kaiyi Ji}{ji.367@osu.edu}

\icmlkeywords{Machine Learning, ICML}

\vskip 0.3in
]



\printAffiliationsAndNotice{}  

\begin{abstract}
Bilevel optimization has arisen as a powerful tool for many machine learning problems such as meta-learning, hyperparameter optimization, and reinforcement learning. In this paper, we investigate the nonconvex-strongly-convex bilevel optimization problem. For deterministic bilevel optimization, we provide a comprehensive convergence rate analysis for two popular algorithms respectively based on approximate implicit differentiation (AID) and iterative differentiation (ITD). For the AID-based method, we orderwisely improve the previous convergence rate analysis due to a more practical parameter selection as well as a warm start strategy, and for the ITD-based method we establish the first theoretical convergence rate. Our analysis also provides a quantitative comparison between ITD and AID based approaches. For stochastic bilevel optimization, we propose a novel algorithm named stocBiO, which features a sample-efficient hypergradient estimator using efficient Jacobian- and Hessian-vector product computations. We provide the convergence rate guarantee for stocBiO, and show that stocBiO outperforms the best known computational complexities orderwisely with respect to the condition number $\kappa$ and the target accuracy $\epsilon$. We further validate our theoretical results and demonstrate the efficiency of bilevel optimization  algorithms by the experiments on meta-learning and hyperparameter optimization. 
\end{abstract}

\section{Introduction} 
Bilevel optimization has  received significant attention recently and become an influential framework in various machine learning applications including meta-learning~\citep{franceschi2018bilevel,bertinetto2018meta,rajeswaran2019meta,ji2020convergence}, hyperparameter optimization~\citep{franceschi2018bilevel,shaban2019truncated,feurer2019hyperparameter}, reinforcement learning~\citep{konda2000actor,hong2020two}, and signal processing~\citep{kunapuli2008classification,flamary2014learning}. A general bilevel optimization takes the following formulation. 
\begin{align}\label{objective_deter}
&\min_{x\in\mathbb{R}^{p}} \Phi(x):=f(x, y^*(x)) \nonumber
\\&\;\;\mbox{s.t.} \quad y^*(x)= \argmin_{y\in\mathbb{R}^{q}} g(x,y),
\end{align}
where the upper- and inner-level functions $f$ and $g$ are both jointly continuously differentiable. The goal of~\cref{objective_deter} is to minimize the objective function $\Phi(x)$ with respect to (w.r.t.)~$x$, where $y^*(x)$ is obtained by solving the lower-level minimization problem. In this paper, we focus on the setting where the lower-level function $g$ is strongly convex w.r.t.~$y$, and the upper-level objective function $\Phi(x)$ is nonconvex. Such geometrics commonly exist in many applications such as meta-learning and hyperparameter optimization, where $g$ corresponds to an empirical loss  with a strongly-convex regularizer and $x$ are parameters of neural networks.

A broad collection of algorithms have been proposed to solve bilevel optimization problems. For example, \citealt{hansen1992new,shi2005extended,moore2010bilevel} reformulated the bilevel problem in~\cref{objective_deter} into a single-level constrained problem based on the optimality conditions of the lower-level problem. However, such type of methods often involve  a large number of constraints, and  are hard to implement in machine learning applications. Recently, more efficient gradient-based bilevel optimization algorithms have been proposed, which can be generally categorized into the approximate implicit differentiation (AID) based approach~\citep{domke2012generic,pedregosa2016hyperparameter,gould2016differentiating,liao2018reviving,ghadimi2018approximation,grazzi2020iteration,lorraine2020optimizing} and the iterative differentiation (ITD) based approach~\citep{domke2012generic,maclaurin2015gradient,franceschi2017forward,franceschi2018bilevel,shaban2019truncated,grazzi2020iteration}. However, most of these studies have focused on the asymptotic convergence analysis, and the nonasymptotic convergence rate analysis (that characterizes how fast an algorithm converges) has not been well explored except a few attempts recently. \citealt{ghadimi2018approximation} provided the convergence rate analysis for the ITD-based approach. \citealt{grazzi2020iteration} provided the iteration complexity for the hypergradient computation via ITD and AID, but did not characterize the  convergence rate for the entire execution of algorithms. 


\begin{table*}[!t]
 \centering
 \caption{Comparison of bilevel deterministic optimization algorithms.}
 \vspace{0.1cm}
 \begin{threeparttable}
  \begin{tabular}{|c|c|c|c|c|c|}
   \hline
Algorithm & Gc($f,\epsilon$) & Gc($g,\epsilon$) & JV($g,\epsilon$) &  HV($g,\epsilon$)    \\\hline\hline
AID-BiO \citep{ghadimi2018approximation}  & $\mathcal{O}(\kappa^4\epsilon^{-1})$ & $\mathcal{O}(\kappa^5\epsilon^{-5/4})$ &  $\mathcal{O}\left(\kappa^4\epsilon^{-1}\right)$& $\mathcal{\widetilde O}\left(\kappa^{4.5}\epsilon^{-1}\right)$ \\ \hline
   \cellcolor{blue!15}{AID-BiO (this paper)} & \cellcolor{blue!15}{$\mathcal{O}(\kappa^3\epsilon^{-1})$} & \cellcolor{blue!15}{$\mathcal{ O}(\kappa^4\epsilon^{-1})$ } & \cellcolor{blue!15}{$\mathcal{O}\left(\kappa^{3}\epsilon^{-1}\right)$} & \cellcolor{blue!15}{ $\mathcal{ O}\left(\kappa^{3.5}\epsilon^{-1}\right)$} \\ \hline
      \cellcolor{blue!15}{ITD-BiO (this paper)} & \cellcolor{blue!15}{$\mathcal{O}(\kappa^3\epsilon^{-1})$} & \cellcolor{blue!15}{$\mathcal{\widetilde O}(\kappa^4\epsilon^{-1})$ } & \cellcolor{blue!15}{$\mathcal{\widetilde O}\left(\kappa^4\epsilon^{-1}\right)$} & \cellcolor{blue!15}{ $\mathcal{\widetilde O}\left(\kappa^4\epsilon^{-1}\right)$} \\ \hline
  \end{tabular}\label{tab:determinstic}
   \begin{tablenotes}
  \item $\mbox{\normalfont Gc}(f,\epsilon)$ and $\mbox{\normalfont Gc}(g,\epsilon)$: 
 number of gradient evaluations w.r.t. $f$ and $g$. \;\;$\kappa:$ condition number.
 \item $\mbox{\normalfont JV}(g,\epsilon)$: number of Jacobian-vector products $\nabla_x\nabla_y g(x,y)v$. Notation $\mathcal{\widetilde O}$: omit $\log\frac{1}{\epsilon}$ terms.
  \item  $\mbox{\normalfont HV}(g,\epsilon)$: number of Hessian-vector products $\nabla_y^2g(x,y) v$.
   \end{tablenotes}
 \end{threeparttable}
  \vspace{-0.1cm}
\end{table*}
\begin{table*}[!t]
\vspace{-0.25cm}
 \centering
 \caption{Comparison of bilevel stochastic optimization algorithms.}
 \vspace{0.1cm}
 \begin{threeparttable}
  \begin{tabular}{|c|c|c|c|c|c|}
   \hline
Algorithm & Gc($F,\epsilon$) & Gc($G,\epsilon$) & JV($G,\epsilon$) &  HV($G,\epsilon$)  
\\\hline\hline
TTSA \citep{hong2020two}  &$ \mathcal{O}(\text{\scriptsize poly}(\kappa)\epsilon^{-\frac{5}{2}})$\tnote{*} & $\mathcal{O}(\text{\scriptsize poly}(\kappa)\epsilon^{-\frac{5}{2}})$& $\mathcal{O}(\text{\scriptsize poly}(\kappa)\epsilon^{-\frac{5}{2}})$&$\mathcal{O}(\text{\scriptsize poly}(\kappa)\epsilon^{-\frac{5}{2}})$
\\ \hline
BSA \citep{ghadimi2018approximation}  & $\mathcal{O}(\kappa^6\epsilon^{-2})$ & $\mathcal{O}(\kappa^9\epsilon^{-3})$ &  $\mathcal{O}\left(\kappa^6\epsilon^{-2}\right)$& $\mathcal{\widetilde O}\left(\kappa^6\epsilon^{-2}\right)$
\\ \hline
   \cellcolor{blue!15}{stocBiO (this paper)} & \cellcolor{blue!15}{$\mathcal{O}(\kappa^5\epsilon^{-2})$} & \cellcolor{blue!15}{$\mathcal{O}(\kappa^9\epsilon^{-2})$ } & \cellcolor{blue!15}{$\mathcal{ O}\left(\kappa^5\epsilon^{-2}\right)$} & \cellcolor{blue!15}{$\mathcal{\widetilde O}\left(\kappa^6\epsilon^{-2}\right)$} \\ \hline 
  \end{tabular}\label{tab:stochastic}
   \begin{tablenotes}
  \item[*] We use $\text{poly}(\kappa)$ because \citealt{hong2020two} does not provide the explicit dependence on $\kappa$.
 \end{tablenotes}
 \end{threeparttable}
\end{table*}

\begin{list}{$\bullet$}{\topsep=0.2ex \leftmargin=0.15in \rightmargin=0.2in \itemsep =0.2in}
\item Thus, the first focus of this paper is to develop a {\em comprehensive and sharper} theory, which covers a broader class of bilevel optimizers via ITD and AID techniques, and more importantly, improves existing  analysis with a more practical parameter selection and orderwisely lower computational complexity. 
%
\end{list}


The {\em stochastic} bilevel optimization often occurs    
in  applications where fresh data are sampled for algorithm iterations (e.g., in reinforcement learning~\citep{hong2020two}) or the sample size of training data is large (e.g., hyperparameter optimization~\citep{franceschi2018bilevel}, Stackelberg game~\citep{roth2016watch}). Typically, the  objective function is given by 
\begin{align}\label{objective}
&\min_{x\in\mathbb{R}^{p}} \Phi(x)=f(x, y^*(x))= 
\begin{cases}
\frac{1}{n}{\sum_{i=1}^nF(x,y^*(x);\xi_i) } \\
\mathbb{E}_{\xi} \left[F(x,y^*(x);\xi)\right] 
\end{cases}\nonumber\\
& \;\mbox{s.t.} \;y^*(x)= \argmin_{y\in\mathbb{R}^q} g(x,y)=
\begin{cases}
\frac{1}{m}{\sum_{i=1}^{m} G(x,y;\zeta_i)} \\
\mathbb{E}_{\zeta} \left[G(x,y;\zeta)\right]
\end{cases}
\end{align}
where $f(x,y)$ and $g(x,y)$ take either the expectation form w.r.t. the random variables $\xi$ and $\zeta$ or the finite-sum form over given 
data $\gD_{n,m}=\{\xi_i,\zeta_j, i=1,...,n;j=1,...,m\}$ often with large sizes $n$ and $m$. During the optimization process, data batch is sampled via the distributions of $\xi$ and $\zeta$ or from the set $\gD_{n,m}$. For such a stochastic setting, \citealt{ghadimi2018approximation} proposed a bilevel stochastic approximation (BSA) method via single-sample gradient and Hessian estimates.  Based on such a method, \citealt{hong2020two} further proposed a two-timescale stochastic approximation (TTSA) algorithm, and showed that TTSA achieves a better trade-off between the complexities of inner- and outer-loop optimization stages than BSA. 
\begin{list}{$\bullet$}{\topsep=0.2ex \leftmargin=0.15in \rightmargin=0.2in \itemsep =0.2in}
\item The second focus of this paper is to design a more sample-efficient algorithm for bilevel  stochastic optimization, which achieves lower computational complexity by orders of magnitude than BSA and TTSA.
\end{list}

\subsection{Main Contributions}
Our main contributions lie in developing shaper theory and provably faster algorithms for nonconvex-strongly-convex bilevel deterministic and stochastic optimization problems, respectively. 
Our  analysis  involves several new developments, which can be of independent interest.  


We first provide a unified convergence rate and complexity analysis for both ITD and  AID based bilevel optimizers, which we call as ITD-BiO and AID-BiO. Compared to existing analysis in~\citealt{ghadimi2018approximation} for AID-BiO that requires a continuously increasing number of inner-loop steps to achieve the guarantee, our analysis allows a constant number of inner-loop steps as often used in practice. In addition, we introduce a warm start initialization for the inner-loop updates and the outer-loop hypergradient estimation,  which allows us to backpropagate the tracking errors to previous loops, and yields an improved computational complexity. As shown in \Cref{tab:determinstic}, the gradient complexities Gc($f,\epsilon$), Gc($g,\epsilon$), and Jacobian- and Hessian-vector product complexities JV($g,\epsilon$) and HV($g,\epsilon$) of AID-BiO to attain an $\epsilon$-accurate stationary point improve those of~\citealt{ghadimi2018approximation} by the order of $\kappa$, $\kappa\epsilon^{-1/4}$, $\kappa$, and $\kappa$, respectively, where $\kappa$ is the condition number. In addition,  our analysis shows that AID-BiO requires less computations of Jacobian- and Hessian-vector products than ITD-BiO by an order of $\kappa$ and $\kappa^{1/2}$, which implies that AID can be more computationally and memory efficient than ITD. 
 
We then propose a stochastic bilevel optimizer (stocBiO) to solve the stochastic bilevel optimization problem in~\cref{objective}. Our algorithm features a {\em mini-batch} hypergradient estimation via implicit differentiation,  where the core design involves 
 a sample-efficient hypergradient estimator via the Neumann series. 
As shown in \Cref{tab:stochastic}, the gradient complexities of our proposed algorithm w.r.t.~$F$  and $G$  improve upon those of BSA~\citep{ghadimi2018approximation} by an order of $\kappa$ and $\epsilon^{-1}$, respectively. In addition, the Jacobian-vector product complexity JV($G,\epsilon$) of our algorithm improves that of BSA by an order of $\kappa$. In terms of the target accuracy $\epsilon$, our computational complexities  improve those of TTSA~\citep{hong2020two} by an order of $\epsilon^{-1/2}$.


Our results further provide the theoretical complexity guarantee for ITD-BiO, AID-BiO and stocBiO in meta-learning and hyperparameter optimization. The experiments validate our theoretical results for deterministic bilevel optimization, and demonstrate the superior efficiency  of stocBiO for stochastic bilevel optimization.  

\subsection{Related Work}
{\bf Bilevel optimization approaches}: Bilevel optimization was first introduced by~\citealt{bracken1973mathematical}. Since then, a number of bilevel optimization algorithms have been proposed, which include but not limited to constraint-based methods~\citep{shi2005extended,moore2010bilevel} and gradient-based methods~\citep{domke2012generic,pedregosa2016hyperparameter,gould2016differentiating,maclaurin2015gradient,franceschi2018bilevel,ghadimi2018approximation,liao2018reviving,shaban2019truncated,hong2020two,liu2020generic,li2020improved,grazzi2020iteration,lorraine2020optimizing,ji2021lower}. Among them, \citealt{ghadimi2018approximation,hong2020two} provided the complexity analysis for their proposed methods for the nonconvex-strongly-convex bilevel optimization problem.
For such a problem, this paper develops a general and enhanced convergence rate analysis for gradient-based bilevel optimizers for the deterministic setting, and proposes a novel algorithm for the stochastic setting  with order-level lower computational complexity than the existing results. 

Other types of loss geometries have also been studied. For example, \citealt{liu2020generic,li2020improved} assumed that the lower- and upper-level functions $g(x,\cdot)$ and $f(x,\cdot)$ are convex and strongly convex, and provided an asymptotic analysis for their methods. \citealt{ghadimi2018approximation,hong2020two}  studied the setting where $\Phi(\cdot)$ is strongly convex or convex, and  $g(x,\cdot)$ is strongly convex. 

\noindent {\bf Bilevel optimization in meta-learning}: Bilevel optimization framework has been successfully applied to meta-learning recently~\citep{snell2017prototypical,franceschi2018bilevel,rajeswaran2019meta,zugner2019adversarial,ji2020convergence,ji2020multi}. For example,
 \citealt{snell2017prototypical}  proposed a bilevel optimization procedure for  meta-learning  to learn a common embedding model  for all tasks. 
\citealt{rajeswaran2019meta} reformulated the model-agnostic meta-learning (MAML)~\citep{finn2017model} as bilevel optimization, and proposed iMAML via implicit gradient. Our work provides a theoretical guarantee for two popular types of bilevel optimizer, i.e., AID-BiO and ITD-BiO, for meta-learning. 

\noindent {\bf Bilevel optimization in hyperparameter optimization}: 
Hyperparameter optimization has become increasingly important as a powerful tool in the automatic machine learning (autoML)~\citep{okuno2018hyperparameter,yu2020hyper}. Recently, various bilevel optimization algorithms have been proposed for hyperparameter optimization, which include implicit differentiation based methods~\citep{pedregosa2016hyperparameter}, dynamical system based methods via reverse or forward gradient computation~\citep{franceschi2017forward,franceschi2018bilevel,shaban2019truncated}, etc.  
Our work demonstrates superior efficiency of the proposed stocBiO algorithm in hyperparameter optimization.

\section{Algorithms}\label{sec:alg}
In this section, we describe two popular types of {\em deterministic} bilevel optimization algorithms, and  
propose a new algorithm for {\em stochastic} bilevel optimization. 
\subsection{Algorithms for Deterministic Bilevel Optimization}

As shown in~\Cref{alg:main_deter}, we describe two popular types of deterministic bilevel optimizers respectively based on AID and ITD (referred to as AID-BiO and ITD-BiO) for solving the problem~\cref{objective_deter}.  

Both AID-BiO and ITD-BiO update in a nested-loop manner. In the inner loop, both of them run $D$ steps of gradient decent (GD) to find an approximation point $y_k^D$ close to $y^*(x_k)$. 
Note that we choose the initialization $y_{k}^0$ of each inner loop as the output $y_{k-1}^D$ of the preceding inner loop rather than a random start.  Such a {\em warm start} allows us to backpropagate the tracking error $\|y_k^D-y^*(x_k)\|$ to previous loops, and yields an improved computational complexity.

  \begin{figure*}[t]
	\centering  
	\includegraphics[width=110mm]{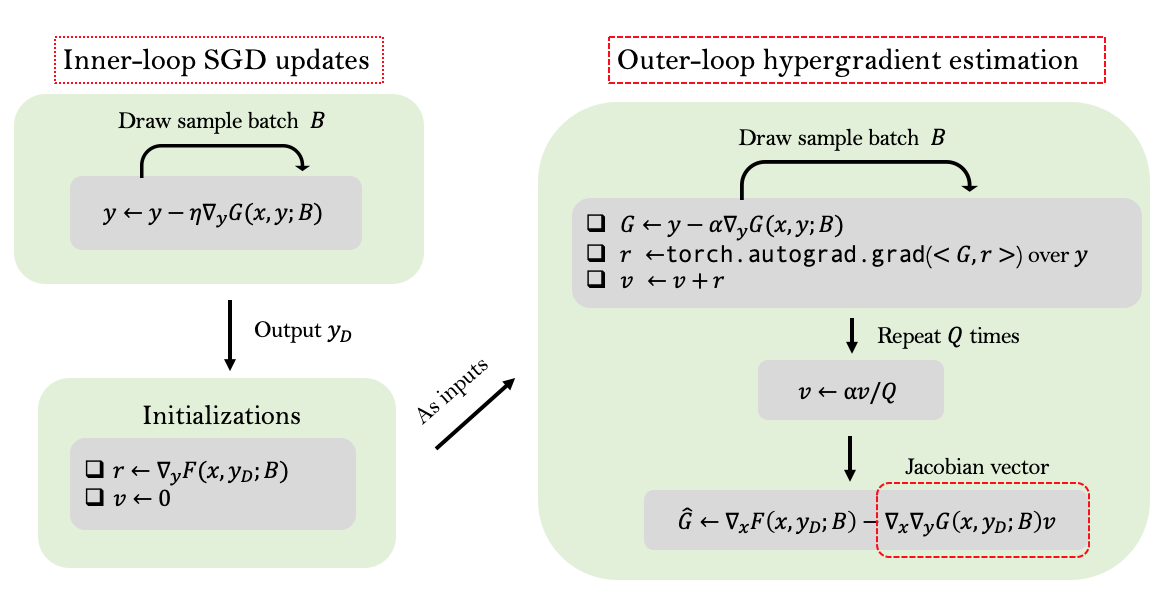}
	\vspace{-0.2cm}
	\caption{Illustration of hyperparameter estimation in our proposed stocBiO algorithm. Note that the  hyperparameter estimation (lines 9-10 in \Cref{alg:main}) involves only computations of automatic differentiation over scalar $<G_j(y),r_i>$ w.r.t.~$y$. In addition, our implementation applies the function {\em torch.autograd.grad} in PyTorch, which automatically determines the size of Jacobians. 
More details can be found in our code via {\small\url{https://github.com/JunjieYang97/StocBio_hp}}.
}\label{fig:stocBiO_illustrate}
\end{figure*}

At the outer loop,  AID-BiO first solves $v_k^N$ from a linear system $\nabla_y^2 g(x_k,y_k^D) v = 
\nabla_y f(x_k,y^D_k)$\footnote{Solving this linear system is equivalent to solving a quadratic programming $\min_v\frac{1}{2} v^T\nabla_y^2g(x_k,y_k^D) v-v^T\nabla_y f(x_k,y^D_k).$ } using $N$ steps of conjugate-gradient (CG) starting from $v_k^0$ (where we also adopt a warm start with $v_k^0=v_{k-1}^N$), and then constructs \begin{align}\label{hyper-aid}
\widehat\nabla \Phi(x_k)= \nabla_x f(x_k,y_k^T) -\nabla_x \nabla_y g(x_k,y_k^T)v_k^N
\end{align}
as an estimate of the true hypergradient $\nabla \Phi(x_k)$, whose form is given by the following proposition. 
\begin{proposition}\label{prop:grad}
Hypergradient $\nabla \Phi(x_k)$ takes the forms of 
\begin{small}
\begin{align}\label{trueG}
\nabla \Phi(x_k) =&  \nabla_x f(x_k,y^*(x_k)) -\nabla_x \nabla_y g(x_k,y^*(x_k)) v_k^*, 
\end{align}
\end{small}
\hspace{-0.09cm}where $v_k^*$ is the solution of the following linear system $$ \nabla_y^2 g(x_k,y^*(x_k))v=
\nabla_y f(x_k,y^*(x_k)).$$
\end{proposition}
As shown in~\citealt{domke2012generic,grazzi2020iteration}, the construction of~\cref{hyper-aid} involves only Hessian-vector products in solving $v_N$ via CG and Jacobian-vector product $\nabla_x \nabla_y g(x_k,y_k^D)v_k^N$, which can be efficiently computed and stored via existing automatic differentiation packages.  

\begin{algorithm}[t]
\small 
	\caption{Bilevel algorithms  via AID or ITD}    
	\label{alg:main_deter}
	\begin{algorithmic}[1]
		\STATE {\bfseries Input:}  $K,D,N$, stepsizes $\alpha, \beta $, initializations $x_0, y_0,v_0$.
		\FOR{$k=0,1,2,...,K$}
		\STATE{Set $y_k^0 = y_{k-1}^{D} \mbox{ if }\; k> 0$ and $y_0$ otherwise  }
		\FOR{$t=1,....,D$}
		\vspace{0.05cm}
		\STATE{Update $y_k^t = y_k^{t-1}-\alpha \nabla_y g(x_k,y_k^{t-1}) $}
		\vspace{0.05cm}
		\ENDFOR
                  \STATE{Hypergradient estimation via 
\vspace{0.1cm}
                  \\\hspace{0cm} {\bf AID}: 1) set $v_k^0 = v_{k-1}^{N} \mbox{ if }\; k> 0$ and $v_0$ otherwise
              \vspace{0.1cm}    \\\hspace{0.8cm}2) solve $v_k^N$ from $\nabla_y^2 g(x_k,y_k^D) v = 
\nabla_y f(x_k,y^D_k)$ 
 \vspace{0.1cm} \\ \hspace{1.05cm} via $N$ steps of CG starting from $v_k^0$
 \vspace{0.1cm} \\\hspace{0.8cm}3) get Jacobian-vector product {\small$\nabla_x \nabla_y g(x_k,y_k^D)v_k^N$} 
 \vspace{0.1cm} \\ \hspace{1.05cm} via automatic differentiation
 \vspace{0.1cm} \\\hspace{0.8cm}4) {\small$\widehat\nabla \Phi(x_k)= \nabla_x f(x_k,y_k^D) -\nabla_x \nabla_y g(x_k,y_k^D)v_k^N$} 
             \vspace{0.0cm}      \\ \hspace{0cm} {\bf  ITD}: compute $\widehat\nabla \Phi(x_k)=\frac{\partial f(x_k,y^D_k)}{\partial x_k}$ via backpropagation
                    }
                 \STATE{Update $x_{k+1}=x_k- \beta \widehat\nabla \Phi(x_k)$}
		\ENDFOR
	\end{algorithmic}
	\end{algorithm}

As a comparison, the outer loop of ITD-BiO computes the gradient $\frac{\partial f(x_k,y^D_k(x_k))}{\partial x_k}$  as an approximation of the hyper-gradient $\nabla \Phi(x_k)=\frac{\partial f(x_k,y^*(x_k))}{\partial x_k}$ via backpropagation, where we write  $y^D_k(x_k)$ because the output $y_k^D$ of the inner loop has a dependence on $x_k$ through the inner-loop iterative GD updates. The explicit form of the estimate $\frac{\partial f(x_k,y^D_k(x_k))}{\partial x_k}$ is given by the following proposition via the chain rule. For notation simplification, let $\prod_{j=D}^{D-1} (\cdot)= I$.
   \begin{proposition}\label{deter:gdform}
 $\frac{\partial f(x_k,y^D_k(x_k))}{\partial x_k}$ takes the analytical form of:
\begin{small}
\begin{align*}
\frac{\partial f(x_k,y^D_k)}{\partial x_k}= &\nabla_x f(x_k,y_k^D) -\alpha\sum_{t=0}^{D-1}\nabla_x\nabla_y g(x_k,y_k^{t})
\\&\times\prod_{j=t+1}^{D-1}(I-\alpha  \nabla^2_y g(x_k,y_k^{j}))\nabla_y f(x_k,y_k^D).
\end{align*}
\end{small}
\end{proposition}
\vspace{-0.4cm}
 \Cref{deter:gdform} shows that the differentiation involves the computations of second-order derivatives such as Hessian $ \nabla^2_y g(\cdot,\cdot)$. Since efficient Hessian-free methods  have been successfully deployed in the existing automatic differentiation tools, computing these second-order derivatives reduces to more efficient computations of Jacobian- and Hessian-vector products.

\subsection{Algorithm for Stochastic Bilevel Optimization}
We  propose a new stochastic bilevel optimizer (stocBiO) in~\Cref{alg:main} to solve the problem~\cref{objective}. 
It has a double-loop structure similar to \Cref{alg:main_deter}, but  runs $D$ steps of stochastic gradient decent (SGD) at the inner loop to obtain an approximated solution  $y_k^D$.  
Based on the output $y_k^D$ of the inner loop, stocBiO first computes a gradient {\small $\nabla_y F(x_k,y_k^D;\gD_F)$} over a sample batch $\gD_F$, and then computes a vector $v_Q$ as an estimated solution  of the linear system {\small $ \nabla_y^2 g(x_k,y^*(x_k))v=
\nabla_y f(x_k,y^*(x_k))$} via \Cref{alg:hessianEst}. Here, $v_Q$  takes a form of 
\begin{align}\label{ours:est}
v_Q =& \eta \sum_{q=-1}^{Q-1}\prod_{j=Q-q}^Q (I - \eta \nabla_y^2G(x_k,y_k^D;\gB_j)) v_0,
\end{align}
where {\small$v_0 = \nabla_y F(x_k,y_k^D;\gD_F)$}, $\{\gB_j,j=1,...,Q\}$ are mutually-independent sample sets, $Q$ and $\eta$ are constants, and we let $\prod_{Q+1}^Q (\cdot)= I$ for notational simplification.  
Our construction of $v_Q$, i.e., \Cref{alg:hessianEst}, is motived by the Neumann series $\sum_{i=0}^\infty U^k=(I-U)^{-1}$, and involves only Hessian-vector products rather than Hessians, and hence is computationally and memory efficient.  {\bf This procedure is illustrated in \Cref{fig:stocBiO_illustrate}. }

Then, we construct 
\begin{small}
 \begin{align}\label{estG}
 \widehat \nabla \Phi(x_k) =&  \nabla_x F(x_k,y_k^D;\gD_F)-\nabla_x \nabla_y G(x_k,y_k^D;\gD_G)v_Q
 \end{align}
 \end{small}
\hspace{-0.12cm}as an estimate of hypergradient $\nabla \Phi(x_k)$. 
Compared to the deterministic case, it is more challenging to design a sample-efficient Hypergradient estimator in the stochastic case. For example, instead of choosing the same batch sizes for all $\gB_j,j=1,...,Q$  in~\cref{ours:est}, 
 our analysis captures the different impact of components $\nabla_y^2G(x_k,y_k^D;\gB_j),j=1,...,Q$ on  the hypergradient estimation variance, and inspires an adaptive and more efficient choice  by setting $|\gB_{Q-j}|$ to   decay  exponentially with  $j$ from $0$ to $Q-1$. By doing so, we  achieve an improved  complexity. 

\begin{algorithm}[t]
	\caption{Stochastic bilevel optimizer (stocBiO)}  
	\small
	\label{alg:main}
	\begin{algorithmic}[1]
		\STATE {\bfseries Input:} $K,D,Q$, stepsizes $\alpha$ and $\beta$, initializations $x_0$ and $y_0$.
		\FOR{$k=0,1,2,...,K$}
		\STATE{Set $y_k^0 = y_{k-1}^{D} \mbox{ if }\; k> 0$ and $y_0$ otherwise 
		}
		\FOR{$t=1,....,D$}
		\STATE{Draw a sample batch $\gS_{t-1}$}  
		\vspace{0.05cm}
		\STATE{Update $y_k^t = y_k^{t-1}-\alpha \nabla_y G(x_k,y_k^{t-1}; \gS_{t-1}) $}
		\ENDFOR
                 \STATE{Draw sample batches $\gD_F,\gD_H$ and $\gD_G$ }
                 \STATE{Compute gradient {\small$v_0=\nabla_y F(x_k,y_k^D;\gD_F)$}}
                  \STATE{Construct estimate $v_Q$ via \Cref{alg:hessianEst} given $v_0$}
                  \STATE{Compute {\small$\nabla_x \nabla_y G(x_k,y_k^D;\gD_G) v_Q $}}

                  \STATE{ Compute gradient estimate $  \widehat \nabla \Phi(x_k) $ via~\cref{estG}
                    }
                 \STATE{Update $x_{k+1}=x_k- \beta \widehat \nabla \Phi(x_k) $}
		\ENDFOR
	\end{algorithmic}
	\end{algorithm}	
	\begin{algorithm}[t]
	\small
	\caption{Construct $v_Q$ given $v_0$  }  
	\label{alg:hessianEst}
	\begin{algorithmic}[1]
		\STATE {\bfseries Input:} Integer $Q$,  samples $\gD_H= \{\gB_j\}_{j=1}^{Q}$ and constant $\eta$.
		\FOR{$j=1,2,...,Q$}
		\STATE{Sample $\gB_j$ and compute $G_j(y)=y-\eta \nabla_y G(x,y; \gB_j)$
		}
		\ENDFOR
		\STATE{Set $r_Q = v_0$}
		\FOR{$i=Q,...,1$}
		\STATE{$r_{i-1}=\partial \big( G_i(y)r_i\big)/\partial y=r_i-\eta \nabla_y^2G(x,y;\gB_i)r_i$ via automatic differentiation}
		\ENDFOR
		\STATE{Return $v_Q=\eta \sum_{i=0}^Qr_i$
		}
	\end{algorithmic}
	\end{algorithm}

\section{Definitions and Assumptions}
Let $z=(x,y)$ denote all parameters. For simplicity, suppose sample sets $\gS_t$  for all $t=0,...,D-1$, $\gD_G$ and $\gD_F$ have the sizes of  $S$, $D_g$ and $D_f$, respectively.  In this paper, we focus on the following types of loss functions for both the deterministic and stochastic cases. 
\begin{assum}\label{assum:geo}
The lower-level function $g(x,y)$ is $\mu$-strongly-convex w.r.t.~$y$ and the total objective function $\Phi(x)=f(x,y^*(x))$ is nonconvex w.r.t.~$x$.  
For the stochastic setting, the same assumptions hold for $G(x,y;\zeta)$ and $\Phi(x)$, respectively.
\vspace{-0.2cm} 
\end{assum}
Since $\Phi(x)$ is nonconvex, algorithms are expected to find an $\epsilon$-accurate stationary point defined as follows. 
\begin{definition}
We say $\bar x$ is an $\epsilon$-accurate stationary point for the objective function $\Phi(x)$ in~\cref{objective} if $\mathbb{E}\|\nabla \Phi(\bar x)\|^2\leq \epsilon$, where $\bar x$ is the output of  an algorithm.
\end{definition}
In order to compare the performance of different bilevel algorithms, we adopt the following metrics of  complexity. 

\begin{definition}\label{com_measure}
For a function $f(x,y)$ and a vector $v$, let $\mbox{\normalfont Gc}(f,\epsilon)$ be the number of the partial gradient $\nabla_x f$ or $\nabla_y f$, and let $\mbox{\normalfont JV}(g,\epsilon)$ and  $\mbox{\normalfont HV}(g,\epsilon)$ be the number of Jacobian-vector products $\nabla_x\nabla_y g (x,y)v$. 
  and  Hessian-vector products $\nabla_y^2g(x,y) v$.  For the stochastic case, similar metrics are adopted but w.r.t.~the stochastic function $F(x,y;\xi)$. 
 \end{definition} 
We  take the following standard assumptions on the loss functions in~\cref{objective}, which have been widely adopted in bilevel optimization~\citep{ghadimi2018approximation,ji2020convergence}.
\begin{assum}\label{ass:lip}
The loss function $f(z)$ and $g(z)$ satisfy
\begin{list}{$\bullet$}{\topsep=0.2ex \leftmargin=0.2in \rightmargin=0.in \itemsep =0.01in}
\item The function $f(z)$ is $M$-Lipschitz, i.e.,  for any $z,z^\prime$, $$|f(z)-f(z^\prime)|\leq M\|z-z^\prime\|.$$
\item $\nabla f(z)$ and $\nabla g(z)$ are $L$-Lipschitz, i.e., for any $z,z^\prime$, 
\vspace{-0.1cm}
\begin{align*}
\|\nabla f(z)-\nabla f(z^\prime)\|\leq& L\|z-z^\prime\|,
\\\|\nabla g(z)-\nabla g(z^\prime)\|\leq& L\|z-z^\prime\|.
\end{align*}
\end{list}
\vspace{-0.2cm}

For the stochastic case, the same assumptions hold for $F(z;\xi)$ and $G(z;\zeta)$ for any given $\xi$ and $\zeta$.
\end{assum}
As shown in~\Cref{prop:grad}, the gradient of the objective function $\Phi(x)$ involves the second-order derivatives $\nabla_x\nabla_y g(z)$ and $\nabla_y^2 g(z)$. The following assumption imposes the Lipschitz conditions on such high-order derivatives, as also made in~\citealt{ghadimi2018approximation}.
\begin{assum}\label{high_lip}
Suppose the derivatives $\nabla_x\nabla_y g(z)$ and $\nabla_y^2 g(z)$ are $\tau$- and $\rho$- Lipschitz, i.e.,
\begin{list}{$\bullet$}{\topsep=0.2ex \leftmargin=0.2in \rightmargin=0.in \itemsep =0.02in}
\item For any $z,z^\prime$, $\|\nabla_x\nabla_y g(z)-\nabla_x\nabla_y g(z^\prime)\| \leq \tau \|z-z^\prime\|$.
\item For any $z,z^\prime$, $\|\nabla_y^2 g(z)-\nabla_y^2 g(z^\prime)\|\leq \rho \|z-z^\prime\|$.
\end{list} 
For the stochastic case, the same assumptions hold for $\nabla_x\nabla_y G(z;\zeta)$ and $\nabla_y^2 G(z;\zeta)$ for any $\zeta$.
\end{assum}
As typically adopted in the analysis for stochastic optimization, we make the following bounded-variance assumption for the lower-level stochastic function $G(z;\zeta)$. 
\begin{assum} \label{ass:bound} 
Gradient $\nabla G(z;\zeta)$ has a bounded variance, i.e., $\mathbb{E}_\xi \|\nabla G(z;\zeta)-\nabla g(z)\|^2 \leq \sigma^2$ for some $\sigma$.
\end{assum}
\section{Main Results for Bilevel Optimization}
\subsection{Deterministic Bilevel Optimization}\label{main:result_deter}
We first characterize the convergence and complexity of AID-BiO.  Let $\kappa=\frac{L}{\mu}$ denote the condition number. 
\begin{theorem}[AID-BiO]\label{th:aidthem}
Suppose Assumptions~\ref{assum:geo},~\ref{ass:lip}, \ref{high_lip} hold. Define  a smoothness parameter $L_\Phi = L + \frac{2L^2+\tau M^2}{\mu} + \frac{\rho L M+L^3+\tau M L}{\mu^2} + \frac{\rho L^2 M}{\mu^3}=\Theta(\kappa^3)$, choose the stepsizes $\alpha\leq \frac{1}{L}$, $\beta=\frac{1}{8L_\Phi}$, and set the inner-loop iteration number $D\geq\Theta(\kappa)$ and the CG iteration number  $N\geq \Theta(\sqrt{\kappa})$, where the detailed forms of  $D,N$ can be found in \Cref{appen:aid-bio}. 
 Then, the outputs of AID-BiO satisfy
\begin{align*}
\frac{1}{K}\sum_{k=0}^{K-1}\| \nabla \Phi(x_k)\|^2 \leq \frac{64L_\Phi (\Phi(x_0) - \inf_x\Phi(x))+5\Delta_0}{K}, 
\end{align*}
where $\Delta_0=\|y_0-y^*(x_{0})\|^2 + \|v_{0}^*-v_0\|^2>0$.

In order to achieve an $\epsilon$-accurate stationary point, the complexities satisfy 
\begin{list}{$\bullet$}{\topsep=0.ex \leftmargin=0.2in \rightmargin=0.in \itemsep =0.01in}
\item Gradient: {\small$\mbox{\normalfont Gc}(f,\epsilon)=\mathcal{O}(\kappa^3\epsilon^{-1}), \mbox{\normalfont Gc}(g,\epsilon)=\mathcal{O}(\kappa^4\epsilon^{-1}).$}
\item Jacobian- and Hessian-vector product complexities: {\small$ \mbox{\normalfont JV}(g,\epsilon)=\mathcal{O}\left(\kappa^3\epsilon^{-1}\right), \mbox{\normalfont HV}(g,\epsilon)=\mathcal{O}\left(\kappa^{3.5}\epsilon^{-1}\right).$}
\end{list}
\end{theorem}
As shown in \Cref{tab:determinstic}, 
the complexities  $\mbox{\normalfont Gc}(f,\epsilon)$, $\mbox{\normalfont Gc}(g,\epsilon)$, $\mbox{\normalfont JV}(g,\epsilon) $ and $\mbox{\normalfont HV}(g,\epsilon)$ of our analysis improves that of \citealt{ghadimi2018approximation} (eq.~(2.30) therein) by the order of $\kappa$, $\kappa\epsilon^{-1/4}$, $\kappa$ and $\kappa$.  Such an improvement is achieved by a refined analysis with a constant number of inner-loop steps, and by a warm start strategy to backpropagate the tracking errors $\|y_k^D-y^*(x_k)\|$  and $\|v_k^N-v^*_k\|$ to previous loops, as also demonstrated by our meta-learning experiments.  We next characterize the convergence and complexity performance of the ITD-BiO algorithm. \begin{theorem}[ITD-BiO]\label{th:determin}
Suppose Assumptions~\ref{assum:geo},~\ref{ass:lip}, and \ref{high_lip} hold. Define $L_\Phi $ as in \Cref{th:aidthem}, and choose $\alpha\leq \frac{1}{L}$,
 $\beta=\frac{1}{4L_\Phi}$ and $D\geq \Theta(\kappa\log\frac{1}{\epsilon})$, where the detailed form of $D$ can be found in \Cref{append:itd-bio}. Then, we have 
\begin{align*}
\frac{1}{K}\sum_{k=0}^{K-1}\| \nabla \Phi(x_k)\|^2 \leq \frac{16 L_\Phi (\Phi(x_0)-\inf_x\Phi(x))}{K} + \frac{2\epsilon}{3}.
\end{align*}
In order to achieve an $\epsilon$-accurate stationary point, the complexities satisfy 
\begin{list}{$\bullet$}{\topsep=0.ex \leftmargin=0.1in \rightmargin=0.in \itemsep =0.01in}
\item Gradient: {\small$\mbox{\normalfont Gc}(f,\epsilon)=\mathcal{O}(\kappa^3\epsilon^{-1}), \mbox{\normalfont Gc}(g,\epsilon)=\mathcal{\widetilde O}(\kappa^4\epsilon^{-1}).$}
\item Jacobian- and Hessian-vector product complexity: {\small$ \mbox{\normalfont JV}(g,\epsilon)=\mathcal{\widetilde O}\big(\kappa^4\epsilon^{-1}\big), \mbox{\normalfont HV}(g,\epsilon)=\mathcal{\widetilde O}\big(\kappa^4\epsilon^{-1}\big).$}
\end{list}
\end{theorem}
By comparing \Cref{th:aidthem} and \Cref{th:determin}, it can be seen that the complexities $\mbox{\normalfont JV}(g,\epsilon)$ and $\mbox{\normalfont HV}(g,\epsilon)$  of AID-BiO are better than those of ITD-BiO by the order of  $\kappa$ and $\kappa^{0.5}$, which implies that AID-BiO is more computationally and memory efficient than ITD-BiO, as verified in \Cref{fig:strfc100}.

\subsection{Stochastic Bilevel Optimization}\label{main:result}
We first  characterize the bias and variance of an important component $v_Q$ in~\cref{ours:est}. 

\begin{proposition}\label{prop:hessian}
Suppose Assumptions~\ref{assum:geo},~\ref{ass:lip} and \ref{high_lip} hold. Let  $\eta\leq \frac{1}{L}$ and choose  $|\gB_{Q+1-j}|=BQ(1-\eta\mu)^{j-1}$ for $j=1,...,Q$, where $B\geq \frac{1}{Q(1-\eta\mu)^{Q-1}}$. 
Then, the  bias satisfies   
\begin{align}{\label{bias}}
\big\|\mathbb{E} v_Q- [\nabla_y^2 g(x_k,y^D_k&)]^{-1}\nabla_y f(x_k,y^D_k)\big\|  \nonumber
\\&\leq  \mu^{-1}(1-\eta \mu)^{Q+1}M.
\end{align}
Furthermore, the estimation variance is given by 
\begin{align}\label{hessian_variance}
\mathbb{E}&\|v_Q-[\nabla_y^2 g(x_k,y^D_k)]^{-1}\nabla_y f(x_k,y^D_k)\|^2 \nonumber
\\&\leq\frac{4\eta^2  L^2M^2}{\mu^2} \frac{1}{B}+\frac{4(1-\eta \mu)^{2Q+2}M^2}{\mu^2}+ \frac{2M^2}{\mu^2D_f}.
\end{align}
\end{proposition}
\vspace{-0.2cm}
\Cref{prop:hessian} shows that if we choose $Q$, $B$ and $D_f$ at the order level of $\mathcal{O}(\log \frac{1}{\epsilon}) $, $\mathcal{O}(1/\epsilon)$ and $\mathcal{O}(1/\epsilon)$,  the bias and variance are smaller than $\mathcal{O}(\epsilon)$, and the required number of samples is
$\sum_{j=1}^Q BQ(1-\eta\mu)^{j-1} = \mathcal{O}\left(\epsilon^{-1}\log \frac{1}{\epsilon}\right)$.
Note that  the chosen batch size $|\gB_{Q+1-j}|$ exponentially decays w.r.t.~the index $j$. In comparison, the uniform choice of all $|\gB_j|$ would yield a worse complexity of $ \mathcal{O}\big( \epsilon^{-1}(\log\frac{1}{\epsilon})^2\big)$.


We next analyze stocBiO when $\Phi(x)$ is nonconvex.
\begin{theorem}\label{th:nonconvex}
Suppose Assumptions~\ref{assum:geo},~\ref{ass:lip}, \ref{high_lip} and~\ref{ass:bound} hold. Define 
$L_\Phi = L + \frac{2L^2+\tau M^2}{\mu} + \frac{\rho L M+L^3+\tau M L}{\mu^2} + \frac{\rho L^2 M}{\mu^3}$, and choose $\beta=\frac{1}{4L_\Phi}, \eta<\frac{1}{L}$, and 
 $D\geq\Theta(\kappa\log \kappa)$, where the detailed form of $D$ can be found in \Cref{mianshisimida}.  
We have
\begin{align}\label{eq:main_nonconvex}
\frac{1}{K}\sum_{k=0}^{K-1}\mathbb{E}\|\nabla&\Phi(x_k)\|^2 \leq  \mathcal{O}\Big( \frac{L_\Phi}{K}+  \kappa^2(1-\eta \mu)^{2Q} \nonumber
\\&+\frac{\kappa^5\sigma^2}{S}+ \frac{\kappa^2}{D_g} +\frac{\kappa^2}{D_f}+ \frac{\kappa^2}{B}\Big).
\end{align}
In order to achieve an $\epsilon$-accurate stationary point, the complexities satisfy 
\begin{list}{$\bullet$}{\topsep=0.4ex \leftmargin=0.15in \rightmargin=0.in \itemsep =0.01in}
\item Gradient: {\small$\mbox{\normalfont Gc}(F,\epsilon)=\mathcal{O}(\kappa^5\epsilon^{-2}), \mbox{\normalfont Gc}(G,\epsilon)=\mathcal{O}(\kappa^9\epsilon^{-2}).$}
\item Jacobian- and Hessian-vector product complexities: {\small$\mbox{\normalfont JV}(G,\epsilon)=\mathcal{O}(\kappa^5\epsilon^{-2}), \mbox{\normalfont HV}(G,\epsilon)=\mathcal{\widetilde O}(\kappa^6\epsilon^{-2}).$}
\end{list}
\end{theorem}
\Cref{th:nonconvex} shows that stocBiO converges sublinearly with the convergence error decaying exponentially w.r.t.~$Q$ and sublinearly w.r.t.~the batch sizes $S,D_g,D_f$ for gradient estimation and $B$ for Hessian inverse estimation. In addition, it can be seen that the number $D$ of the inner-loop steps is 
at  a  constant level, rather than a typical choice of $\Theta(\log(\frac{1}{\epsilon}))$.

As shown in~\Cref{tab:stochastic}, the gradient complexities of our proposed algorithm in terms of  $F$  and $G$  improve those of BSA in \citealt{ghadimi2018approximation} by an order of $\kappa$ and $\epsilon^{-1}$, respectively. In addition, the Jacobian-vector product complexity $\mbox{\normalfont JV}(G,\epsilon)$ of our algorithm improves that of BSA by the order of $\kappa$. 
In terms of the accuracy $\epsilon$, our gradient,  Jacobian- and Hessian-vector product complexities  improve those of TTSA in \citealt{hong2020two} all by an order of $\epsilon^{-0.5}$.  
\section{Applications to Meta-Learning}
Consider the few-shot meta-learning problem with $m$ tasks $\{\mathcal{T}_i,i=1,...,m\}$
 sampled from distribution $P_\gT$. Each task $\mathcal{T}_i$ has a loss function $\gL(\phi,w_i;\xi)$ over each data sample $\xi$, where $\phi$ are the parameters of an embedding model shared by all tasks, and $w_i$ are the task-specific parameters. The goal of this framework is to  find good parameters $\phi$ for all tasks, and building on the embedded features, each task then adapts its own parameters $w_i$ by minimizing its loss.
 

  \begin{figure*}[ht]
  \vspace{-2mm}
	\centering    
	\subfigure[dataset: miniImageNet  ]{\label{fig1:a}\includegraphics[width=40mm]{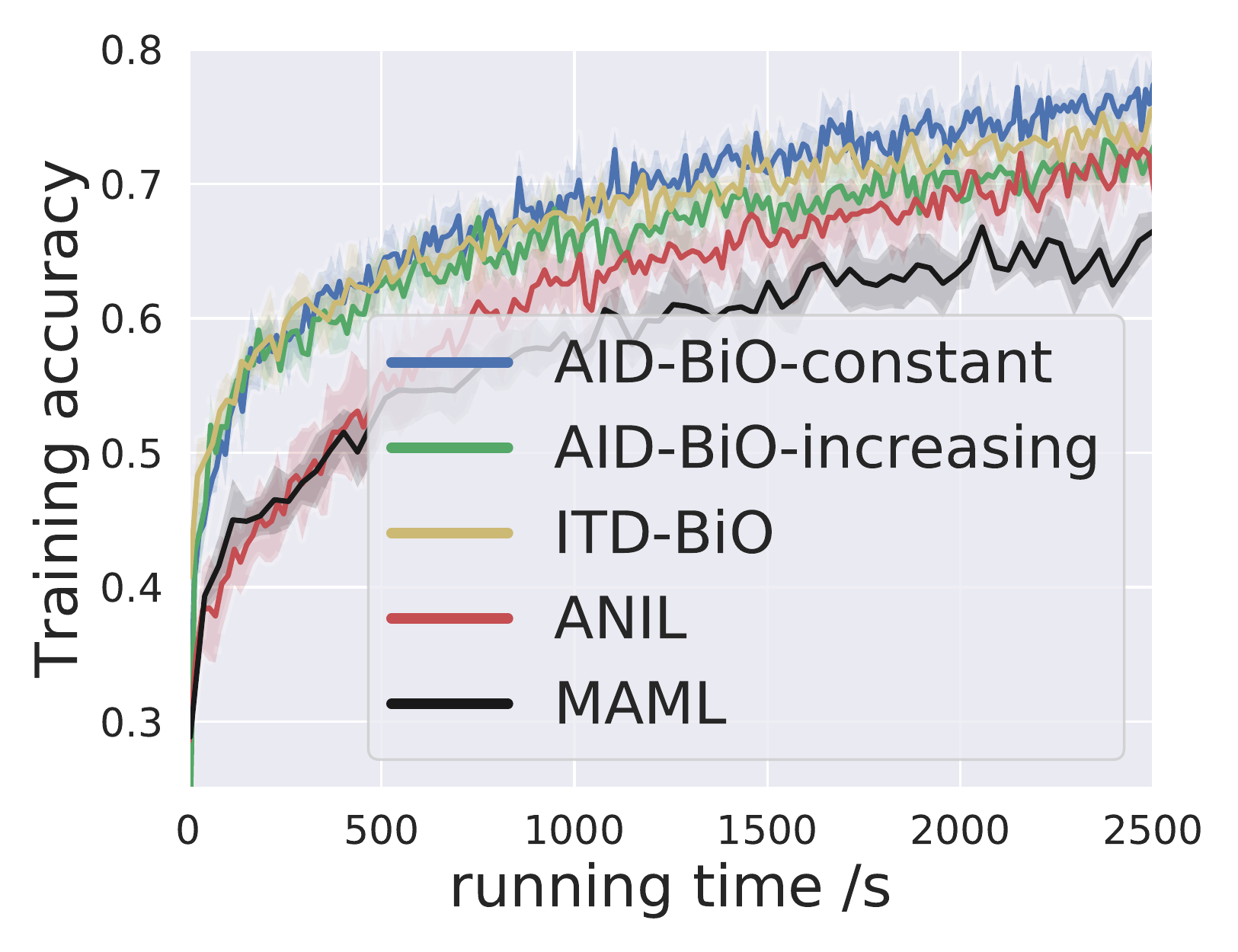}\includegraphics[width=40mm]{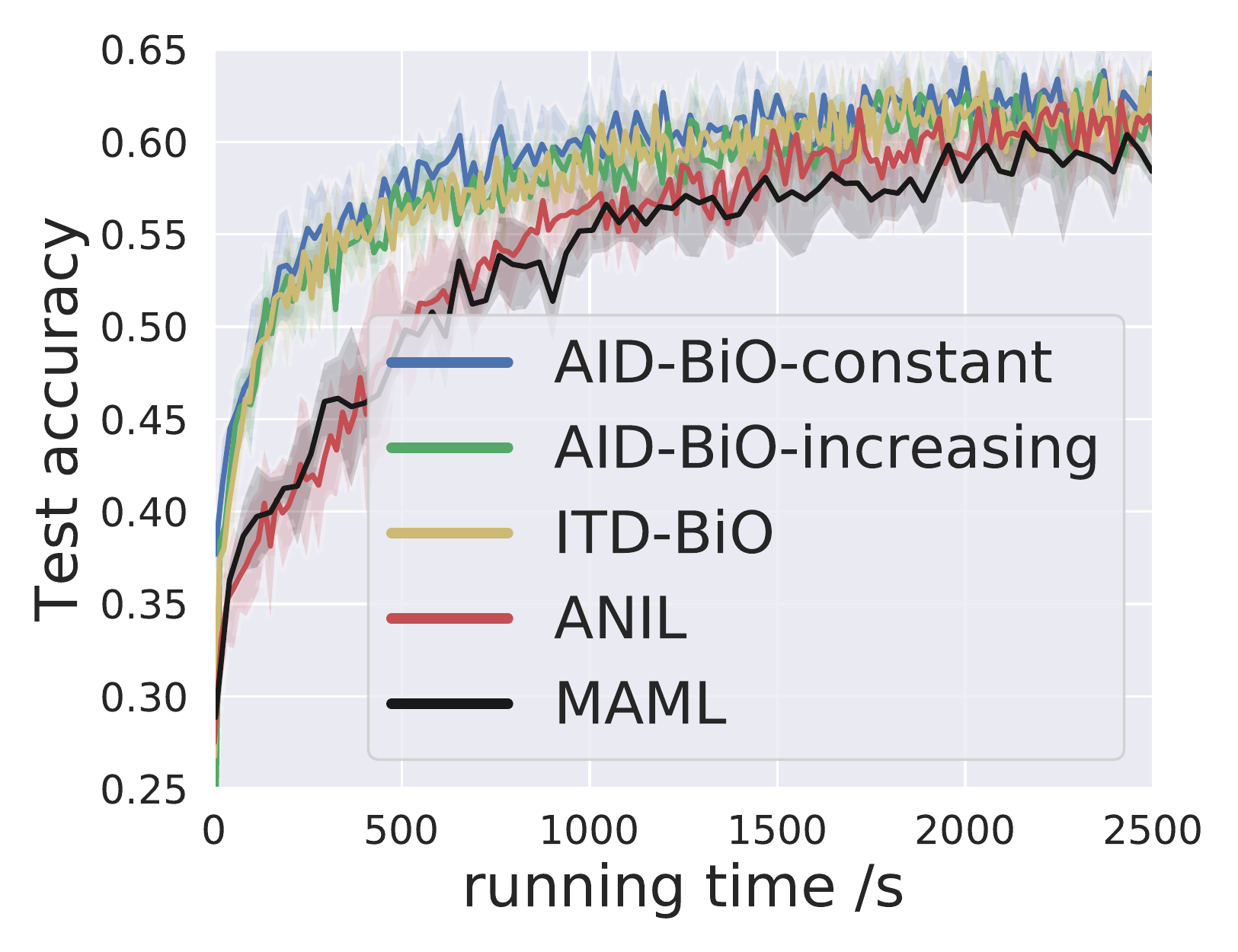}} 
	\subfigure[dataset: FC100]{\label{fig1:b}\includegraphics[width=40mm]{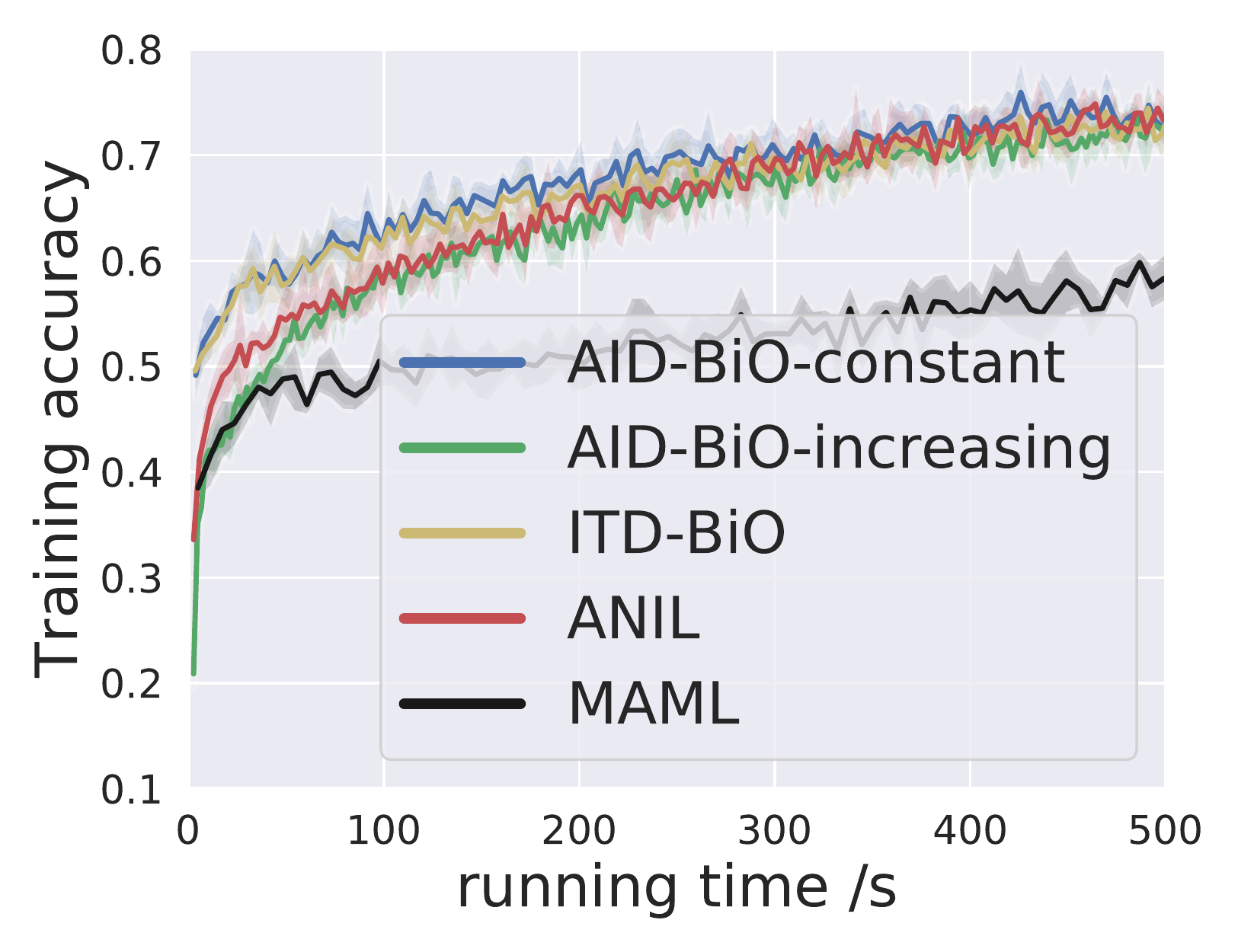}\includegraphics[width=40mm]{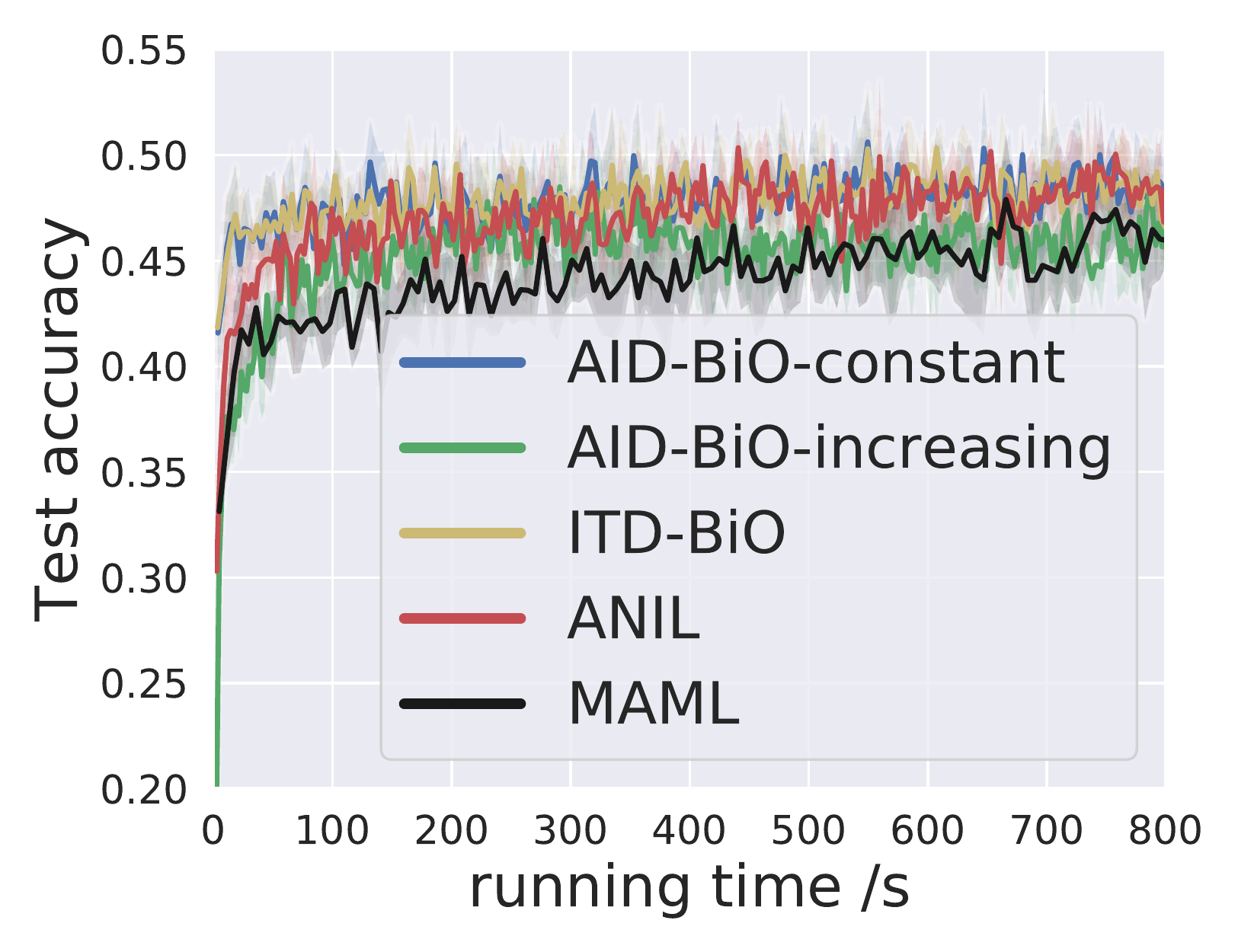}}  
	\vspace{-0.2cm}
	\caption{Comparison of various bilevel algorithms on meta-learning. For each dataset, left plot: training accuracy v.s. running time; right plot: test accuracy v.s. running time.}\label{fig:strfc100}
	  \vspace{-0.2cm}
\end{figure*}

  \begin{figure*}[ht]
  \vspace{-2mm}
	\centering    
	\subfigure[$T=10$, miniImageNet dataset]{\label{fig1:ci}\includegraphics[width=41mm]{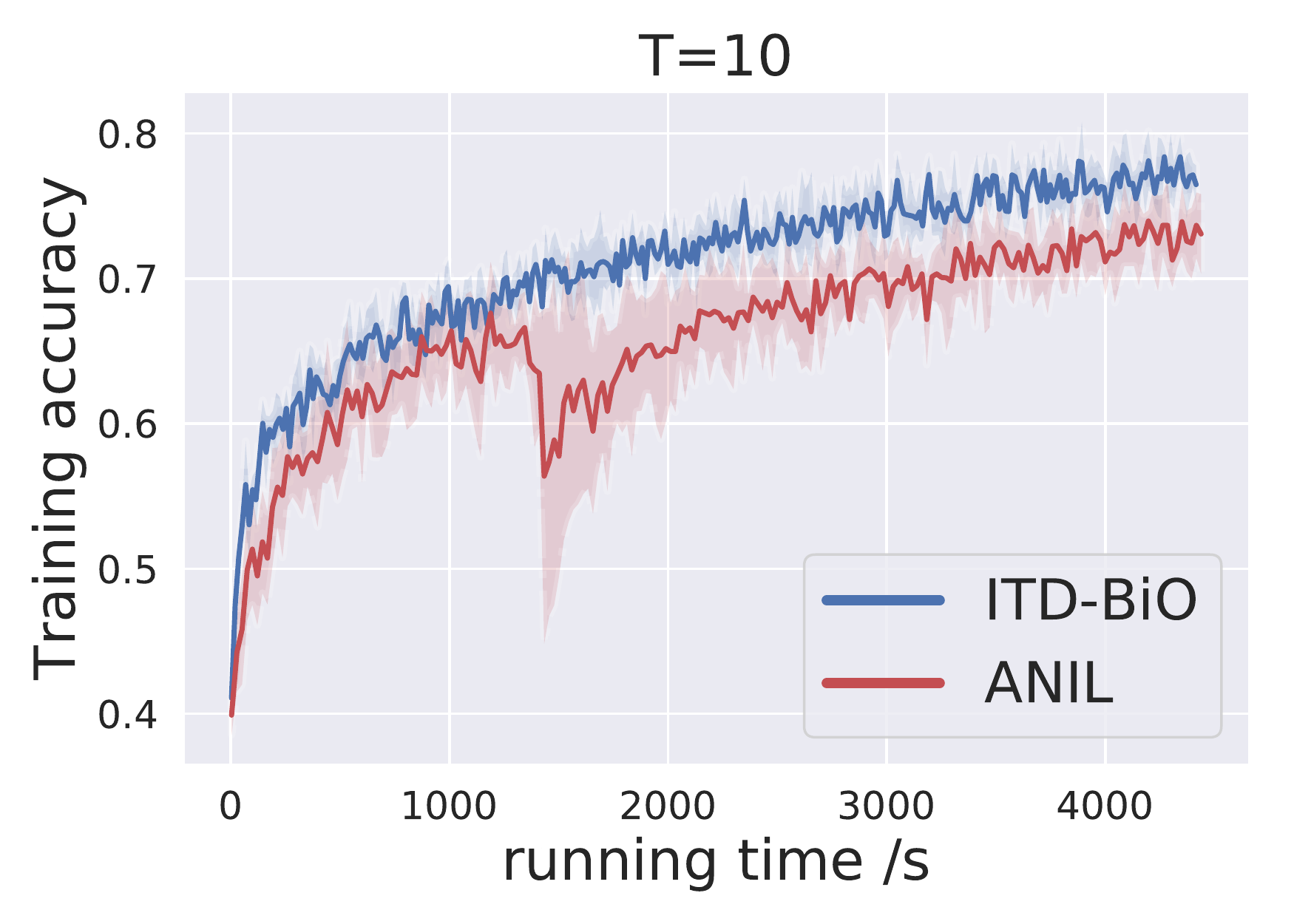}\includegraphics[width=41mm]{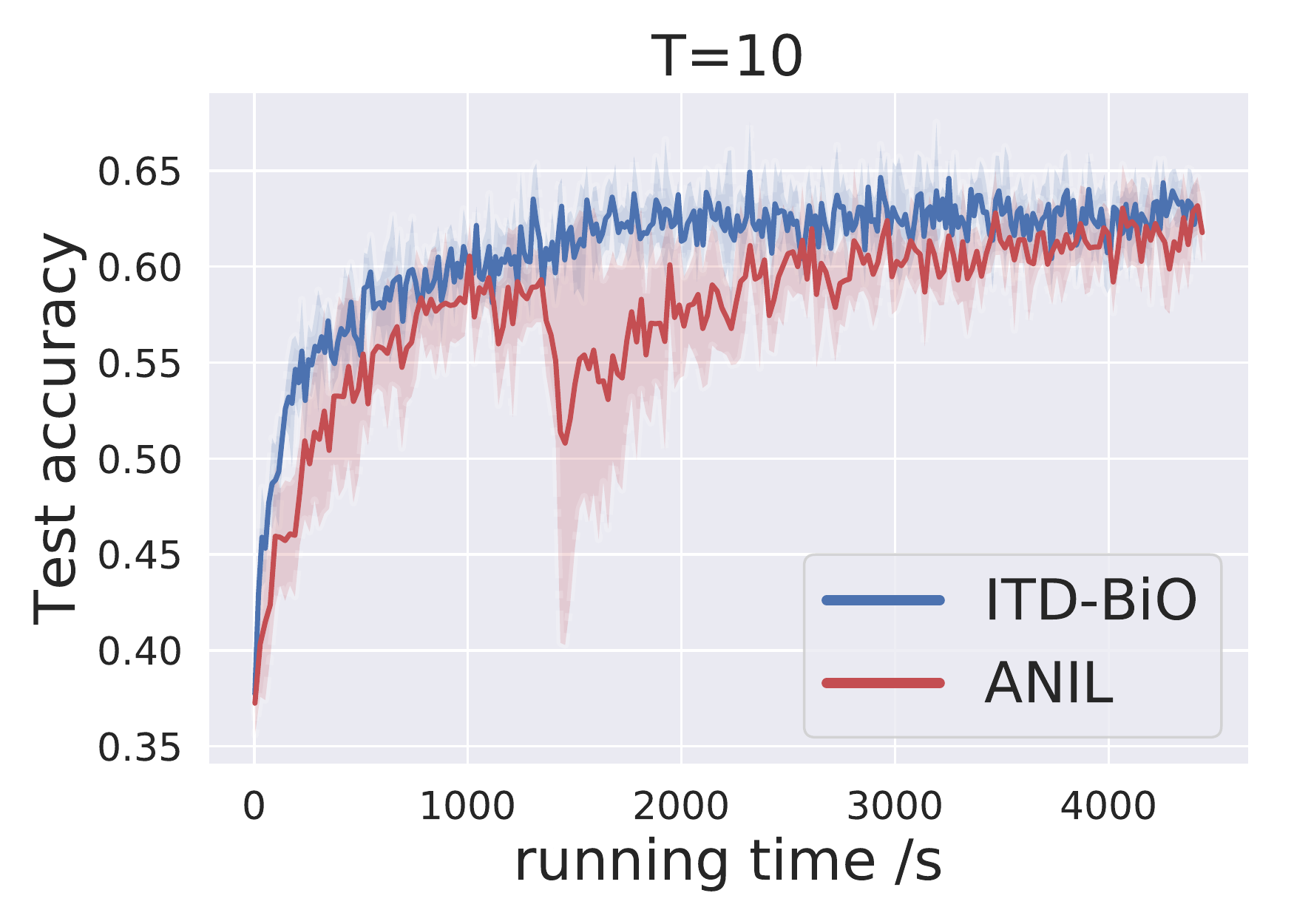}} 
	\subfigure[$T=20$, FC100 dataset]{\label{fig1:di}\includegraphics[width=38mm]{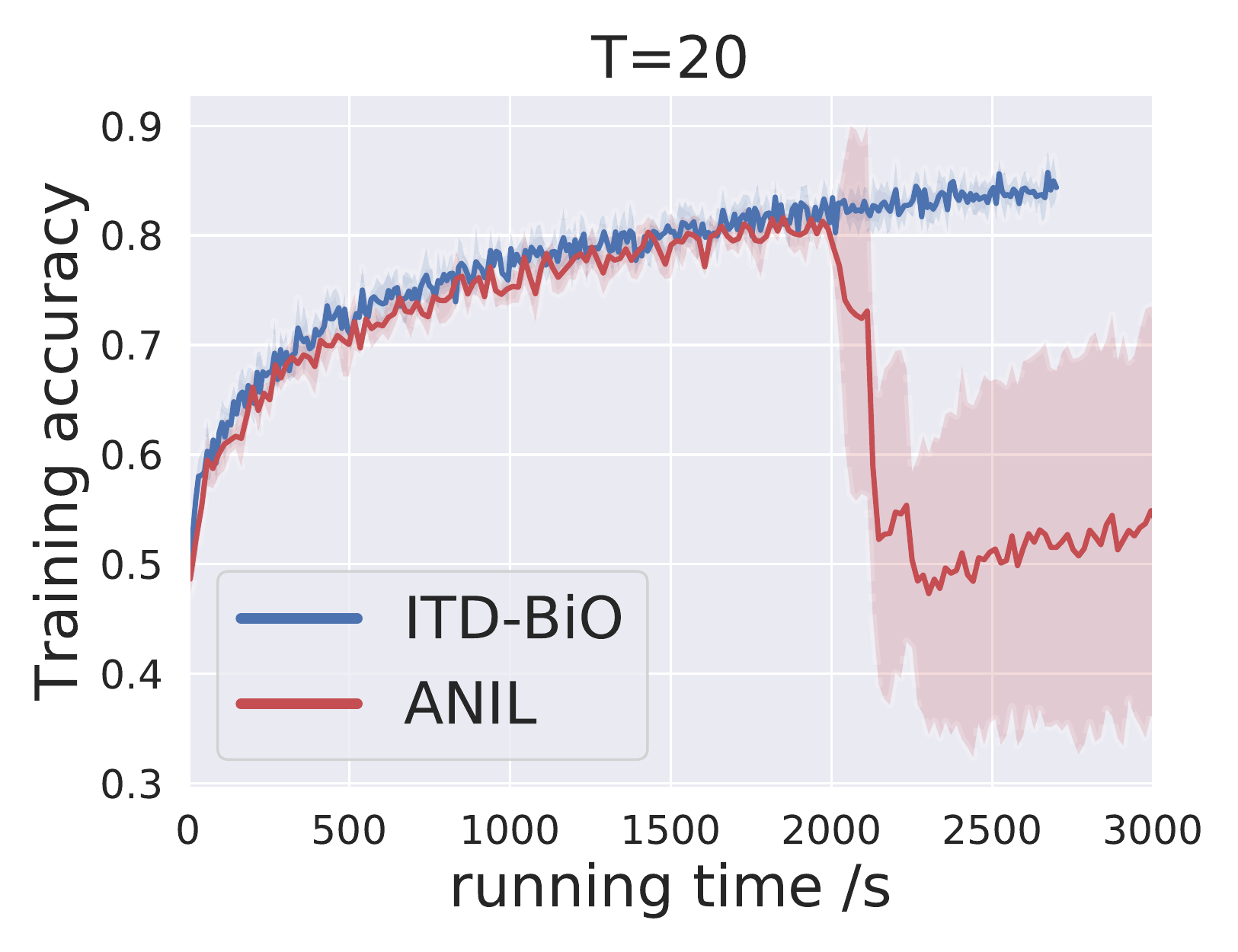}\includegraphics[width=38mm]{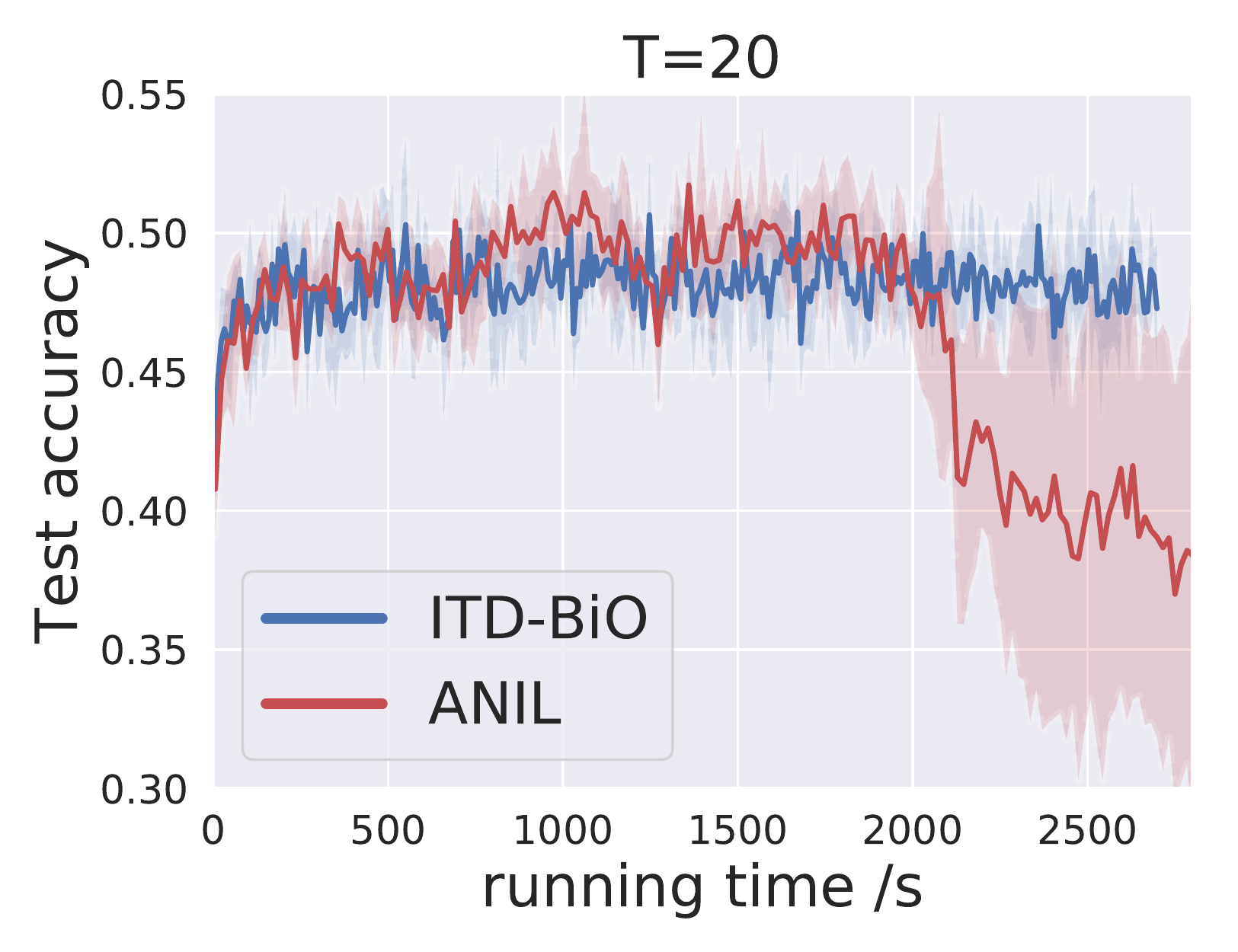}}  
	\vspace{-0.2cm}
	\caption{Comparison of ITD-BiO and ANIL with a relatively large inner-loop iteration number $T$.}\label{figure:resultlg}
	  \vspace{-0.2cm}
\end{figure*}

The model training  takes a bilevel procedure. In the lower-level stage, building on the embedded features, the base learner of task $\mathcal{T}_i$ searches $w_i^*$ as the minimizer of its 
loss 
  over a training set $\gS_i$. In the upper-level stage, the meta-learner evaluates the minimizers $w_i^*,i=1,...,m$ on held-out test sets, and optimizes $\phi$ of the embedding model over all tasks. Let $\widetilde w=(w_1,...,w_m)$ denote all task-specific parameters. Then, the objective function is given by 
 \begin{align}\label{obj:meta}
 &\min_{\phi} \gL_{\gD} (\phi,\widetilde w^{*})=\frac{1}{m}\sum_{i=1}^m\underbrace{\frac{1}{|\gD_i|}\sum_{\xi\in\gD_i}\mathcal{L}(\phi,w_i^*;\xi)}_{\gL_{\gD_i}(\phi,w_i^*)\text{: task-specific upper-level loss}} \nonumber
 \\& \;\mbox{s.t.} \; \widetilde w^* = \argmin_{\widetilde w} \gL_{\gS} (\phi,\widetilde w)=\frac{\sum_{i=1}^m\gL_{\gS_i}(\phi,w_i)
 }{m},
 \end{align} 
 where $\gL_{\gS_i}(\phi,w_i)= \frac{1}{|\gS_i|}\sum_{\xi\in\gS_i}\mathcal{L}(\phi,w_i;\xi) + \gR(w_i)$ with a strongly-convex regularizer $\gR(w_i)$, e.g., $L^2$, and $\gS_i,\gD_i$ are the training and test datasets of task $\mathcal{T}_i$. Note that the lower-level problem is equivalent to solving each $w^*_i$ as a minimizer of the task-specific loss $\gL_{\gS_i}(\phi,w_i)$ for $i=1,...,m$.  In practice, $w_i$ often corresponds to the parameters of the last {\em linear} layer of a neural network and $\phi$ are the parameters of the remaining layers (e.g., $4$ convolutional layers in~\citealt{bertinetto2018meta,ji2020convergence}), and hence the lower-level function is {\em strongly-convex} w.r.t. $\widetilde w$ and the upper-level function $\gL_{\gD} (\phi,\widetilde w^{*}(\phi))$ is generally nonconvex w.r.t. $\phi$. In addition, due to the small sizes of datasets $\gD_i$ and $\gS_i$ in few-shot learning, all updates for each task $\gT_i$ use {\em full gradient descent} without data resampling. 
 As a result,  AID-BiO and ITD-BiO in~\Cref{alg:main_deter} can be applied here.  
In some applications where the number $m$ of tasks is large,  it is more efficient to sample a batch $\gB$ of i.i.d.\ tasks from $\{\mathcal{T}_i,i=1,...,m\}$ at each meta (outer) iteration, and optimizes the mini-batch versions $ \gL_{\gD} (\phi,\widetilde w;\gB) = \frac{1}{|\gB|}\sum_{i\in\gB}\gL_{\gD_i}(\phi,w_i)$ and $\gL_{\gS} (\phi,\widetilde w;\gB)=\frac{1}{|\gB|}\sum_{i\in\gB}\gL_{\gS_i}(\phi,w_i)$ instead. 

We next  provide the convergence result of ITD-BiO for this case, and that of AID-BiO can be similarly derived.
\begin{theorem}\label{th:meta_learning}
Suppose Assumptions~\ref{assum:geo},~\ref{ass:lip} and \ref{high_lip} hold and suppose each task loss $\gL_{\gS_i}(\phi,\cdot)$ is $\mu$-strongly-convex. Choose the same parameters $\beta,D$ as in~\Cref{th:determin}. Then, we have
\begin{align*}
\frac{1}{K}\sum_{k=0}^{K-1}\mathbb{E}\| \nabla \Phi(\phi_k)\|^2 \leq& \mathcal{O} \left(\frac{1}{K}+\frac{\kappa^2}{|\gB|}\right).
\end{align*}
\end{theorem}
\vspace{-0.3cm}
 \Cref{th:meta_learning} shows  that compared to the full batch case (i.e., without task sampling) in~\cref{obj:meta}, task sampling introduces a variance term $\mathcal{O}(\frac{1}{|\gB|})$ due to the stochastic nature of the algorithm.  



\subsection{Experiments}
 To validate our theoretical results for deterministic bilevel optimization, we compare the performance among the following four algorithms: ITD-BiO, AID-BiO-constant (AID-BiO with a constant number of inner-loop steps as in our analysis), AID-BiO-increasing (AID-BiO with an increasing number of inner-loop steps under analysis in~\citealt{ghadimi2018approximation}), and two popular meta-learning algorithms MAML\footnote{MAML consists of an inner loop for task  adaptation and an outer loop for meta initialization training.}~\citep{finn2017model} and ANIL\footnote{ANIL refers to almost no inner loop, which is an efficient MAML variant with task adaption on the last-layer of parameters.}~\citep{raghu2019rapid}. We conduct experiments over a 5-way 5-shot task on two datasets: FC100 and miniImageNet. The results are averaged over 10 trials with different random seeds. Due to the space limitations, we provide the model architectures and hyperparameter settings in~\Cref{appen:meta_learning}.

 \begin{figure*}[ht]
	\centering  
	\subfigure[Test loss and test accuracy v.s. running time]{\label{figure:lr}\includegraphics[width=40mm]{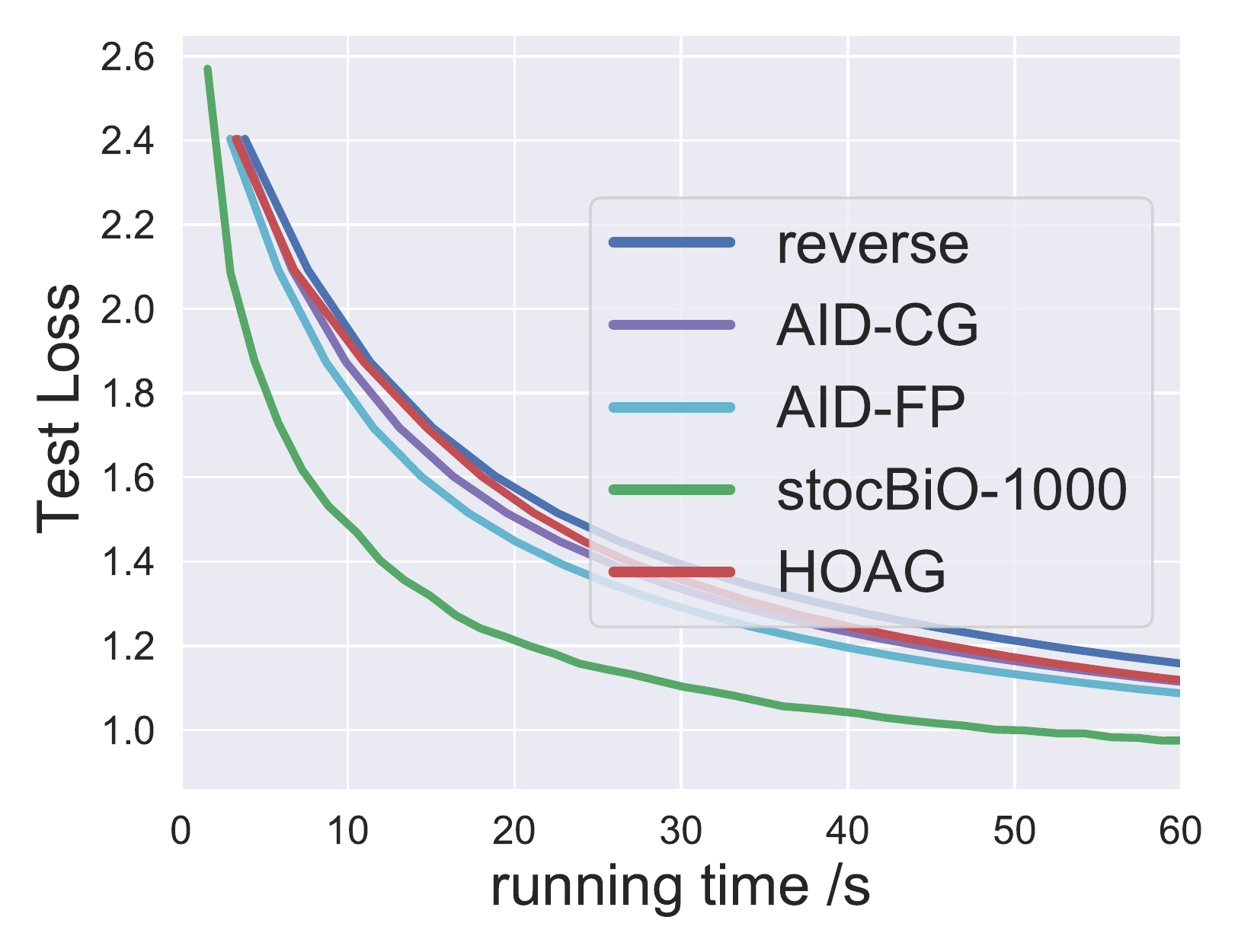}\includegraphics[width=40mm]{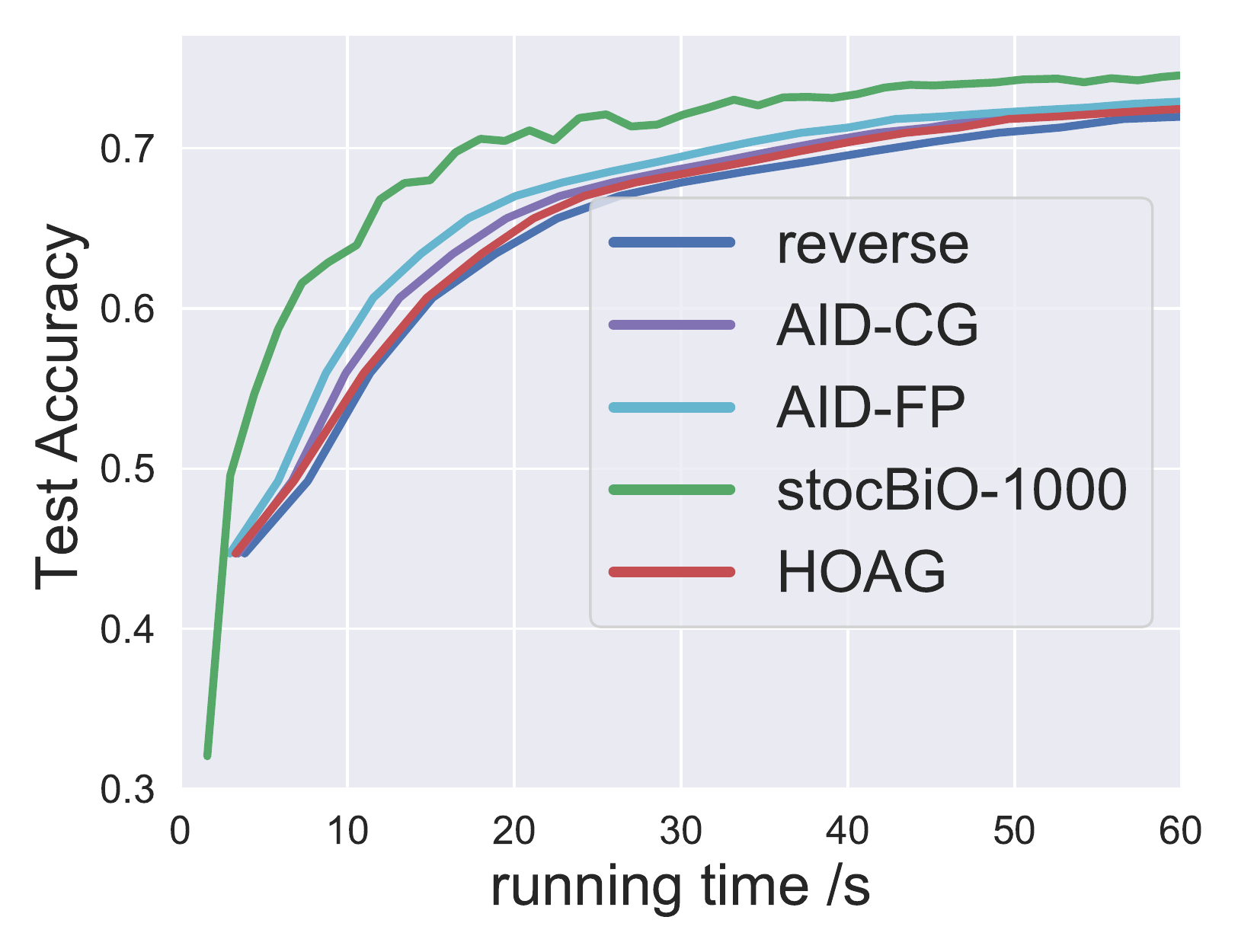}}    
	\subfigure[Convergence rate with different batch sizes]{\label{figure:batch}\includegraphics[width=40mm]{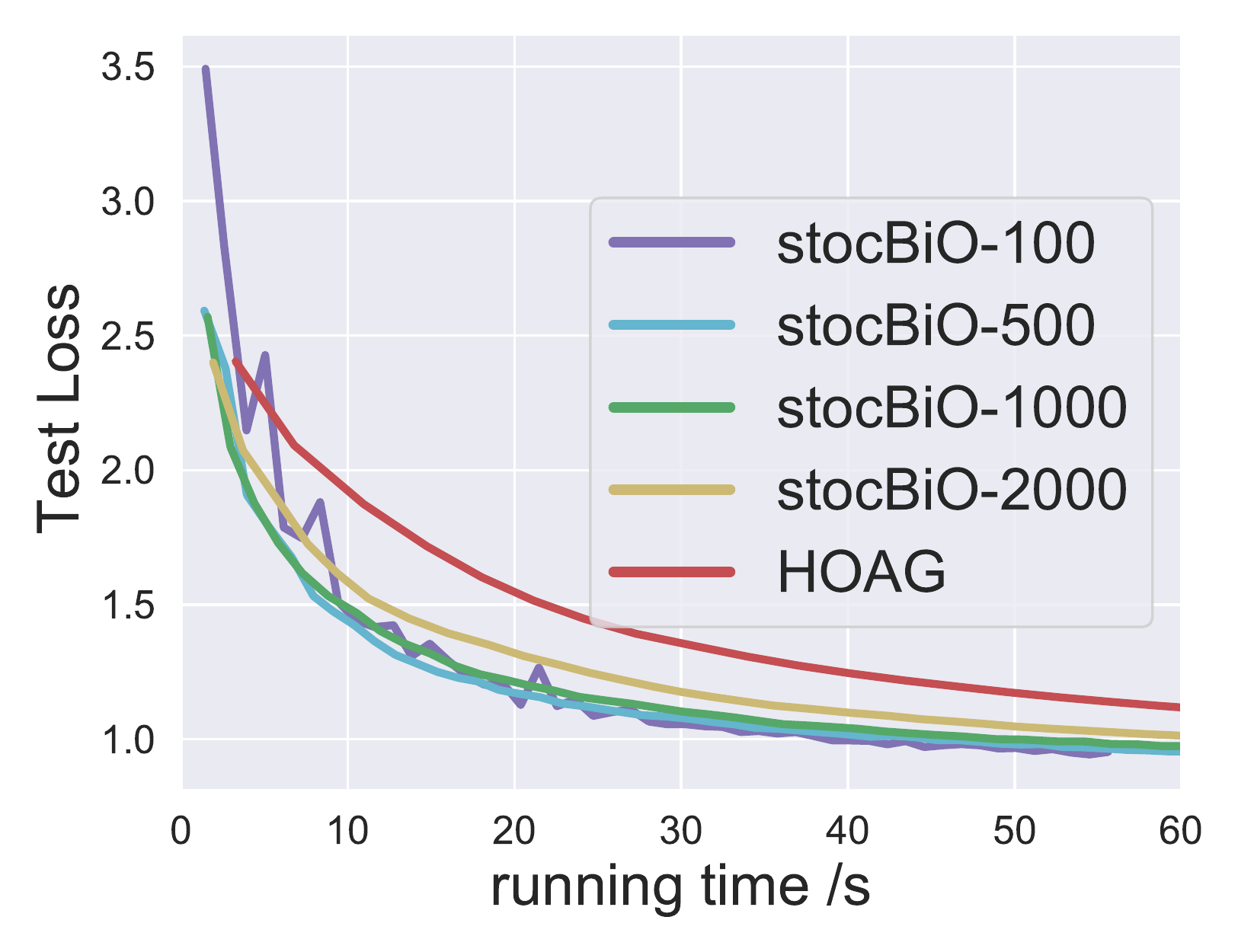}\includegraphics[width=40mm]{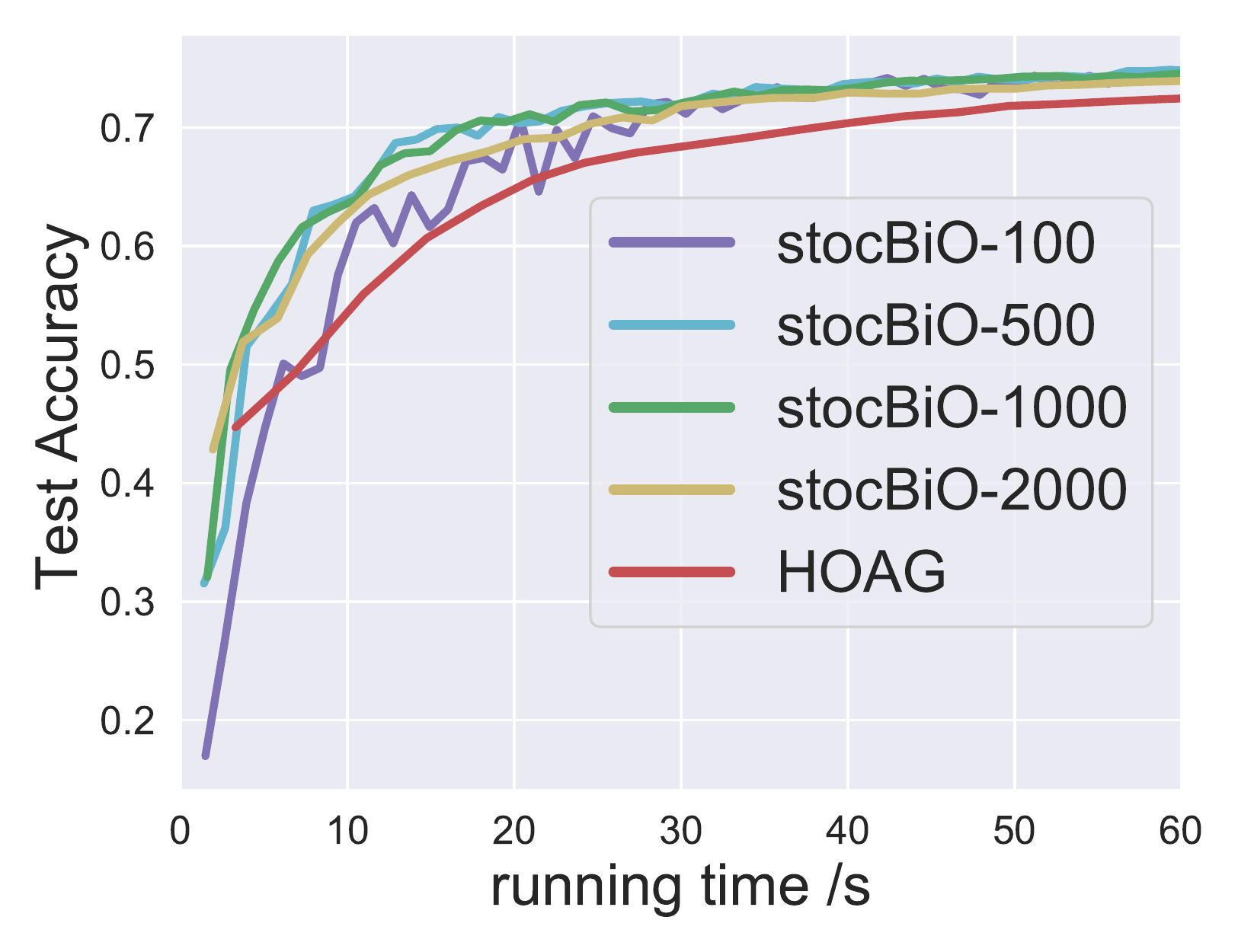}}  
	\vspace{-0.2cm}
	\caption{Comparison of various stochastic bilevel algorithms on logistic regression on 20 Newsgroup dataset.}\label{figure:newfigure}
	  \vspace{-0.3cm}
\end{figure*}

  \begin{figure*}[ht]
  \vspace{-2mm}
	\centering  
	\subfigure[Corruption rate $p=0.1$]{\label{fig2:c}\includegraphics[width=40mm]{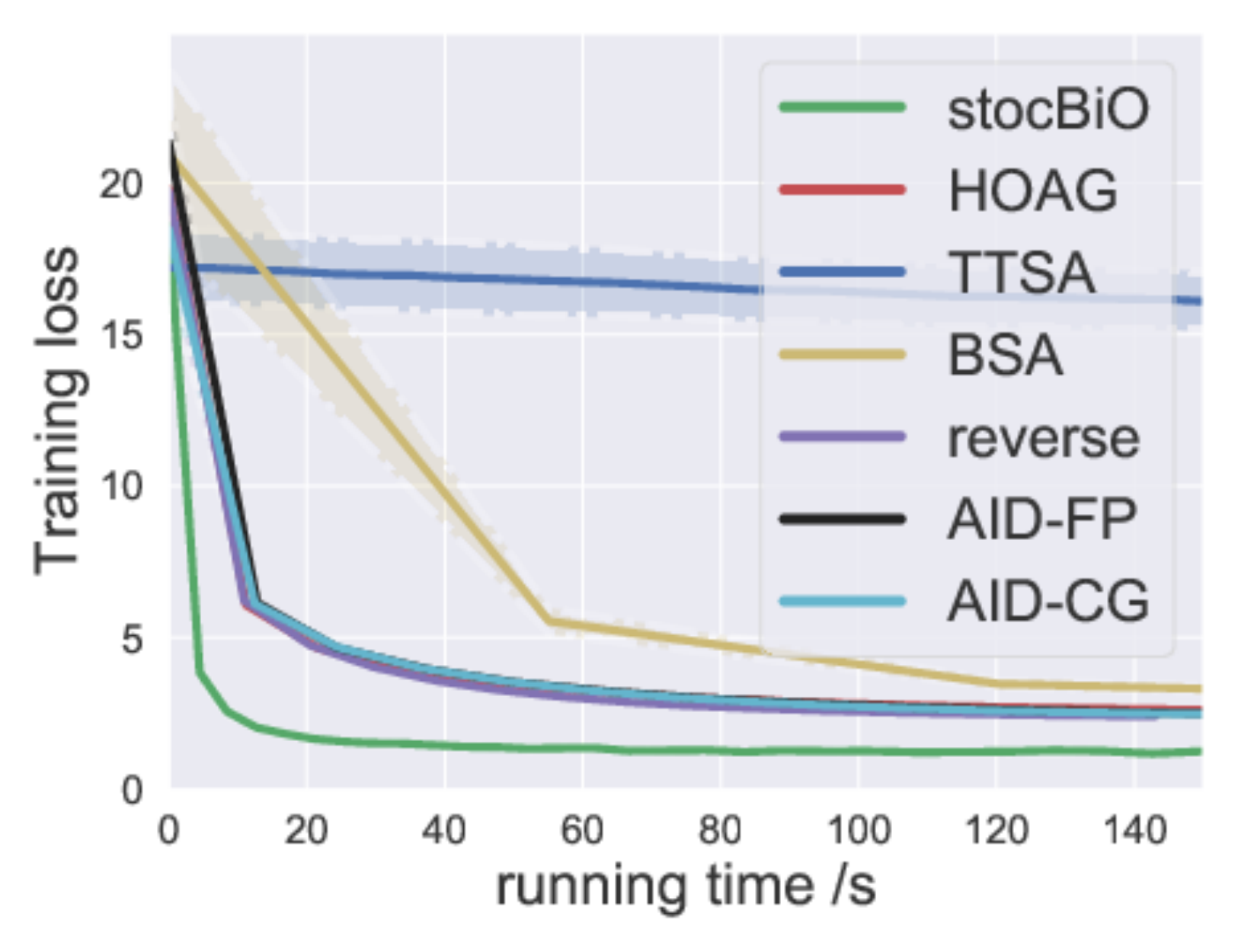}\includegraphics[width=40mm]{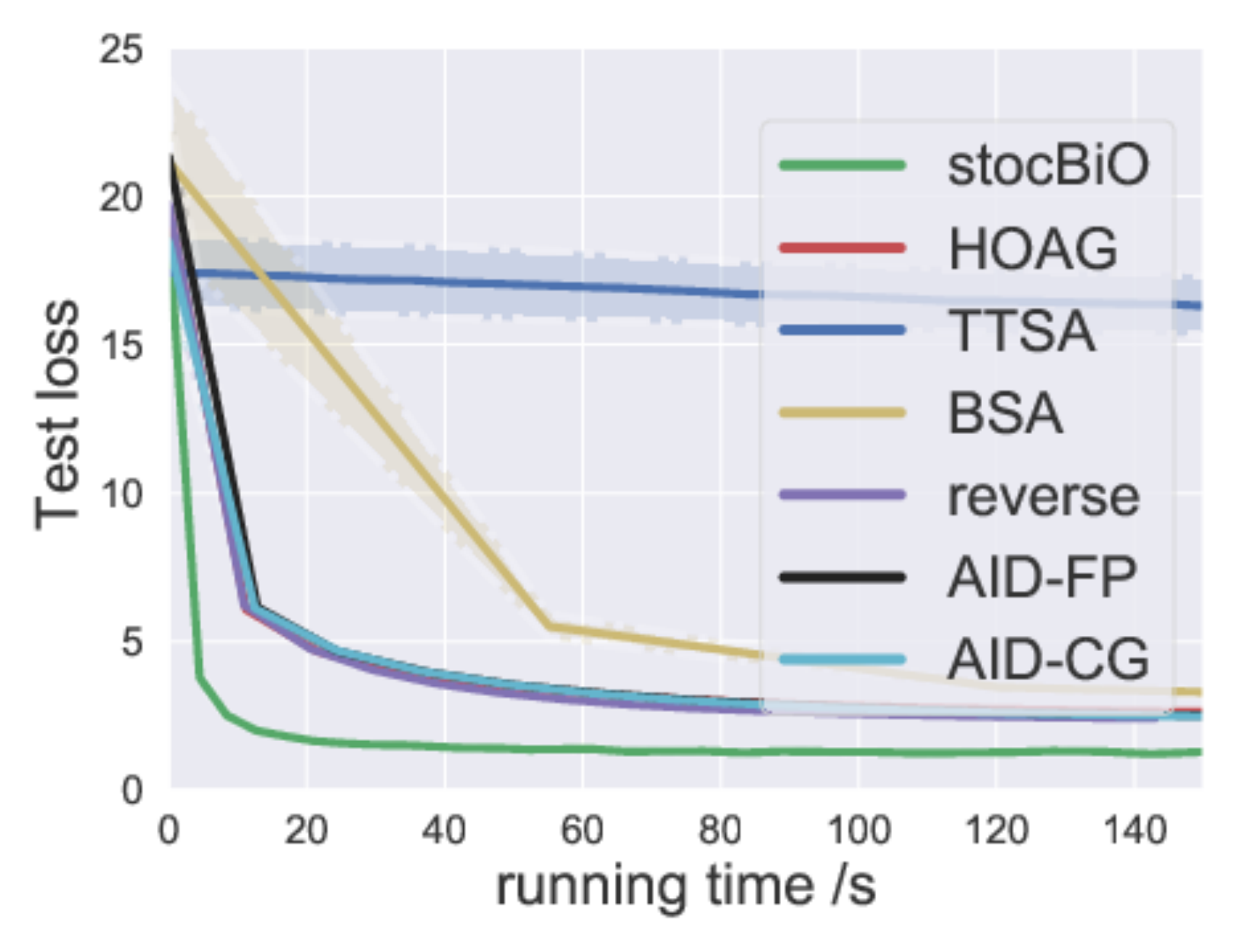}}    
	\subfigure[Corruption rate $p=0.4$]{\label{fig2:b}\includegraphics[width=40mm]{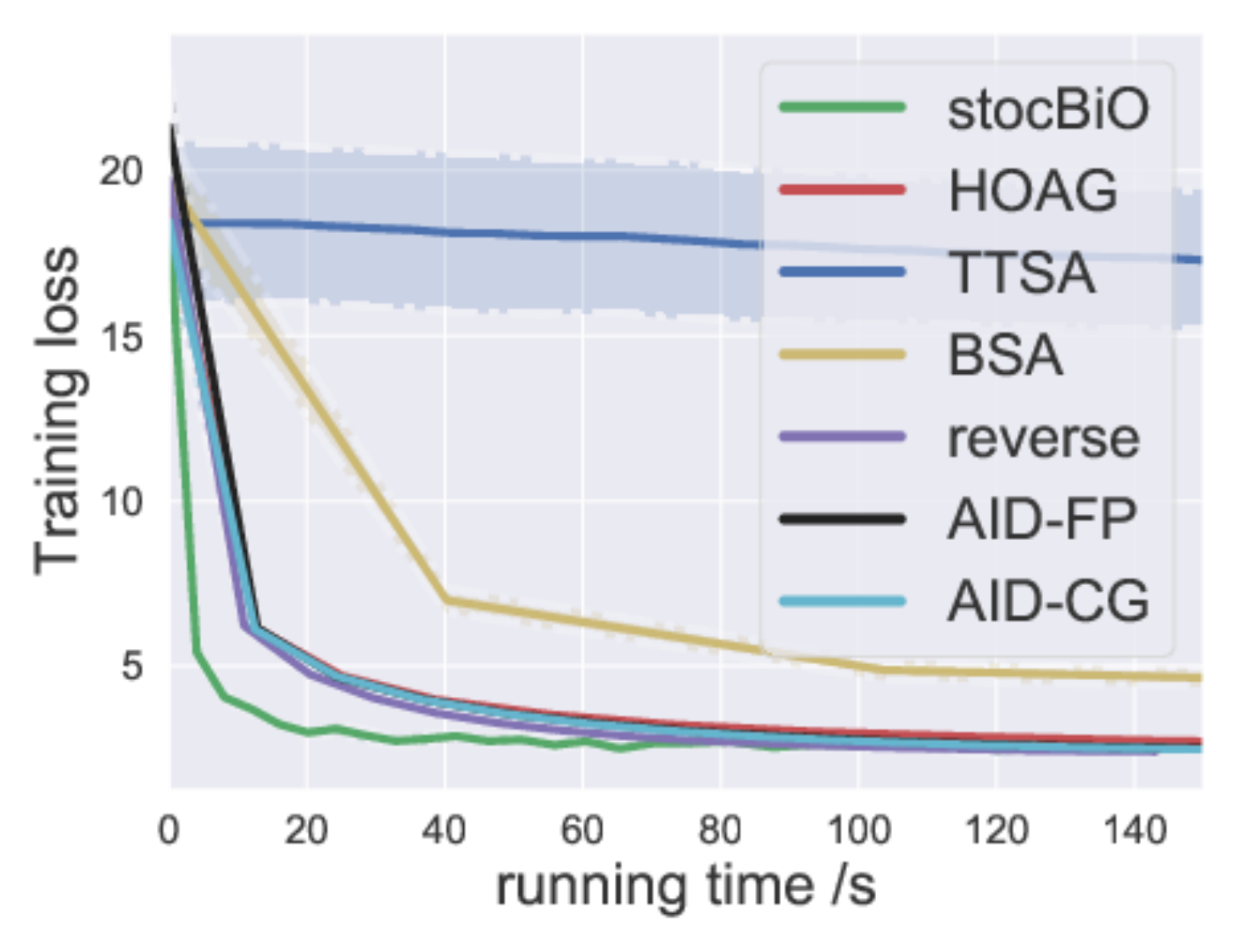}\includegraphics[width=40mm]{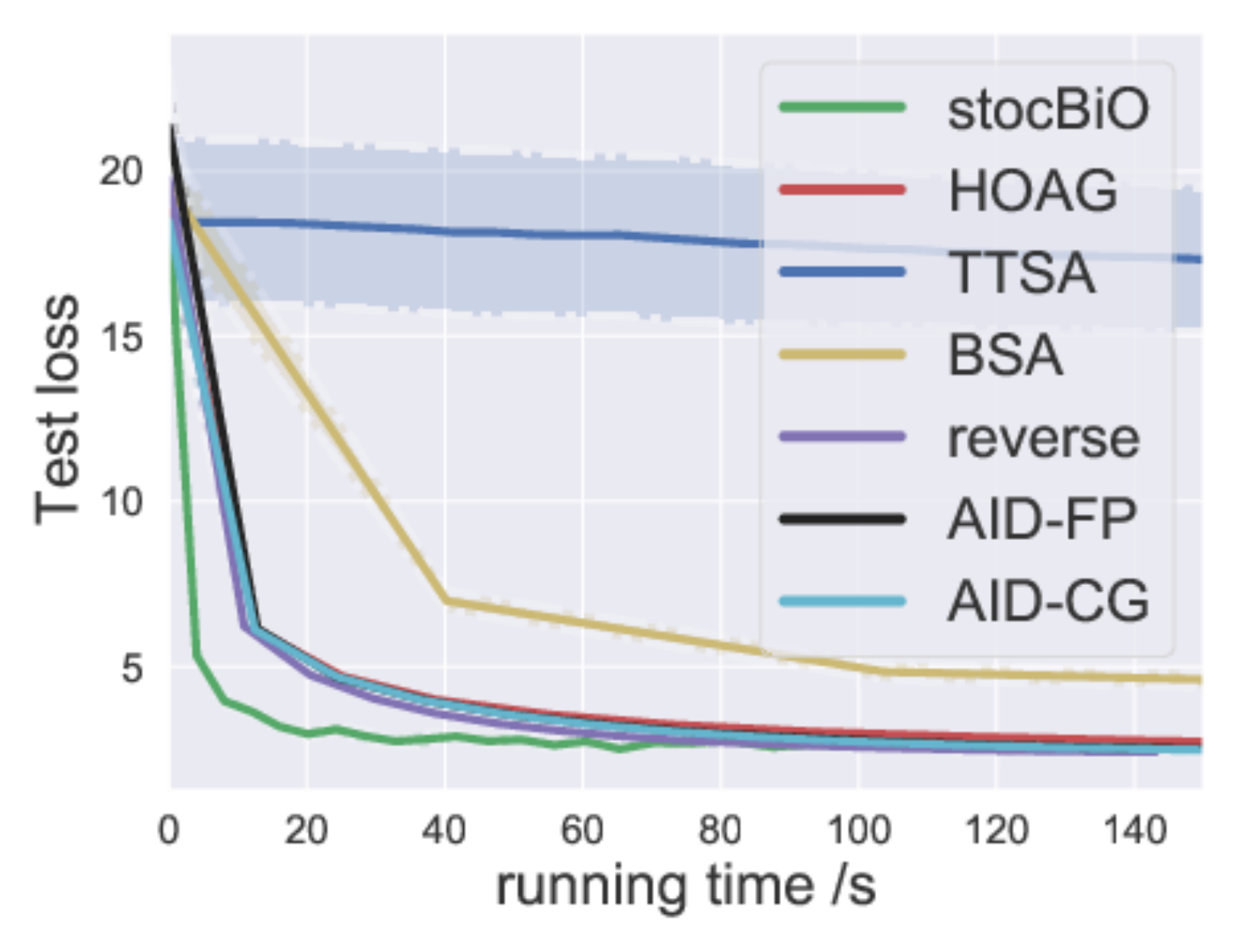}}  
	\vspace{-0.2cm}
	\caption{Comparison of various stochastic bilevel algorithms on hyperparameter optimization at different corruption rates.   For each corruption rate $p$, left plot: training loss v.s. running time; right plot: test loss v.s. running time.}\label{fig:hyper}
	  \vspace{-0.3cm}
\end{figure*}

It can be seen from \Cref{fig:strfc100} that for both the miniImageNet and FC100 datasets, AID-BiO-constant converges  faster  than AID-BiO-increasing in terms of both the training accuracy and test accuracy, and achieves a better final test accuracy than ANIL and MAML. This demonstrates the superior improvement of our developed analysis over existing analysis in~\citealt{ghadimi2018approximation} for AID-BiO algorithm.  
Moreover,  it can be observed that AID-BiO is slightly faster than ITD-BiO in terms of  the training accuracy and test accuracy. This is in consistence with our theoretical results. 

We also compare the robustness between the bilevel optimizer ITD-BiO (AID-BiO performs similarly to ITD-BiO in terms of the convergence rate) and ANIL when  the number $T$ (i.e., $D$ in \Cref{alg:main_deter})  of inner-loop steps is relatively large. 
It can be seen from~\Cref{figure:resultlg}  that when the number of inner-loop steps is large, i.e., $T=10$ for miniImageNet and $T=20$ for FC100, the bilevel optimizer ITD-BiO converges stably with a small variance, whereas ANIL suffers from a sudden descent at 1500s on miniImageNet  and even diverges after 2000s on FC100.

\section{Applications to Hyperparameter Optimization}

The goal of hyperparameter optimization~\citep{franceschi2018bilevel,feurer2019hyperparameter} is to search for representation or regularization parameters  $\lambda$ to minimize the validation error evaluated over the learner's parameters $w^*$,  
 where $w^*$ is the minimizer of the inner-loop regularized  training error. Mathematically, the objective function is given by 
\begin{align}\label{obj:hyper_opt}
&\min_\lambda \gL_{\gD_{\text{val}}}(\lambda) = \frac{1}{|\gD_{\text{val}}|}\sum_{\xi\in \gD_{\text{val}}} \gL(w^*; \xi) \nonumber
\\& \;\mbox{s.t.} \; w^*= \argmin_{w} \underbrace{\frac{1}{|\gD_{\text{tr}}|}\sum_{\xi\in \gD_{\text{tr}}} \big(\gL(w,\lambda;\xi)  + \gR_{w,\lambda}\big)}_{\gL_{\gD_{\text{tr}}}(w,\lambda)},
\end{align}
where $\gD_{\text{val}}$ and $\gD_{\text{tr}}$ are validation and training data,  $\gL$ is the loss, and $\gR_{w,\lambda}$ is a regularizer. In practice,  the lower-level function $\gL_{\gD_{\text{tr}}}(w,\lambda)$ is often strongly-convex w.r.t.~$w$. For example, for the data hyper-cleaning application proposed by~\citealt{franceschi2018bilevel,shaban2019truncated}, the predictor is modeled by a linear classifier, and  the 
loss function $\gL(w;\xi)  $ is convex w.r.t.~$w$ and $\gR_{w,\lambda}$ is a strongly-convex regularizer, e.g., $L^2$ regularization.  
The sample sizes of  $\gD_{\text{val}}$ and $\gD_{\text{tr}}$ are often large, and stochastic algorithms are preferred for achieving better efficiency. 
As a result, the above hyperparameter optimization falls into the stochastic bilevel optimization we study in~\cref{objective}, and we can apply the proposed stocBiO here. Furthermore, \Cref{th:nonconvex} establishes its performance guarantee. 
\subsection{Experiments}
We compare our proposed {\bf stocBiO} with the following baseline bilevel optimization algorithms. 
\begin{list}{$\bullet$}{\topsep=0.ex \leftmargin=0.12in \rightmargin=0.in \itemsep =0.0in}
\item {\bf BSA}~\citep{ghadimi2018approximation}: implicit gradient based stochastic bilevel optimizer via single-sample sampling.
\item {\bf TTSA}~\citep{hong2020two}: two-time-scale stochastic optimizer via single-sample data sampling. 
\item {\bf HOAG}~\citep{pedregosa2016hyperparameter}: a hyperparameter optimization algorithm with approximate gradient. We use the implementation in the repository~ {\small\url{https://github.com/fabianp/hoag}}.  
\item {\bf reverse}~\citep{franceschi2017forward}: an iterative differentiation based method that approximates the hypergradient via backpropagation. We use its implementation in {\small \url{https://github.com/prolearner/hypertorch}}.  
 \item {\bf AID-FP}~\citep{grazzi2020iteration}: AID with the  fixed-point method. We use its implementation in {\small \url{https://github.com/prolearner/hypertorch}} 
\item  {\bf AID-CG}~\citep{grazzi2020iteration}: AID with the conjugate gradient method. We use its implementation in {\small \url{https://github.com/prolearner/hypertorch}}.  
 \end{list}
We demonstrate the effectiveness of the proposed stocBiO algorithm on two experiments: data hyper-cleaning and logistic regression.  Due to the space limitations, we provide the details of the objective functions and hyperparameter settings in~\Cref{appen:hyperoptimization}. 

\vspace{0.1cm}
\noindent{\bf Logistic Regression on 20 Newsgroup:} 
  As shown in \Cref{figure:lr}, the proposed stocBiO achieves the fastest convergence rate as well as the best test accuracy among all comparison algorithms. This demonstrates the practical advantage of our proposed algorithm stocBiO. Note that we do not include BSA and TTSA in the comparison, because they converge too slowly with a large variance, and are much worse than the other competing algorithms. In addition, we investigate the impact of the batch size on the performance of our stocBiO in \Cref{figure:batch}. It can be seen that stocBiO outperforms HOAG under the batch sizes of $100,500,1000,2000$. This shows that the performance of stocBiO is not very sensitive to the batch size, and hence the tuning of the batch size is easy to handle in practice.

\vspace{0.1cm}
\noindent{\bf Data Hyper-Cleaning on MNIST.} It can be seen from Figures \ref{fig:hyper} and \ref{fig:hyper_appen} that our proposed stocBiO algorithm achieves the fastest convergence rate among all competing algorithms in terms of both the training loss and the test loss. It is also observed that such an improvement is more significant when the corruption rate $p$ is smaller.  We note that the stochastic algorithm TTSA converges very slowly with a large variance. This is because TTSA  updates the costly outer loop more frequently than other algorithms,  and has a larger variance due to the single-sample data sampling. As a comparison, our stocBiO has a much smaller variance for hypergradient estimation as well as a much faster convergence rate.  This validates our theoretical results in~\Cref{th:nonconvex}. 

  \begin{figure}[ht]
  \vspace{-2mm}
	\centering  
	\subfigure{\includegraphics[width=40mm]{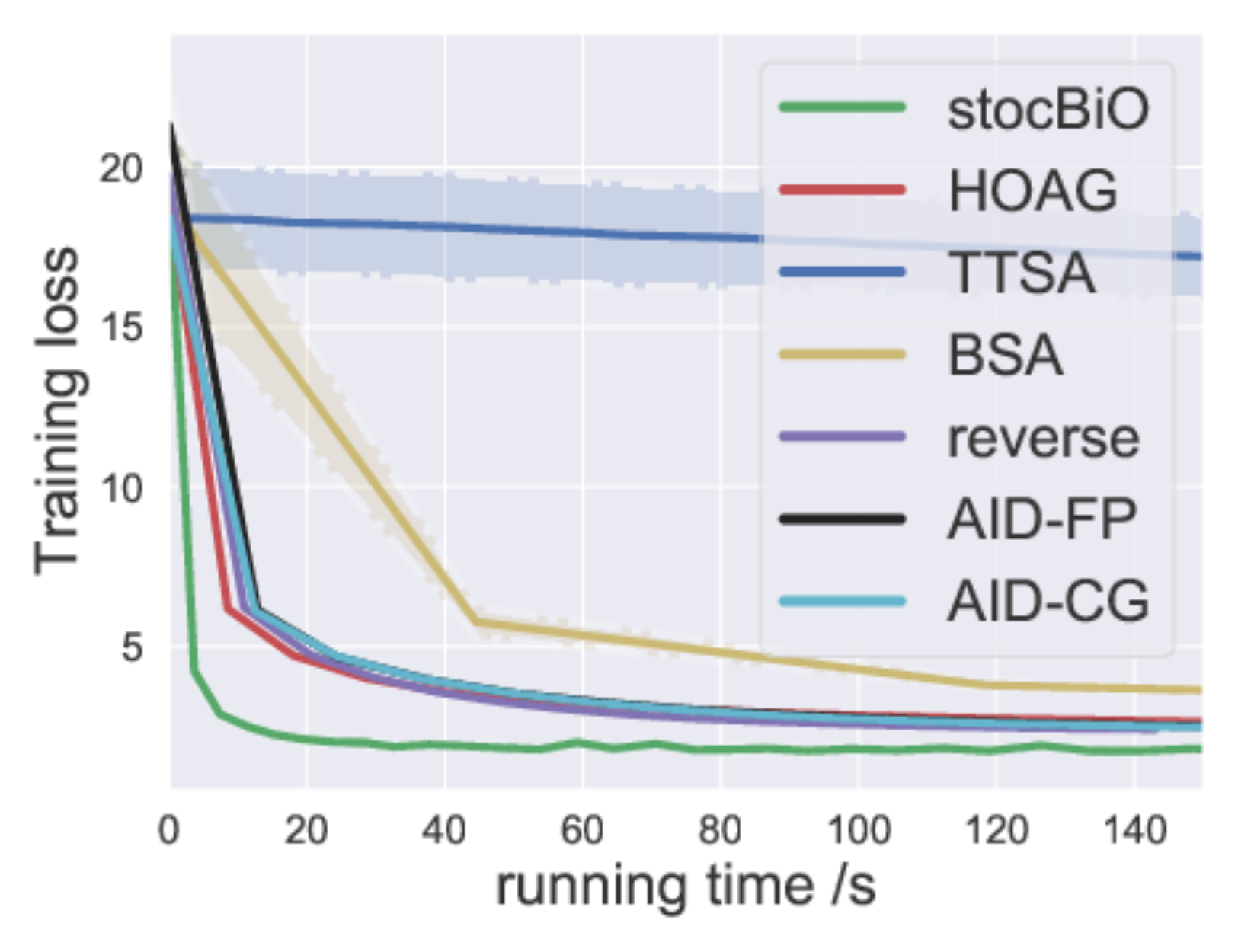}}
	\subfigure{\includegraphics[width=40mm]{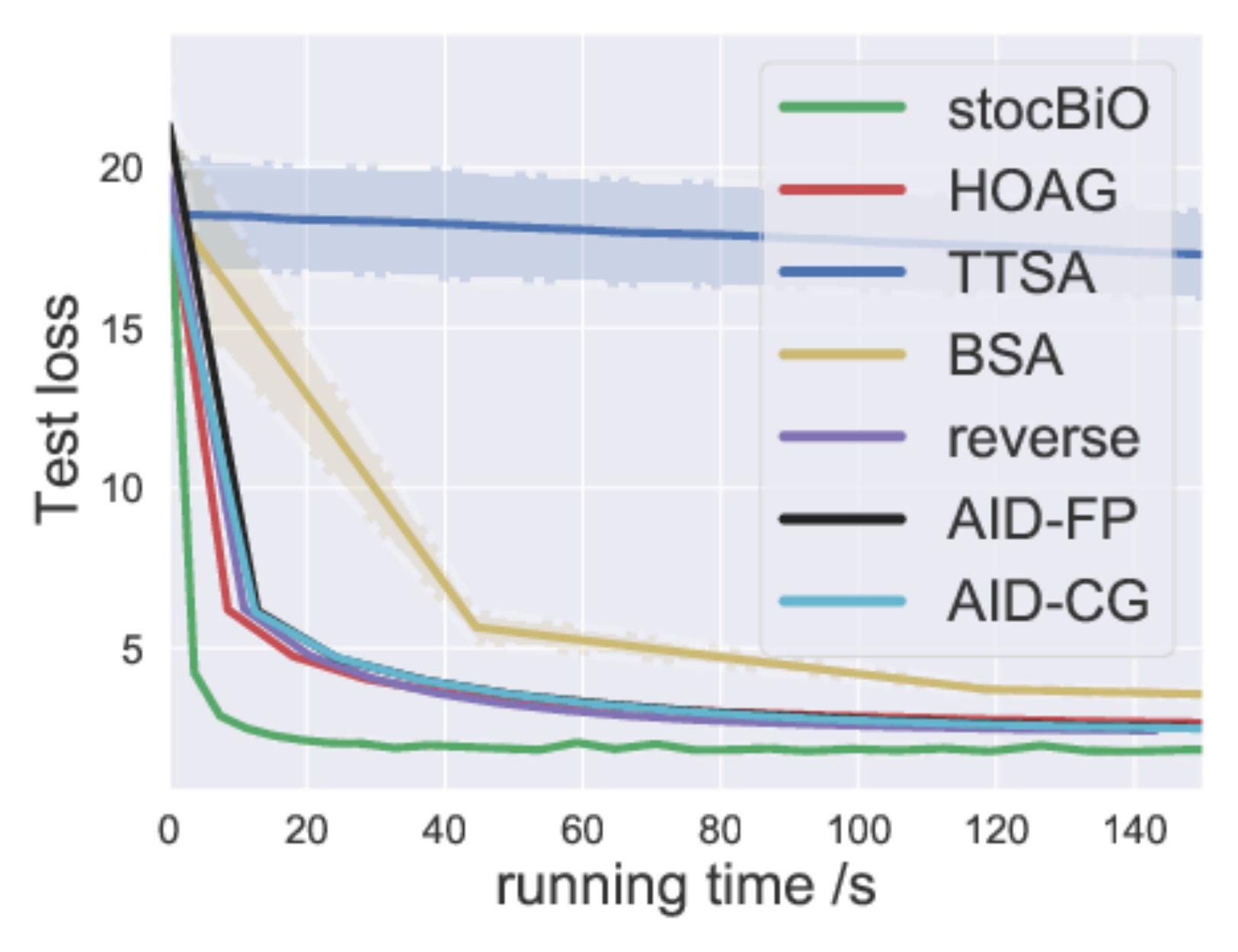}} 
	\vspace{-0.2cm}
	\caption{Convergence of algorithms at corruption rate $p=0.2$.   
	}\label{fig:hyper_appen}
	  \vspace{-0.4cm}
\end{figure}

\vspace{-0.1cm}
\section{Conclusion} 
\vspace{-0.1cm}
In this paper, we develop a general and enhanced convergence rate analysis for the nonconvex-strongly-convex bilevel deterministic optimization, and propose a novel algorithm for the stochastic setting and show that its computational complexity outperforms the best known results orderwisely. Our results also provide the theoretical guarantee for various bilevel optimizers in meta-learning and hyperparameter optimization. Our experiments validate our theoretical results and demonstrate the superior performance of the proposed algorithm. 
We anticipate that the convergence rate analysis that we develop will be useful for analyzing other  bilevel optimization problems with different loss geometries, and the proposed algorithms will be useful for other applications such as reinforcement learning and Stackelberg game.

\vspace{-0.1cm}
\section*{Acknowledgements}
The work was supported in part by the U.S. National Science Foundation under the grants CCF-1909291 and CCF-1900145.

\bibliography{icml2021BiO}
\bibliographystyle{icml2021}

\newpage
\onecolumn
\appendix
{\Large\bf Supplementary Materials}

\section{Further Specifications on Meta-Learning Experiments}\label{appen:meta_learning}
\subsection{Datasets and Model Architectures}
FC100~\citep{oreshkin2018tadam} is a dataset derived from CIFAR-100~\citep{krizhevsky2009learning}, and contains $100$ classes with each class consisting of $600$ images of size $32\time 32$. Following~\citealt{oreshkin2018tadam}, these $100$ classes are split  into $60$ classes for meta-training, $20$ classes for meta-validation, and $20$ classes for meta-testing.   For all comparison algorithms, we use a $4$-layer convolutional neural networks (CNN) with four convolutional blocks, in which each convolutional block contains a $3\times 3$ convolution ($\text{padding}=1$, $\text{stride}=2$), batch normalization, ReLU activation, and $2\times 2$
max pooling. Each convolutional layer has $64$ filters. 

The miniImageNet dataset~\citep{vinyals2016matching} is generated from ImageNet~\cite{russakovsky2015imagenet}, and consists of $100$ classes with each class containing $600$ images of size $84\times 84$. Following the repository~\cite{learn2learn2019}, we partition these classes into $64$ classes for meta-training, $16$ classes for meta-validation, and $20$ classes for meta-testing.
Following the repository~\citep{learn2learn2019}, we use a four-layer CNN with four convolutional blocks, where each block sequentially consists of  a $3\times 3$ convolution, batch normalization, ReLU activation, and $2\times 2$
max pooling. Each convolutional layer has $32$ filters. 
\subsection{Implementations and Hyperparameter Settings}
We adopt the existing implementations in the repository~\citep{learn2learn2019} for ANIL and MAML. 
 For all algorithms, we adopt Adam~\citep{kingma2014adam} as the optimizer for the outer-loop update. 
 
\vspace{0.2cm}
\noindent {\bf Parameter selection for the experiments in~\Cref{fig1:a}:} For ANIL and MAML, we adopt the suggested hyperparameter selection in the repository~\citep{learn2learn2019}. In specific, for ANIL, we choose the inner-loop stepsize as $0.1$, the outer-loop (meta) stepsize as $0.002$, the task sampling size as $32$, and the number of inner-loop steps as $5$. For MAML, we choose the inner-loop stepsize as $0.5$, the outer-loop stepsize as $0.003$, the task sampling sizeas $32$, and the number of inner-loop steps as $3$. 
For ITD-BiO, AID-BiO-constant and AID-BiO-increasing, we use a grid search to choose the inner-loop stepsize from $\{0.01,0.1,1,10\}$, the task sampling size from $\{32,128,256\}$, and  the  outer-loop stepsize from $\{10^{i},i=-3,-2,-1,0,1,2,3\}$, where values that achieve the lowest loss after a fixed running time are selected.  
 For ITD-BiO and AID-BiO-constant, we choose the  number of inner-loop steps from $\{5,10,15,20,50\}$, and for AID-BiO-increasing, we choose the number of inner-loop steps as $\lceil c{(k+1)}^{1/4}\rceil$ as adopted by the analysis in \citealt{ghadimi2018approximation}, where we choose $c$ from $\{0.5,2,5,10,50\}$.
For both AID-BiO-constant and AID-BiO-increasing, we choose the number $N$ of CG steps for solving the linear system from $\{5,10,15\}$.


\vspace{0.2cm}
\noindent 
{\bf Parameter selection for the experiments in~\Cref{fig1:b}:}  For ANIL and MAML, we adopt the suggested hyperparameter selection in the repository~\citep{learn2learn2019}. Specifically, for ANIL, we choose the inner-loop stepsize as $0.1$, the outer-loop (meta) stepsize as $0.001$, the task sampling size as $32$ and the number of inner-loop steps as $10$. For MAML, we choose the inner-loop stepsize as $0.5$, the outer-loop stepsize as $0.001$,  the  task samling size as $32$, and the number of inner-loop steps as $3$. For ITD-BiO, AID-BiO-constant and AID-BiO-increasing, we adopt the same procedure as in the experiments in~\Cref{fig1:a}. 


\vspace{0.2cm}
\noindent 
{\bf Parameter selection for the experiments in~\Cref{figure:resultlg}:} 
For the experiments in~\Cref{fig1:ci}, we choose the inner-loop stepsize as $0.05$, the outer-loop (meta) stepsize as $0.002$, the  mini-batch size as $32$, and the number $T$ of inner-loop steps as $10$ for both ANIL and  ITD-BiO. For the experiments in~\Cref{fig1:di}, we choose the inner-loop stepsize as $0.1$, the outer-loop (meta) stepsize as $0.001$, the  mini-batch size as $32$, and the number $T$ of inner-loop steps as $20$ for both ANIL and  ITD-BiO.

%

\section{Further Specifications on Hyperparameter Optimization Experiments}\label{appen:hyperoptimization}
We demonstrate the effectiveness of the proposed stocBiO algorithm on two experiments: data hyper-cleaning and logistic regression, as introduced below.  

\vspace{0.2cm}
\noindent{\bf Logistic Regression on 20 Newsgroup:} 
We compare the performance of our algorithm {\bf stocBiO} with the existing baseline algorithms {\bf reverse, AID-FP, AID-CG and HOAG }over a logistic regression problem on $20$ Newsgroup dataset~\cite{grazzi2020iteration}. The objective function of such a problem is given by 
 \begin{align*}
&\min_\lambda E(\lambda,w^*) = \frac{1}{|\gD_{\text{val}}|}\sum_{(x_i,y_i)\in \gD_{\text{val}}} L(x_iw^*, y_i) \nonumber
\\& \;\mbox{s.t.} \quad w^* = \argmin_{w\in\mathbb{R}^{p\times c}}  \Big(\frac{1}{|\gD_{\text{tr}}|}\sum_{(x_i,y_i)\in \gD_{\text{tr}}}L(x_iw, y_i)  + \frac{1}{cp} \sum_{i=1}^c\sum_{j=1}^p \exp(\lambda_j)w_{ij}^2\Big),
\end{align*}
where $L$ is the cross-entropy loss, $c=20$ is the number of topics, and $p=101631$ is the feature dimension. Following \citealt{grazzi2020iteration}, we use SGD as the optimizer for the outer-loop update for all algorithms. For reverse, AID-FP, AID-CG, we use the suggested and well-tuned hyperparameter setting in their implementations~\url{https://github.com/prolearner/hypertorch} on this application. In specific, they choose the inner- and outer-loop stepsizes as $100$, the number of inner loops as $10$, the number of CG steps as $10$.  For HOAG, we use the same parameters as reverse, AID-FP, AID-CG. For stocBiO, we use the same parameters as reverse, AID-FP, AID-CG, and choose $\eta=0.5,Q=10$. We use stocBiO-$B$ as a  shorthand of stocBiO with a batch size of $B$.

%

\vspace{0.2cm}

\noindent{\bf Data Hyper-Cleaning on MNIST.} We compare the performance of our proposed algorithm stocBiO with other baseline algorithms BSA, TTSA, HOAG on a hyperparameter optimization problem: data hyper-cleaning~\citep{shaban2019truncated} on a dataset derived from MNIST~\citep{lecun1998gradient}, which consists of 20000 images for training, 5000 images for validation, and 10000 images for testing.  
Data hyper-cleaning is to train a classifier in a corrupted setting where each label of training data is replaced by a random class number with a probability $p$ (i.e., the corruption rate). The objective function is given by 
\begin{align*}
&\min_\lambda E(\lambda,w^*) = \frac{1}{|\gD_{\text{val}}|}\sum_{(x_i,y_i)\in \gD_{\text{val}}} L(w^*x_i, y_i) \nonumber
\\& \;\mbox{s.t.} \quad w^* = \argmin_{w} \gL(w,\lambda):= \frac{1}{|\gD_{\text{tr}}|}\sum_{(x_i,y_i)\in \gD_{\text{tr}}}\sigma(\lambda_i)L(wx_i, y_i)  + C_r \|w\|^2,
\end{align*}
where $L$ is the cross-entropy loss, $\sigma(\cdot)$ is the sigmoid function, $C_r$ is a regularization parameter. Following~\citealt{shaban2019truncated}, we choose $C_r=0.001$. 
 All results are averaged over 10 trials with
different random seeds. We adopt Adam~\citep{kingma2014adam} as the optimizer for the outer-loop update for all algorithms. For stochastic algorithms, we set the batch size as $50$ for stocBiO, and $1$ for BSA and TTSA because they use the single-sample data sampling. For all algorithms, we use a grid search to choose the inner-loop stepsize from $\{0.01,0.1,1,10\}$, the  outer-loop stepsize from $\{10^{i},i=-4,-3,-2,-1,0,1,2,3,4\}$, and the number $D$ of inner-loop steps from $\{1,10,50,100,200,1000\}$, where values that achieve the lowest loss after a fixed running time are selected. For  stocBiO, BSA, and TTSA, we choose $\eta$ from $\{0.5\times 2^i, i=-3,-2,-1,0,1,2,3\}$, and $Q$ from $\{3\times 2^i, i=0,1,2,3\}$.

\section{Supporting Lemmas}
In this section, we provide some auxiliary lemmas used for proving the main convergence results. 

First note that  the Lipschitz properties in Assumption~\ref{ass:lip} imply the following lemma.
\begin{lemma}\label{le:boundv}
Suppose Assumption~\ref{ass:lip} holds. Then, the stochastic derivatives $\nabla F(z;\xi)$, $\nabla G(z;\xi)$, $\nabla_x\nabla_y G(z;\xi)$ and $\nabla_y^2 G(z;\xi)$ have bounded variances, i.e., for any $z$ and $\xi$, 
\begin{itemize}
\item $\mathbb{E}_\xi\left \|\nabla F(z;\xi)-\nabla f(z)\right\|^2 \leq M^2.$
\item $\mathbb{E}_\xi\left \|\nabla_x\nabla_y G(z;\xi)-\nabla_x\nabla_y g(z)\right\|^2 \leq L^2.$
\item $\mathbb{E}_\xi \left\|\nabla_y^2 G(z;\xi)-\nabla_y^2 g(z)\right\|^2 \leq L^2.$
\end{itemize}
\end{lemma}
Recall that  $\Phi(x)=f(x,y^*(x))$ in~\cref{objective}. Then, we use 
the following lemma to characterize the Lipschitz properties of $\nabla \Phi(x)$, which is adapted from Lemma 2.2 in~\citealt{ghadimi2018approximation}.
\begin{lemma}\label{le:lipphi}
Suppose Assumptions~\ref{assum:geo},~\ref{ass:lip} and \ref{high_lip} hold. Then, we have, for any $x,x^\prime\in\mathbb{R}^p$,  
\begin{align*}
\|\nabla \Phi(x)- \nabla \Phi(x^\prime)\| \leq L_\Phi \|x-x^\prime\|,
\end{align*}
where the constant $L_\Phi$ is given by
\begin{align}
L_\Phi = L + \frac{2L^2+\tau M^2}{\mu} + \frac{\rho L M+L^3+\tau M L}{\mu^2} + \frac{\rho L^2 M}{\mu^3}.
\end{align}
\end{lemma}

\section{Proof of ~\Cref{prop:grad,deter:gdform}}
In this section, we provide the proofs for~\Cref{prop:grad} and~\Cref{deter:gdform} in~\Cref{sec:alg}.
\subsection{Proof of~\Cref{prop:grad}}
Using the chain rule over the gradient  $\nabla \Phi(x_k)=\frac{\partial f(x_k,y^*(x_k))}{\partial x_k}$, we have
\begin{align}\label{eq:maoxian}
\nabla \Phi(x_k)= \nabla_x f(x_k,y^*(x_k)) + \frac{\partial y^*(x_k)}{\partial x_k}\nabla_y f(x_k,y^*(x_k)).
\end{align}
Based on the optimality of $y^*(x_k)$, we have $\nabla_yg(x_k,y^*(x_k)) = 0$, which, using the implicit differentiation w.r.t. $x_k$, yields
\begin{align}\label{ineq:sscsa}  
\nabla_x\nabla_yg(x_k,y^*(x_k)) +\frac{\partial y^*(x_k)}{\partial x_k}\nabla_y^2g(x_k,y^*(x_k))= 0.
\end{align}
Let $v_k^*$ be the solution of the linear system $ \nabla_y^2g(x_k,y^*(x_k))v=\nabla_y f(x_k,y^*(x_k))$. Then, multiplying $v_k^*$ at the both sides of \cref{ineq:sscsa}, yields
\begin{align*}
-\nabla_x\nabla_yg(x_k,y^*(x_k)) v_k^* = \frac{\partial y^*(x_k)}{\partial x_k}\nabla_y^2g(x_k,y^*(x_k)) v_k^*=  \frac{\partial y^*(x_k)}{\partial x_k} \nabla_y f(x_k,y^*(x_k)),
\end{align*}
which, in conjunction with~\cref{eq:maoxian}, completes the proof.  

\subsection{Proof of~\Cref{deter:gdform} }
Based on the iterative update of  line $5$ in~\Cref{alg:main_deter}, we have $y_k^D = y_k^{0}-\alpha \sum_{t=0}^{D-1}\nabla_y g(x_k,y_k^{t})$, which, combined with the fact that  $\nabla_y g(x_k,y_k^{t})$ is differentiable w.r.t. $x_k$, indicates that the inner output $y_k^T$ is differentiable w.r.t. $x_k$. Then, based on the chain rule, 
we have 
\begin{align}\label{grad:est}
\frac{\partial f(x_k,y^D_k)}{\partial x_k}= \nabla_x f(x_k,y_k^D) + \frac{\partial y_k^D}{\partial x_k}\nabla_y f(x_k,y_k^D).
\end{align}
Based on the iterative updates that $y_k^t = y_k^{t-1}-\alpha \nabla_y g(x_k,y_k^{t-1}) $ for $t=1,...,D$, we have 
\begin{align*}
\frac{\partial y_k^t}{\partial x_k} =& \frac{\partial y_k^{t-1}}{\partial x_k}-\alpha \nabla_x\nabla_y g(x_k,y_k^{t-1})-\alpha\frac{\partial y_k^{t-1}}{\partial x_k} \nabla^2_y g(x_k,y_k^{t-1}) 
\\= &\frac{\partial y_k^{t-1}}{\partial x_k}(I-\alpha  \nabla^2_y g(x_k,y_k^{t-1}))-\alpha \nabla_x\nabla_y g(x_k,y_k^{t-1}).
\end{align*}
Telescoping the above equality over $t$ from $1$ to $D$ yields
\begin{align}\label{gd:formss}
\frac{\partial y_k^D}{\partial x_k} =&\frac{\partial y_k^0}{\partial x_k} \prod_{t=0}^{D-1}(I-\alpha  \nabla^2_y g(x_k,y_k^{t}))-\alpha\sum_{t=0}^{D-1}\nabla_x\nabla_y g(x_k,y_k^{t})\prod_{j=t+1}^{D-1}(I-\alpha  \nabla^2_y g(x_k,y_k^{j})) \nonumber
\\\overset{(i)}=&-\alpha\sum_{t=0}^{D-1}\nabla_x\nabla_y g(x_k,y_k^{t})\prod_{j=t+1}^{D-1}(I-\alpha  \nabla^2_y g(x_k,y_k^{j})). 
\end{align}
where $(i)$ follows from the fact that  $\frac{\partial y_k^0}{\partial x_k}=0$. 
Combining~\cref{grad:est} and~\cref{gd:formss} finishes the proof.


\section{Proof of \Cref{th:aidthem}}\label{appen:aid-bio}
For notation simplification, we define the following quantities. 
\begin{align}\label{eq:notaionssscas}
\Gamma =&3L^2+\frac{3\tau^2 M^2}{\mu^2} + 6L^2\big(1+\sqrt{\kappa}\big)^2\big(\kappa +\frac{\rho M}{\mu^2}\big)^2,\; \delta_{D,N}=\Gamma (1-\alpha \mu)^D  + 6L^2 \kappa \big( \frac{\sqrt{\kappa}-1}{\sqrt{\kappa}+1} \big)^{2N}
\nonumber
\\\Omega =&8\Big(\beta\kappa^2+\frac{2\beta ML}{\mu^2}+\frac{2\beta LM\kappa}{\mu^2}\Big)^2,\; \Delta_0 = \|y_0-y^*(x_{0})\|^2 + \|v_{0}^*-v_0\|^2.
\end{align}

We first provide some supporting lemmas. The following lemma 
characterizes the Hypergradient estimation error $\|\widehat \nabla \Phi(x_k)- \nabla \Phi(x_k)\|$, where $\widehat \nabla \Phi(x_k)$ is given by \cref{hyper-aid} via implicit differentiation. 
\begin{lemma}\label{le:aidhy}
Suppose Assumptions~\ref{assum:geo},~\ref{ass:lip} and \ref{high_lip} hold.  
Then, we have 
\begin{align}
\|\widehat \nabla \Phi(x_k)- \nabla \Phi(x_k)\|^2 \leq &\Gamma (1-\alpha \mu)^D \|y^*(x_k)-y_k^0\|^2 + 6L^2 \kappa \Big( \frac{\sqrt{\kappa}-1}{\sqrt{\kappa}+1} \Big)^{2N}\|v_k^*-v_k^0\|^2. \nonumber
\end{align}
where $\Gamma$ is given by \cref{eq:notaionssscas}. 
\end{lemma}
\begin{proof}[\bf Proof of \Cref{le:aidhy}]
Based on the form of $\nabla\Phi(x_k)$ given by~\Cref{prop:grad},  we have 
\begin{align*}
\|\widehat \nabla \Phi(x_k)- \nabla \Phi(x_k)\|^2 \leq& 3\|\nabla_x f(x_k,y^*(x_k))-\nabla_x f(x_k,y_k^D)\|^2 +3\|\nabla_x \nabla_y g(x_k,y_k^D)\|^2\|v_k^*-v_k^N\|^2  \nonumber
\\&+ 3\|\nabla_x \nabla_y g(x_k,y^*(x_k))-\nabla_x \nabla_y g(x_k,y_k^D) \|^2 \|v_k^*\|^2,
\end{align*}
which, in conjunction with Assumptions~\ref{assum:geo},~\ref{ass:lip} and \ref{high_lip}, yields
\begin{align}\label{eq:midpos}
\|\widehat \nabla \Phi(x_k)- \nabla \Phi(x_k)\|^2 \leq &3L^2\|y^*(x_k)-y_k^D\|^2 + 3L^2\|v_k^*-v_k^N\|^2+3\tau^2\|v_k^*\|^2\|y_k^D-y^*(x_k)\|^2\nonumber
\\\overset{(i)}\leq&  3L^2\|y^*(x_k)-y_k^D\|^2 + 3L^2\|v_k^*-v_k^N\|^2+\frac{3\tau^2 M^2}{\mu^2}\|y_k^D-y^*(x_k)\|^2.
\end{align}
where $(i)$ follows from the fact that $\|v_k^*\|\leq\|(\nabla_y^2g(x_k,y^*(x_k)))^{-1}\|\|\nabla_y f(x_k,y^*(x_k))\|\leq \frac{M}{\mu}$.   

 For notation simplification, let $\widehat v_k=(\nabla_y^2g(x_k,y^D_k))^{-1}\nabla_y f(x_k,y^D_k)$.  We next upper-bound $\|v_k^*-v_k^N\|$ in \cref{eq:midpos}. Based on the convergence result of CG for the quadratic programing, e.g., eq. (17) in~\citealt{grazzi2020iteration}, we have 
 $\|v_k^N-\widehat v_k\| \leq \sqrt{\kappa}\Big( \frac{\sqrt{\kappa}-1}{\sqrt{\kappa}+1} \Big)^N\|v_k^0-\widehat v_k\|.$
Based on this inequality, we further have 
 \begin{align}\label{eq:letknca}
 \|v_k^*-v_k^N\| \leq &\|v_k^*-\widehat v_k\| + \|v_k^N-\widehat v_k\| \leq  \|v_k^*-\widehat v_k\|  +  \sqrt{\kappa}\Big( \frac{\sqrt{\kappa}-1}{\sqrt{\kappa}+1} \Big)^N\|v_k^0-\widehat v_k\| \nonumber
 \\\leq& \Big(1+\sqrt{\kappa}\Big( \frac{\sqrt{\kappa}-1}{\sqrt{\kappa}+1} \Big)^N\Big) \|v_k^*-\widehat v_k\| + \sqrt{\kappa}\Big( \frac{\sqrt{\kappa}-1}{\sqrt{\kappa}+1} \Big)^N\|v_k^*-v_k^0\|. 
 \end{align}
Next, based on the definitions of $v_k^*$ and $ \widehat v_k$, we have 
\begin{align}\label{eq:omgsaca}
\|v_k^*-\widehat v_k\|=& \|(\nabla_y^2g(x_k,y^D_k))^{-1}\nabla_y f(x_k,y^D_k) -(\nabla_y^2g(x_k,y^*(x_k))^{-1}\nabla_y f(x_k,y^*(x_k))\| \nonumber
\\\leq & \Big(\kappa +\frac{\rho M}{\mu^2} \Big)\|y^D_k-y^*(x_k)\|. 
\end{align}
Combining \cref{eq:midpos},~\cref{eq:letknca}, \cref{eq:omgsaca} yields
\begin{align*}
\|\widehat \nabla \Phi(x_k)- \nabla \Phi(x_k)\|^2 \leq &\Big(3L^2+\frac{3\tau^2 M^2}{\mu^2}\Big)\|y^*(x_k)-y_k^D\|^2 + 6L^2 \kappa \Big( \frac{\sqrt{\kappa}-1}{\sqrt{\kappa}+1} \Big)^{2N}\|v_k^*-v_k^0\|^2 \nonumber
\\&+ 6L^2\Big(1+\sqrt{\kappa}\Big( \frac{\sqrt{\kappa}-1}{\sqrt{\kappa}+1} \Big)^N\Big)^2\Big(\kappa +\frac{\rho M}{\mu^2} \Big)^2\|y^D_k-y^*(x_k)\|^2, 
\end{align*}
which, in conjunction with $\|y_k^{D} -y^*(x_k)\| \leq (1-\alpha\mu)^{\frac{D}{2}} \|y^0_k-y^*(x_k)\|$ and the notations in \cref{eq:notaionssscas}, finishes the proof. 
\end{proof}

\begin{lemma}\label{le:bibibiss}
Suppose Assumptions~\ref{assum:geo},~\ref{ass:lip} and \ref{high_lip} hold. Choose
\begin{small}
\begin{align}\label{eq:findbogas}
D\geq& \log{(36 \kappa (\kappa +\frac{\rho M}{\mu^2} )^2+16(\kappa^2+\frac{4LM\kappa}{\mu^2})^2\beta^2\Gamma)}/\log\frac{1}{1-\alpha}=\Theta(\kappa)  \nonumber
\\ N\geq& \frac{1}{2}\log(8\kappa+48(\kappa^2+\frac{2ML}{\mu^2}+\frac{2LM\kappa}{\mu^2})^2\beta^2L^2 \kappa ) /\log \frac{\sqrt{\kappa}+1}{\sqrt{\kappa}-1} = \Theta(\sqrt{\kappa}),
 \end{align}
 \end{small}
\hspace{-0.12cm}where $\Gamma$ is given by \cref{eq:notaionssscas}.   Then, we have 
\begin{align}
\|y^0_k-y^*(x_k)\|^2 + &\|v_k^*-v_k^0\|^2   \leq \Big(\frac{1}{2}\Big)^k  \Delta_0+\Omega\sum_{j=0}^{k-1}\Big(\frac{1}{2}\Big)^{k-1-j}\|\nabla \Phi(x_{j})\|^2,
\end{align}
where $\Omega$ and $\Delta_0$ are given by \cref{eq:notaionssscas}. 
\end{lemma}
\begin{proof}[\bf Proof of \Cref{le:bibibiss}]
Recall that $y^0_k=y^D_{k-1}$. Then, we have 
\begin{align}\label{wocaoleis}
\|y^0_k-y^*(x_k)\|^2 \leq &2\|y^D_{k-1}-y^*(x_{k-1})\|^2 + 2\|y^*(x_k)-y^*(x_{k-1})\|^2 \nonumber
\\\overset{(i)}\leq& 2(1-\alpha\mu)^{D} \|y_{k-1}^0-y^*(x_{k-1})\|^2 + 2\kappa^2\beta^2\|\widehat \nabla \Phi(x_{k-1})\|^2 \nonumber
\\\leq&2(1-\alpha\mu)^{D} \|y_{k-1}^0-y^*(x_{k-1})\|^2 + 4\kappa^2\beta^2\| \nabla \Phi(x_{k-1})-\widehat \nabla \Phi(x_{k-1})\|^2  \nonumber
\\&+ 4\kappa^2\beta^2 \| \nabla \Phi(x_{k-1})\|^2 \nonumber
\\\overset{(ii)}\leq&\big(2(1-\alpha\mu)^{D}+ 4\kappa^2\beta^2\Gamma (1-\alpha \mu)^D\big)\|y^*(x_{k-1})-y_{k-1}^0\|^2\nonumber
\\ &+24\kappa^4L^2\beta^2  \Big( \frac{\sqrt{\kappa}-1}{\sqrt{\kappa}+1} \Big)^{2N} \|v_{k-1}^*-v_{k-1}^0\|^2+ 4\kappa^2\beta^2 \| \nabla \Phi(x_{k-1})\|^2,
\end{align}
where $(i)$ follows from Lemma 2.2 in \citealt{ghadimi2018approximation} and $(ii)$ follows from \Cref{le:aidhy}. In addition, note that 
\begin{align}\label{eq:kdasdaca}
\|v_k^*-v_k^0\|^2 =& \|v_k^*-v_{k-1}^N\|^2 \leq 2\|v_{k-1}^*-v_{k-1}^N\|^2+2\|v_k^*-v_{k-1}^*\|^2 \nonumber
\\\overset{(i)}\leq &4 \Big(1+\sqrt{\kappa}\Big)^2\Big(\kappa +\frac{\rho M}{\mu^2} \Big)^2(1-\alpha\mu)^D\|y_{k-1}^0-y^*(x_{k-1})\|^2 \nonumber
\\&+4\kappa \Big( \frac{\sqrt{\kappa}-1}{\sqrt{\kappa}+1} \Big)^{2N}\|v_{k-1}^*-v_{k-1}^0\|^2 + 2\|v_k^*-v_{k-1}^*\|^2, 
\end{align}
where $(i)$ follows from \cref{eq:letknca}. Combining \cref{eq:kdasdaca} with $\|v_k^*-v_{k-1}^*\|\leq(\kappa^2+\frac{2ML}{\mu^2}+\frac{2LM\kappa}{\mu^2})\|x_k-x_{k-1}\|$, we have 
\begin{align}\label{kopcasa}
\|v_k^*-v_k^0\|^2\overset{(i)}\leq&\Big(16 \kappa \Big(\kappa +\frac{\rho M}{\mu^2} \Big)^2+4\Big(\kappa^2+\frac{4LM\kappa}{\mu^2}\Big)^2\beta^2\Gamma \Big)(1-\alpha\mu)^D\|y_{k-1}^0-y^*(x_{k-1})\|^2 \nonumber
\\&+\Big(4\kappa+48\Big(\kappa^2+\frac{2ML}{\mu^2}+\frac{2LM\kappa}{\mu^2}\Big)^2\beta^2L^2 \kappa \Big) \Big( \frac{\sqrt{\kappa}-1}{\sqrt{\kappa}+1} \Big)^{2N}\|v_{k-1}^*-v_{k-1}^0\|^2 \nonumber
\\&+ 4\Big(\kappa^2+\frac{2ML}{\mu^2}+\frac{2LM\kappa}{\mu^2}\Big)^2\beta^2\|\nabla \Phi(x_{k-1})\|^2, 
\end{align}
where $(i)$ follows from \Cref{le:aidhy}. Combining \cref{wocaoleis} and  \cref{kopcasa} yields
\begin{align*}
\|y^0_k-y^*(x_k)&\|^2 + \|v_k^*-v_k^0\|^2  \nonumber
\\\leq &\Big(18 \kappa \Big(\kappa +\frac{\rho M}{\mu^2} \Big)^2+8\Big(\kappa^2+\frac{4LM\kappa}{\mu^2}\Big)^2\beta^2\Gamma \Big)(1-\alpha\mu)^D\|y_{k-1}^0-y^*(x_{k-1})\|^2 \nonumber
\\&+\Big(4\kappa+24\Big(\kappa^2+\frac{2ML}{\mu^2}+\frac{2LM\kappa}{\mu^2}\Big)^2\beta^2L^2 \kappa \Big) \Big( \frac{\sqrt{\kappa}-1}{\sqrt{\kappa}+1} \Big)^{2N}\|v_{k-1}^*-v_{k-1}^0\|^2 \nonumber
\\&+ 8\Big(\kappa^2+\frac{2ML}{\mu^2}+\frac{2LM\kappa}{\mu^2}\Big)^2\beta^2\|\nabla \Phi(x_{k-1})\|^2,
\end{align*}
which, in conjunction with \cref{eq:findbogas},   yields
\begin{align}\label{eq:televk00}
\|y^0_k-y^*(x_k)\|^2 + \|v_k^*-v_k^0\|^2 \leq &\frac{1}{2} (\|y^0_{k-1}-y^*(x_{k-1})\|^2 + \|v_{k-1}^*-v_{k-1}^0\|^2) \nonumber
\\&+8\Big(\beta\kappa^2+\frac{2\beta ML}{\mu^2}+\frac{2\beta LM\kappa}{\mu^2}\Big)^2\|\nabla \Phi(x_{k-1})\|^2.
\end{align}
Telescoping \cref{eq:televk00} over $k$ and using the notations in~\cref{eq:notaionssscas},  we finish the proof. 
\end{proof}
\begin{lemma}\label{le:gamma1}
Under the same setting as in \Cref{le:bibibiss}, we have 
\begin{align*}
\|\widehat \nabla \Phi(x_k)- \nabla \Phi(x_k)\|^2 \leq &\delta_{D,N}\Big(\frac{1}{2}\Big)^k  \Delta_0+ \delta_{D,N}\Omega\sum_{j=0}^{k-1}\Big(\frac{1}{2}\Big)^{k-1-j}\|\nabla \Phi(x_{j})\|^2. \end{align*}
where $\delta_{T,N}$, $\Omega$ and $\Delta_0$ are given by \cref{eq:notaionssscas}.
\end{lemma}
\begin{proof}[\bf Proof of \Cref{le:gamma1}] 
Based on \Cref{le:aidhy}, \cref{eq:notaionssscas} and using $ab+cd\leq (a+c)(b+d)$ for any positive $a,b,c,d$, we have
\begin{align*}
\|\widehat \nabla \Phi(x_k)- \nabla \Phi(x_k)\|^2 \leq &\delta_{D,N}(\|y^*(x_k)-y_k^0\|^2+\|v_k^*-v_k^0\|^2),
\end{align*}
which, in conjunction with \Cref{le:bibibiss}, finishes the proof. 
\end{proof}

\subsection{Proof of \Cref{th:aidthem}}
In this subsection, provide the proof for \Cref{th:aidthem}. 
Based on the smoothness of the function $\Phi(x)$ established in~\Cref{le:lipphi}, we have 
\begin{align}\label{eq:intimidern_pre}
\Phi(x_{k+1}) \leq & \Phi(x_k)  + \langle \nabla \Phi(x_k), x_{k+1}-x_k\rangle + \frac{L_\Phi}{2} \|x_{k+1}-x_k\|^2 \nonumber
\\\leq& \Phi(x_k)  - \beta \langle \nabla \Phi(x_k),\widehat \nabla \Phi(x_k)- \nabla \Phi(x_k)\rangle -\beta\| \nabla \Phi(x_k)\|^2 + \beta^2 L_\Phi \|\nabla\Phi(x_k)\|^2\nonumber
\\&+\beta^2 L_\Phi\|\nabla\Phi(x_k)-\widehat \nabla\Phi(x_k)\|^2\nonumber
\\\leq&\Phi(x_k) -\Big(\frac{\beta}{2}-\beta^2 L_\Phi \Big)\| \nabla \Phi(x_k)\|^2 +\Big(\frac{\beta}{2}+\beta^2 L_\Phi\Big)\|\nabla\Phi(x_k)-\widehat \nabla\Phi(x_k)\|^2,
\end{align}
which, combined with \Cref{le:gamma1}, yields
\begin{align}\label{eq:teletop}
\Phi(x_{k+1}) \leq &\Phi(x_k) -\Big(\frac{\beta}{2}-\beta^2 L_\Phi \Big)\| \nabla \Phi(x_k)\|^2 + \Big(\frac{\beta}{2}+\beta^2 L_\Phi\Big)
\delta_{D,N}\Big(\frac{1}{2}\Big)^k  \Delta_0 \nonumber
\\&+ \Big(\frac{\beta}{2}+\beta^2 L_\Phi\Big)\delta_{D,N}\Omega\sum_{j=0}^{k-1}\Big(\frac{1}{2}\Big)^{k-1-j}\|\nabla \Phi(x_{j})\|^2.
\end{align}
Telescoping \cref{eq:teletop} over k from $0$ to $K-1$ yields
\begin{align*}
\Big(\frac{\beta}{2}-\beta^2 L_\Phi \Big) \sum_{k=0}^{K-1}\| &\nabla \Phi(x_k)\|^2 \leq \Phi(x_0) - \inf_x\Phi(x) + \Big(\frac{\beta}{2}+\beta^2 L_\Phi\Big) \delta_{D,N} \Delta_0\nonumber
\\&+ \Big(\frac{\beta}{2}+\beta^2 L_\Phi\Big)\delta_{D,N}\Omega\sum_{k=1}^{K-1}\sum_{j=0}^{k-1}\Big(\frac{1}{2}\Big)^{k-1-j}\|\nabla \Phi(x_{j})\|^2,
\end{align*}
which, using  the fact that {\small $\sum_{k=1}^{K-1}\sum_{j=0}^{k-1}\Big(\frac{1}{2}\Big)^{k-1-j}\|\nabla \Phi(x_{j})\|^2 \leq \sum_{k=0}^{K-1}\frac{1}{2^k}\sum_{k=0}^{K-1}\|\nabla \Phi(x_{k})\|^2\leq 2\sum_{k=0}^{K-1}\|\nabla \Phi(x_{k})\|^2$}, yields
\begin{align}\label{jkunisas}
\Big(\frac{\beta}{2}-\beta^2 L_\Phi -\big(\beta\Omega+2\Omega\beta^2 &L_\Phi\big)\delta_{D,N}\Big) \sum_{k=0}^{K-1}\| \nabla \Phi(x_k)\|^2  \nonumber
\\&\leq \Phi(x_0) - \inf_x\Phi(x) + \Big(\frac{\beta}{2}+\beta^2 L_\Phi\Big) \delta_{D,N} \Delta_0. 
\end{align}
Choose $N$ and $D$ such that 
\begin{align}\label{NTsatis}
 \big(\Omega+2\Omega\beta L_\Phi\big)\delta_{D,N} \leq \frac{1}{4}, \quad \delta_{D,N}\leq 1.
\end{align}
Note that based on the definition of $\delta_{D,N}$ in~\cref{eq:notaionssscas}, it suffices to choose $D\geq\Theta(\kappa)$ and $N\geq \Theta(\sqrt{\kappa})$ to satisfy \cref{NTsatis}. Then, substituting \cref{NTsatis} into \cref{jkunisas} yields 
\begin{align*}
\Big(\frac{\beta}{4}-\beta^2 L_\Phi \Big) \sum_{k=0}^{K-1}\| \nabla \Phi(x_k)\|^2 \leq \Phi(x_0) - \inf_x\Phi(x) + \Big(\frac{\beta}{2}+\beta^2 L_\Phi\Big)\Delta_0,
\end{align*}
which, in conjunction with $\beta\leq \frac{1}{8L_\Phi}$, yields
\begin{align}\label{eq:woele}
\frac{1}{K}\sum_{k=0}^{K-1}\| \nabla \Phi(x_k)\|^2 \leq \frac{64L_\Phi (\Phi(x_0) - \inf_x\Phi(x))+5\Delta_0}{K}.  
\end{align}
In order to achieve an $\epsilon$-accurate stationary point, we obtain from~\cref{eq:woele} that 
AID-BiO requires at most the total number $K=\mathcal{O}(\kappa^3\epsilon^{-1})$ of outer iterations. 
Then, based on \cref{hyper-aid}, we have the following complexity results.
\begin{itemize}
\item Gradient complexity: $$\mbox{\normalfont Gc}(f,\epsilon)=2K=\mathcal{O}(\kappa^3\epsilon^{-1}), \mbox{\normalfont Gc}(g,\epsilon)=KD=\mathcal{O}\big(\kappa^4\epsilon^{-1}\big).$$
\item Jacobian- and Hessian-vector product complexities: $$ \mbox{\normalfont JV}(g,\epsilon)=K=\mathcal{O}\left(\kappa^3\epsilon^{-1}\right), \mbox{\normalfont HV}(g,\epsilon)=KN=\mathcal{O}\left(\kappa^{3.5}\epsilon^{-1}\right).$$
\end{itemize}
Then, the proof is complete. 

\section{Proof of~\Cref{th:determin}}\label{append:itd-bio}
We first characterize an important estimation property of the outer-loop gradient estimator $\frac{\partial f(x_k,y^D_k)}{\partial x_k}$ in ITD-BiO for approximating the true gradient $\nabla \Phi(x_k)$ based on \Cref{deter:gdform}.

\begin{lemma}\label{prop:partialG}  
Suppose Assumptions~\ref{assum:geo},~\ref{ass:lip} and \ref{high_lip} hold. Choose  $\alpha\leq \frac{1}{L}$. Then, we have
\begin{small}
\begin{align*}
\Big\|\frac{\partial f(x_k,y^D_k)}{\partial x_k}-\nabla\Phi(x_k)\Big\| \leq&\Big( \frac{L(L+\mu)(1-\alpha\mu)^{\frac{D}{2}}}{\mu} +\frac{2M\left(  \tau\mu+ L\rho \right)}{\mu^2}(1-\alpha\mu)^{\frac{D-1}{2}} \Big)\|y^0_k-y^*(x_k)\|  \nonumber
\\&+ \frac{LM(1-\alpha\mu)^D}{\mu}.
\end{align*}
\end{small}
\end{lemma}
\vspace{-0.4cm}
\Cref{prop:partialG} shows that the gradient estimation error  $\big\|\frac{\partial f(x_k,y^D_k)}{\partial x_k}-\nabla\Phi(x_k)\big\|$ decays exponentially w.r.t. the number $D$ of the inner-loop steps. 
We note that \citealt{grazzi2020iteration} proved a similar result via a fixed point based approach. As a comparison, our proof of \Cref{prop:partialG} directly characterizes the rate of the sequence  $\big(\frac{\partial y^t_k}{\partial x_k},t=0,...,D\big)$ converging to $\frac{\partial y^*(x_k)}{\partial x_k}$ via the differentiation over all corresponding points along the inner-loop GD path as well as the optimality of the point $y^*(x_k)$. 
\begin{proof}[\bf Proof of \Cref{prop:partialG}]
Using $\nabla \Phi(x_k)= \nabla_x f(x_k,y^*(x_k)) + \frac{\partial y^*(x_k)}{\partial x_k}\nabla_y f(x_k,y^*(x_k))$ and~\cref{grad:est} , and using the triangle inequality, we have 
 \begin{align}\label{eq:woaijingii}
\Big\|\frac{\partial f(x_k,y^D_k)}{\partial x_k}& -\nabla\Phi(x_k)\Big\| \nonumber
\\=&\| \nabla_x f(x_k,y_k^D)-\nabla_x f(x_k,y^*(x_k))\| + \left\|\frac{\partial y_k^D}{\partial x_k}-\frac{\partial y^*(x_k)}{\partial x_k}\right\|\|\nabla_y f(x_k,y_k^D)\| \nonumber
\\&+\Big\|\frac{\partial y^*(x_k)}{\partial x_k}\Big\|\big\|\nabla_y f(x_k,y_k^D)-\nabla_y f(x_k,y^*(x_k))\big\|\nonumber
\\\overset{(i)}\leq& L\|y_k^D-y^*(x_k)\| + M \left\|\frac{\partial y_k^D}{\partial x_k}-\frac{\partial y^*(x_k)}{\partial x_k}\right\| + L\Big\|\frac{\partial y^*(x_k)}{\partial x_k}\Big\|\|y_k^D-y^*(x_k)\|, 
\end{align}
where $(i)$ follows from Assumption~\ref{ass:lip}. Our next step is to upper-bound $\left\|\frac{\partial y_k^D}{\partial x_k}-\frac{\partial y^*(x_k)}{\partial x_k}\right\| $ in \cref{eq:woaijingii}.

Based on the updates $y_k^t = y_k^{t-1}-\alpha \nabla_y g(x_k,y_k^{t-1}) $ for $t=1,...,D$  in ITD-BiO and using the chain rule, we have 
\begin{align}\label{eq:1ss}
\frac{\partial y_k^t}{\partial x_k} = \frac{\partial y_k^{t-1}}{\partial x_k} - \alpha \left( \nabla_x\nabla_y g(x_k,y_k^{t-1}) +\frac{\partial y_k^{t-1}}{\partial x_k}\nabla_y^2 g(x_k,y_k^{t-1})\right).
\end{align}
Based on the optimality of $y^*(x_k)$, we have $\nabla_y g(x_k,y^*(x_k))=0$, which, in conjunction with the implicit differentiation theorem, yields
\begin{align}\label{eq:2sss}
\nabla_x\nabla_y g(x_k,y^*(x_k)) + \frac{\partial y^*(x_k)}{\partial x_k}\nabla_y^2 g(x_k,y^*(x_k))=0.
\end{align}
Substituting \cref{eq:2sss} into~\cref{eq:1ss} yields
\begin{align}\label{eq:zhaoposdo}
\frac{\partial y_k^t}{\partial x_k} -\frac{\partial y^*(x_k)}{\partial x_k} =& \frac{\partial y_k^{t-1}}{\partial x_k} -\frac{\partial y^*(x_k)}{\partial x_k}- \alpha \left( \nabla_x\nabla_y g(x_k,y_k^{t-1}) +\frac{\partial y_k^{t-1}}{\partial x_k}\nabla_y^2 g(x_k,y_k^{t-1})\right) \nonumber
\\&+\alpha\left( \nabla_x\nabla_y g(x_k,y^*(x_k)) + \frac{\partial y^*(x_k)}{\partial x_k}\nabla_y^2 g(x_k,y^*(x_k)) \right) \nonumber
\\ = &\frac{\partial y_k^{t-1}}{\partial x_k} -\frac{\partial y^*(x_k)}{\partial x_k}- \alpha \left( \nabla_x\nabla_y g(x_k,y_k^{t-1}) - \nabla_x\nabla_y g(x_k,y^*(x_k))\right) \nonumber
\\&-\alpha\left(\frac{\partial y_k^{t-1}}{\partial x_k}- \frac{\partial y^*(x_k)}{\partial x_k}\right)\nabla_y^2 g(x_k,y_k^{t-1}) \nonumber
\\&+\alpha\frac{\partial y^*(x_k)}{\partial x_k}\left(\nabla_y^2 g(x_k,y^*(x_k))-\nabla_y^2 g(x_k,y_k^{t-1})   \right). 
\end{align}
Combining \cref{eq:2sss} and Assumption~\ref{ass:lip} yields
\begin{align}\label{ggpdas}
\left\| \frac{\partial y^*(x_k)}{\partial x_k}\right\| =\left\|\nabla_x\nabla_y g(x_k,y^*(x_k))\left[\nabla_y^2 g(x_k,y^*(x_k))\right]^{-1}\right\|\leq\frac{L}{\mu}.
\end{align}
Then, combining~\cref{eq:zhaoposdo} and~\cref{ggpdas} 
yields 
\begin{align}\label{eq:ykdef}
\Big\|\frac{\partial y_k^t}{\partial x_k} -\frac{\partial y^*(x_k)}{\partial x_k} \Big\| \overset{(i)}\leq &\Big\| I-\alpha \nabla_y^2 g(x_k,y_k^{t-1})  \Big\| \Big\| \frac{\partial y_k^{t-1}}{\partial x_k} -\frac{\partial y^*(x_k)}{\partial x_k}\Big\|  \nonumber
\\&+\alpha\left( \tau+ \frac{L\rho}{\mu} \right)\|y_k^{t-1} -y^*(x_k)\| \nonumber
\\\overset{(ii)}\leq &(1-\alpha\mu) \Big\| \frac{\partial y_k^{t-1}}{\partial x_k} -\frac{\partial y^*(x_k)}{\partial x_k}\Big\|  +\alpha\left( \tau+ \frac{L\rho}{\mu} \right)\|y_k^{t-1} -y^*(x_k)\|,
\end{align}
where $(i)$ follows from Assumption \ref{high_lip} and $(ii)$ follows from the strong-convexity of $g(x,\cdot)$. Based on the strong-convexity of the lower-level function $g(x,\cdot)$, we have 
\begin{align}\label{eq:mindsa}
\|y_k^{t-1} -y^*(x_k)\| \leq (1-\alpha\mu)^{\frac{t-1}{2}} \|y^0_k-y^*(x_k)\|.
\end{align}
Substituting~\cref{eq:mindsa} into~\cref{eq:ykdef} and telecopting~\cref{eq:ykdef} over $t$ from $1$ to $D$, we have
\begin{align} \label{mamadewen}
\Big\|\frac{\partial y_k^D}{\partial x_k} -\frac{\partial y^*(x_k)}{\partial x_k} \Big\| \leq& (1-\alpha\mu)^{D}\Big\|\frac{\partial y_k^0}{\partial x_k} -\frac{\partial y^*(x_k)}{\partial x_k} \Big\|  \nonumber
\\&+\alpha\left( \tau+ \frac{L\rho}{\mu} \right)\sum_{t=0}^{D-1} (1-\alpha\mu)^{D-1-t}(1-\alpha\mu)^{\frac{t}{2}} \|y^0_k-y^*(x_k)\| \nonumber
\\=&(1-\alpha\mu)^{D}\Big\|\frac{\partial y_k^0}{\partial x_k} -\frac{\partial y^*(x_k)}{\partial x_k} \Big\| 
+ \frac{2\left(  \tau\mu+ L\rho \right)}{\mu^2}(1-\alpha\mu)^{\frac{D-1}{2}}\|y^0_k-y^*(x_k)\| \nonumber
\\\leq &\frac{L(1-\alpha\mu)^D}{\mu} + \frac{2\left(  \tau\mu+ L\rho \right)}{\mu^2}(1-\alpha\mu)^{\frac{D-1}{2}}\|y^0_k-y^*(x_k)\|,
\end{align}
where the last inequality follows from $\frac{\partial y_k^0}{\partial x_k} =0$ and \cref{ggpdas}. Then, combining \cref{eq:woaijingii}, \cref{ggpdas}, \cref{eq:mindsa} and \cref{mamadewen} completes the proof.
\end{proof}
\subsection{Proof of~\Cref{th:determin}}
Based on the characterization on the estimation error of the gradient estimate $\frac{\partial f(x_k,y^D_k)}{\partial x_k}$ in \Cref{prop:partialG}, we now prove~\Cref{th:determin}. 

Recall the notation that $\widehat \nabla\Phi(x_k) =\frac{\partial f(x_k,y^D_k)}{\partial x_k}$.
Using an approach similar to \cref{eq:intimidern_pre}, we have 
\begin{align}\label{eq:intimidern}
\Phi(x_{k+1}) \leq&\Phi(x_k) -\Big(\frac{\beta}{2}-\beta^2 L_\Phi \Big)\| \nabla \Phi(x_k)\|^2 +\Big(\frac{\beta}{2}+\beta^2 L_\Phi\Big)\|\nabla\Phi(x_k)-\widehat \nabla\Phi(x_k)\|^2,
\end{align}
which, in conjunction with \Cref{prop:partialG} and using $\|y^0_k-y^*(x_k)\|^2\leq\Delta$, yields
\begin{align}\label{eq:pladwa}
\Phi(x_{k+1}) \leq & \Phi(x_k) -\Big(\frac{\beta}{2}-\beta^2 L_\Phi \Big)\| \nabla \Phi(x_k)\|^2 \nonumber
\\ &+3\Delta\Big(\frac{\beta}{2}+\beta^2 L_\Phi\Big)\Big( \frac{L^2(L+\mu)^2}{\mu^2} (1-\alpha\mu)^{D} +\frac{4M^2\left(  \tau\mu+ L\rho \right)^2}{\mu^4}(1-\alpha\mu)^{D-1} \Big) \nonumber
\\&+3\Big(\frac{\beta}{2}+\beta^2 L_\Phi\Big)\frac{L^2M^2(1-\alpha\mu)^{2D}}{\mu^2}.
\end{align}
Telescoping~\cref{eq:pladwa} over $k$ from $0$ to $K-1$ yields
\begin{align}\label{eq:xiangle}
\frac{1}{K}\sum_{k=0}^{K-1}&\Big(\frac{1}{2}-\beta L_\Phi \Big)\| \nabla \Phi(x_k)\|^2 \leq \frac{ \Phi(x_0)-\inf_x\Phi(x)}{\beta K} +3\Big(\frac{1}{2}+\beta L_\Phi\Big)\frac{L^2M^2(1-\alpha\mu)^{2D}}{\mu^2} \nonumber
\\+& 3\Delta\Big(\frac{1}{2}+\beta L_\Phi\Big)\Big( \frac{L^2(L+\mu)^2}{\mu^2} (1-\alpha\mu)^{D} +\frac{4M^2\left(  \tau\mu+ L\rho \right)^2}{\mu^4}(1-\alpha\mu)^{D-1} \Big). 
\end{align}
Substuting $\beta=\frac{1}{4L_\Phi}$ and $D=\log\Big(\max\big\{\frac{3LM}{\mu},9\Delta L^2(1+\frac{L}{\mu})^2,\frac{36\Delta M^2(\tau\mu+L\rho)^2}{(1-\alpha\mu)\mu^4}  \big\}\frac{9}{2\epsilon}\Big)/\log\frac{1}{1-\alpha\mu}=\Theta(\kappa\log\frac{1}{\epsilon})$ in~\cref{eq:xiangle} yields
\begin{align}\label{holiugen}
\frac{1}{K}\sum_{k=0}^{K-1}\| \nabla \Phi(x_k)\|^2 \leq \frac{16 L_\Phi (\Phi(x_0)-\inf_x\Phi(x))}{K} + \frac{2\epsilon}{3}.
\end{align}
In order to achieve an $\epsilon$-accurate stationary point, we obtain from~\cref{holiugen} that 
ITD-BiO requires at most the total number $K=\mathcal{O}(\kappa^3\epsilon^{-1})$ of outer iterations. 
Then, based on the gradient form given by \Cref{deter:gdform}, we have the following complexity results.
\begin{itemize}
\item Gradient complexity: $\mbox{\normalfont Gc}(f,\epsilon)=2K=\mathcal{O}(\kappa^3\epsilon^{-1}), \mbox{\normalfont Gc}(g,\epsilon)=KD=\mathcal{O}\left(\kappa^4\epsilon^{-1}\log\frac{1}{\epsilon}\right).$
\item Jacobian- and Hessian-vector product complexities: $$ \mbox{\normalfont JV}(g,\epsilon)=KD=\mathcal{O}\left(\kappa^4\epsilon^{-1}\log\frac{1}{\epsilon}\right), \mbox{\normalfont HV}(g,\epsilon)=KD=\mathcal{O}\left(\kappa^4\epsilon^{-1}\log\frac{1}{\epsilon}\right).$$
\end{itemize}
Then, the proof is complete. 


\section{Proofs of \Cref{prop:hessian} and~\Cref{th:nonconvex}}
In this section, we provide the proofs for the convergence and complexity results of the proposed algorithm stocBiO for the stochastic case. 
\subsection{Proof of \Cref{prop:hessian}}
Based on the definition of $v_Q$ in~\cref{ours:est} and conditioning on $x_k,y_k^D$, we have 
\begin{align*}
\mathbb{E}v_Q=& \mathbb{E}  \eta \sum_{q=-1}^{Q-1}\prod_{j=Q-q}^Q (I - \eta \nabla_y^2G(x_k,y_k^D;\gB_j)) \nabla_y F(x_k,y_k^D;\gD_F),  \nonumber
 \\ & = \eta \sum_{q=0}^{Q} (I - \eta \nabla_y^2g(x_k,y_k^D))^q\nabla_y f(x_k,y_k^D)  \nonumber
 \\& = \eta \sum_{q=0}^{\infty} (I - \eta \nabla_y^2g(x_k,y_k^D))^q\nabla_y f(x_k,y_k^D) -\eta \sum_{q=Q+1}^{\infty} (I - \eta \nabla_y^2g(x_k,y_k^D))^q\nabla_y f(x_k,y_k^D) \nonumber
 \\& = \eta (\eta \nabla_y^2 g(x_k,y_k^D))^{-1}\nabla_y f(x_k,y_k^D) - \eta \sum_{q=Q+1}^{\infty} (I - \eta \nabla_y^2g(x_k,y_k^D))^q\nabla_y f(x_k,y_k^D), 
\end{align*}
which, in conjunction with the strong-convexity of function $g(x,\cdot)$, yields
\begin{align}\label{eq:fm}
\big\|\mathbb{E}v_Q- [\nabla_y^2 g(x_k,y^D_k)]^{-1}\nabla_y f(x_k,y_k^D) \big \| \leq \eta \sum_{q=Q+1}^{\infty}(1-\eta\mu)^{q} M\leq \frac{(1-\eta\mu)^{Q+1}M}{\mu}. 
\end{align}
This finishes the proof for the estimation bias. 
We next prove the variance bound. Note that 
\begin{small}
\begin{align}\label{eq:init}
\mathbb{E} \bigg\|& \eta \sum_{q=-1}^{Q-1}\prod_{j=Q-q}^Q (I - \eta \nabla_y^2G(x_k,y_k^D;\gB_j)) \nabla_y F(x_k,y_k^D;\gD_F)-( \nabla_y^2 g(x_k,y_k^D))^{-1}\nabla_y f(x_k,y_k^D)\bigg\|^2 \nonumber
\\ \overset{(i)}\leq & 2\mathbb{E} \bigg\| \ \eta \sum_{q=-1}^{Q-1}\prod_{j=Q-q}^Q (I - \eta \nabla_y^2G(x_k,y_k^D;\gB_j))  -  ( \nabla_y^2 g(x_k,y_k^D))^{-1} \bigg \|^2M^2 + \frac{2M^2}{\mu^2D_f}  \nonumber
\\\leq& 4\mathbb{E}\bigg\|  \eta \sum_{q=-1}^{Q-1}\prod_{j=Q-q}^Q (I - \eta \nabla_y^2G(x_k,y_k^D;\gB_j)) -  \eta\sum_{q=0}^Q(I - \eta \nabla_y^2g(x_k,y_k^D))^q \bigg \|^2 M^2 \nonumber
\\&+ 4\mathbb{E}\bigg\|\eta\sum_{q=0}^Q(I - \eta \nabla_y^2g(x_k,y_k^D))^q) -  ( \nabla_y^2 g(x_k,y_k^D))^{-1} \bigg \|^2 M^2+ \frac{2M^2}{\mu^2D_f}  \nonumber
\\\overset{(ii)}\leq & 4\eta^2 \mathbb{E}\bigg\| \sum_{q=0}^{Q}\prod_{j=Q+1-q}^Q (I - \eta \nabla_y^2G(x_k,y_k^D;\gB_j)) -  \sum_{q=0}^Q(I - \eta \nabla_y^2g(x_k,y_k^D))^q \bigg \|^2M^2  +\frac{4(1-\eta \mu)^{2Q+2}M^2}{\mu^2}+ \frac{2M^2}{\mu^2D_f}  \nonumber
\\\overset{(iii)}\leq&4\eta^2 M^2 Q\mathbb{E} \sum_{q=0}^{Q} \underbrace{\bigg\|\prod_{j=Q+1-q}^Q (I - \eta \nabla_y^2G(x_k,y_k^D;\gB_j))-  (I - \eta \nabla_y^2g(x_k,y_k^D))^q \bigg \|^2}_{M_q}   \nonumber
\\&+\frac{4(1-\eta \mu)^{2Q+2}M^2}{\mu^2}  + \frac{2M^2}{\mu^2D_f}
\end{align}
\end{small}
\hspace{-0.12cm}where $(i)$ follows from \Cref{le:boundv}, $(ii)$ follows from~\cref{eq:fm}, and $(iii)$ follows from the Cauchy-Schwarz inequality.

 Our next step is to upper-bound $M_q$ in~\cref{eq:init}. For simplicity, we define a general quantity $M_i$ for by replacing  $q$ in $M_q$ with $i$. Then, we have 
\begin{align}\label{eq:mq}
\mathbb{E}M_i =&\mathbb{E}\bigg\| (I - \eta \nabla_y^2g(x_k,y_k^D))\prod_{j=Q+2-i}^Q (I - \eta \nabla_y^2G(x_k,y_k^D;\gB_j)) -  (I - \eta \nabla_y^2g(x_k,y_k^D))^i \bigg \|^2 \nonumber
\\&+ \mathbb{E}\bigg\| \eta( \nabla_y^2g(x_k,y_k^D)-\nabla_y^2G(x_k,y_k^D;\gB_{Q+1-i}) ) \prod_{j=Q+2-i}^Q (I - \eta \nabla_y^2G(x_k,y_k^D;\gB_j)) \bigg \|^2 \nonumber
\\&+ 2\mathbb{E}\Big\langle(I - \eta \nabla_y^2g(x_k,y_k^D))\prod_{j=Q+2-i}^Q (I - \eta \nabla_y^2G(x_k,y_k^D;\gB_j)) -  (I - \eta \nabla_y^2g(x_k,y_k^D))^i, \nonumber
\\&\hspace{0.8cm}\eta( \nabla_y^2g(x_k,y_k^D)-\nabla_y^2G(x_k,y_k^D;\gB_{Q+1-i}) ) \prod_{j=Q+2-i}^Q (I - \eta \nabla_y^2G(x_k,y_k^D;\gB_j)) \Big\rangle  \nonumber
\\\overset{(i)}=&\mathbb{E}\bigg\| (I - \eta \nabla_y^2g(x_k,y_k^D))\prod_{j=Q+2-i}^Q (I - \eta \nabla_y^2G(x_k,y_k^D;\gB_j)) -  (I - \eta \nabla_y^2g(x_k,y_k^D))^i \bigg \|^2 \nonumber
\\&+ \mathbb{E}\bigg\|\eta( \nabla_y^2g(x_k,y_k^D)-\nabla_y^2G(x_k,y_k^D;\gB_{Q+1-i}) ) \prod_{j=Q+2-i}^Q (I - \eta \nabla_y^2G(x_k,y_k^D;\gB_j))\bigg \|^2 \nonumber
\\\overset{(ii)} \leq&(1-\eta\mu)^2\mathbb{E}M_{i-1} + \eta^2(1-\eta\mu)^{2i-2} \mathbb{E}\| \nabla_y^2g(x_k,y_k^D)-\nabla_y^2G(x_k,y_k^D;\gB_{Q+1-i}) \|^2 \nonumber
\\ \overset{(iii)}\leq&(1-\eta\mu)^2\mathbb{E} M_{i-1}+ \eta^2(1-\eta\mu)^{2i-2} \frac{L^2}{|\gB_{Q+1-i}|},
\end{align}
where $(i)$ follows from the fact that $\mathbb{E}_{\gB_{Q+1-i}}\nabla_y^2G(x_k,y_k^D;\gB_{Q+1-i})  = \nabla_y^2g(x_k,y_k^D)$, $(ii)$ follows from the strong-convexity of function $G(x,\cdot;\xi)$, and $(iii)$ follows from~\Cref{le:boundv}.

Then, telescoping~\cref{eq:mq} over $i$ from $2$ to $q$ yields
\begin{align*}
\mathbb{E} M_q \leq L^2\eta^2(1-\eta\mu)^{2q-2}\sum_{j=1}^q \frac{1}{|\gB_{Q+1-j}|},
\end{align*}
which, in conjunction with the choice of $|\gB_{Q+1-j}|=BQ(1-\eta\mu)^{j-1}$ for $j=1,...,Q$, yields
\begin{align}\label{eq:edhs}
\mathbb{E} M_q \leq &\eta^2(1-\eta\mu)^{2q-2}  \sum_{j=1}^q \frac{L^2}{BQ} \Big(\frac{1}{1-\eta\mu}\Big)^{j-1} \nonumber
\\=&\frac{\eta^2L^2}{BQ} (1-\eta\mu)^{2q-2} \frac{\left(\frac{1}{1-\eta\mu}\right)^{q-1}-1}{\frac{1}{1-\eta\mu} -1} \leq  \frac{\eta L^2}{(1-\eta\mu)\mu} \frac{1}{BQ}(1-\eta\mu)^q.
\end{align}
Substituting~\cref{eq:edhs} into~\cref{eq:init} yields
\begin{align}
\mathbb{E} \bigg\|& \eta \sum_{q=-1}^{Q-1}\prod_{j=Q-q}^Q (I - \eta \nabla_y^2G(x_k,y_k^D;\gB_j)) \nabla_y F(x_k,y_k^D;\gD_F)-( \nabla_y^2 g(x_k,y_k^D))^{-1}\nabla_y f(x_k,y_k^D)\bigg\|^2 \nonumber
\\\leq& 4\eta^2 M^2 Q \sum_{q=0}^{Q} \frac{\eta L^2}{(1-\eta\mu)\mu} \frac{1}{BQ}(1-\eta\mu)^q+\frac{4(1-\eta \mu)^{2Q+2}M^2}{\mu^2}  + \frac{2M^2}{\mu^2D_f}  \nonumber
\\\leq &\frac{4\eta^2  L^2M^2}{\mu^2} \frac{1}{B}+\frac{4(1-\eta \mu)^{2Q+2}M^2}{\mu^2}+ \frac{2M^2}{\mu^2D_f} ,
\end{align} 
where the last inequality follows from the fact that $\sum_{q=0}^S x^{q}\leq \frac{1}{1-x}$. Then, the proof is complete.

\subsection{Auxiliary Lemmas for Proving~\Cref{th:nonconvex}}\label{se:supplemma}
We first use the following lemma to characterize the first-moment error of the gradient estimate $\widehat \nabla \Phi(x_k)$, whose form is given by~\cref{estG}. 
\begin{lemma}\label{le:first_m}
Suppose Assumptions~\ref{assum:geo},~\ref{ass:lip} and \ref{high_lip} hold.  Then, conditioning on $x_k$ and $y_k^D$, we have
\begin{align*}
\big\|\mathbb{E}\widehat \nabla \Phi(x_k)-\nabla \Phi(x_k)\big\|^2\leq 2 \Big( L+\frac{L^2}{\mu} + \frac{M\tau}{\mu}+\frac{LM\rho}{\mu^2}\Big)^2\|y_k^D-y^*(x_k)\|^2 +\frac{2L^2M^2(1-\eta \mu)^{2Q}}{\mu^2}.
\end{align*}
\end{lemma}
\begin{proof}[\bf Proof of~\Cref{le:first_m}] To simplify notations, we define 
\begin{align}\label{def:phih}
\widetilde \nabla \Phi_D(x_k) =  \nabla_x f(x_k,y^D_k) -\nabla_x \nabla_y g(x_k,y^D_k) \big[\nabla_y^2 g(x_k,y^D_k)\big]^{-1}
\nabla_y f(x_k,y^D_k).
\end{align}
Based on the definition of $\widehat \nabla \Phi(x_k)$ in~\cref{estG} and conditioning on $x_k$ and $y_k^D$, we have 
\begin{align*}
\mathbb{E}\widehat \nabla \Phi(x_k) =&  \nabla_x f(x_k,y_k^D)-\nabla_x \nabla_y g(x_k,y_k^D)\mathbb{E} v_Q\nonumber
\\=&\widetilde \nabla \Phi_D(x_k) - \nabla_x \nabla_y g(x_k,y_k^D)( \mathbb{E}v_Q- [\nabla_y^2 g(x_k,y^D_k)]^{-1}\nabla_y f(x_k,y_k^D)),
\end{align*}
which further implies that 
\begin{align}\label{eq:first_eg}
\big\|\mathbb{E}\widehat \nabla& \Phi(x_k)- \nabla \Phi(x_k)  \big\|^2\nonumber
\\\leq &2\mathbb{E}\|\widetilde \nabla \Phi_D(x_k) -\nabla \Phi(x_k)\|^2 +  2\|\mathbb{E}\widehat \nabla \Phi(x_k)- \widetilde \nabla \Phi_D(x_k)\|^2 \nonumber
\\\leq&2\mathbb{E}\|\widetilde \nabla \Phi_D(x_k) -\nabla \Phi(x_k)\|^2 + 2L^2\| \mathbb{E}v_Q- [\nabla_y^2 g(x_k,y^D_k)]^{-1}\nabla_y f(x_k,y_k^D) \|^2\nonumber
\\\leq& 2\mathbb{E}\|\widetilde \nabla \Phi_D(x_k) -\nabla \Phi(x_k)\|^2  + \frac{ 2L^2M^2(1-\eta\mu)^{2Q+2}}{\mu^2},
\end{align}
where the last inequality follows from~\Cref{prop:hessian}.  
Our next step is to upper-bound the first term at the right hand side of~\cref{eq:first_eg}. Using the fact that $\big\|\nabla_y^2 g(x,y)^{-1}\big\|\leq \frac{1}{\mu}$ and based on Assumptions~\ref{ass:lip} and~\ref{high_lip}, we have
\begin{align}\label{eq:appox}
\|\widetilde \nabla \Phi_D(x_k) -\nabla \Phi(x_k)\| \leq &\| \nabla_x f(x_k,y^D_k)- \nabla_x f(x_k,y^*(x_k))\| \nonumber
\\&+\frac{L^2}{\mu}\|y_k^D-y^*(x_k)\|+ \frac{M\tau}{\mu}\|y_k^D-y^*(x_k)\| \nonumber
\\&+ LM \big\|\nabla_y^2 g(x_k,y^D_k)^{-1} \nonumber
-\nabla_y^2 g(x_k,y^*(x_k))^{-1}
\big\|
\\\leq & \Big( L+\frac{L^2}{\mu} + \frac{M\tau}{\mu}+\frac{LM\rho}{\mu^2}\Big) \|y_k^D-y^*(x_k)\|,
\end{align} 
where  the last inequality follows from the inequality $\|M_1^{-1}-M_2^{-1}\|\leq \|M_1^{-1}M_2^{-1}\|\|M_1-M_2\| $ for any two matrices $M_1$ and $M_2$. Combining~\cref{eq:first_eg} and~\cref{eq:appox}
 yields
\begin{align*}
\big\|\mathbb{E}\widehat \nabla& \Phi(x_k)- \nabla \Phi(x_k)  \big\|^2 \leq 2 \Big( L+\frac{L^2}{\mu} + \frac{M\tau}{\mu}+\frac{LM\rho}{\mu^2}\Big)^2 \|y_k^D-y^*(x_k)\|^2 +\frac{2L^2M^2(1-\eta \mu)^{2Q}}{\mu^2},
\end{align*}
which  completes the proof.
\end{proof}
Then, we use the following lemma to characterize the variance of the estimator $\widehat \nabla \Phi(x_k)$. 
\begin{lemma}\label{le:variancc} 
Suppose Assumptions~\ref{assum:geo},~\ref{ass:lip} and \ref{high_lip} hold. Then, we have
\begin{align*}
\mathbb{E}\|\widehat \nabla \Phi(x_k)-\nabla \Phi(x_k)\|^2 \leq &  \frac{4L^2M^2}{\mu^2D_g} + \Big(\frac{8L^2}{\mu^2} + 2\Big) \frac{M^2}{D_f}+ \frac{16\eta^2  L^4M^2}{\mu^2} \frac{1}{B}+\frac{16 L^2M^2(1-\eta \mu)^{2Q}}{\mu^2} \nonumber
\\&+ \Big( L+\frac{L^2}{\mu} + \frac{M\tau}{\mu}+\frac{LM\rho}{\mu^2}\Big)^2 \mathbb{E}\|y_k^D-y^*(x_k)\|^2.
\end{align*}
\end{lemma}
\begin{proof}[\bf Proof of~\Cref{le:variancc}] Based on the definitions of $\nabla \Phi(x_k)$ and $\widetilde \nabla \Phi_D(x_k)$ in~\cref{trueG} and~\cref{def:phih} and conditioning on $x_k$ and $y_k^D$, we have
\begin{align}\label{eq:vbbs}
\mathbb{E}\|\widehat \nabla& \Phi(x_k)-\nabla \Phi(x_k)\|^2  \nonumber
\\\overset{(i)}=& \mathbb{E} \|\widehat \nabla \Phi(x_k)- \widetilde \nabla \Phi_D(x_k)\|^2+\|\widetilde \nabla \Phi_D(x_k)-\nabla \Phi(x_k)\|^2 
\nonumber
\\\overset{(ii)}\leq &2\mathbb{E} \big\|\nabla_x \nabla_y G(x_k,y_k^D;\gD_G) v_Q -\nabla_x \nabla_y g(x_k,y^D_k) \big[\nabla_y^2 g(x_k,y^D_k)\big]^{-1}
\nabla_y f(x_k,y^D_k)\big\|^2 + \frac{2M^2}{D_f} \nonumber
\\&+ \Big( L+\frac{L^2}{\mu} + \frac{M\tau}{\mu}+\frac{LM\rho}{\mu^2}\Big)^2 \|y_k^D-y^*(x_k)\|^2 \nonumber
\\\overset{(iii)}\leq& \frac{4M^2}{\mu^2}\mathbb{E} \|\nabla_x \nabla_y G(x_k,y_k^D;\gD_G) -\nabla_x \nabla_y g(x_k,y^D_k) \|^2 + 4L^2\mathbb{E}\|v_Q-\big[\nabla_y^2 g(x_k,y^D_k)\big]^{-1}
\nabla_y f(x_k,y^D_k)\|^2 \nonumber
\\&+ \Big( L+\frac{L^2}{\mu} + \frac{M\tau}{\mu}+\frac{LM\rho}{\mu^2}\Big)^2 \|y_k^D-y^*(x_k)\|^2 + \frac{2M^2}{D_f},
\end{align}
where $(i)$ follows from the fact that $\mathbb{E}_{\gD_G,\gD_H,\gD_F} \widehat \nabla \Phi(x_k) =  \widetilde \nabla \Phi_D(x_k)$, $(ii)$ follows from \Cref{le:boundv} and~\cref{eq:appox}, and $(iii)$
follows from the Young's inequality and Assumption~\ref{ass:lip}. 

 Using~\Cref{le:boundv} and~\Cref{prop:hessian} in \cref{eq:vbbs}, yields 
\begin{align}
\mathbb{E}\|\widehat \nabla \Phi(x_k)-\nabla \Phi(x_k)\|^2 \leq &  \frac{4L^2M^2}{\mu^2D_g} + \frac{16\eta^2  L^4M^2}{\mu^2} \frac{1}{B}+\frac{16(1-\eta \mu)^{2Q}L^2M^2}{\mu^2}+ \frac{8L^2M^2}{\mu^2D_f}\nonumber
\\&+ \Big( L+\frac{L^2}{\mu} + \frac{M\tau}{\mu}+\frac{LM\rho}{\mu^2}\Big)^2 \|y_k^D-y^*(x_k)\|^2 + \frac{2M^2}{D_f},
\end{align}
which, unconditioning on $x_k$ and $y_k^D$, completes the proof. 
\end{proof}
It can be seen from Lemmas~\ref{le:first_m} and~\ref{le:variancc} that the upper bounds on both the estimation error and bias depend on the tracking error $\|y_k^D-y^*(x_k)\|^2$. The following lemma provides an upper bound on such a tracking error $\|y_k^D-y^*(x_k)\|^2$. 
\begin{lemma}\label{tra_error}
Suppose Assumptions~\ref{assum:geo},~\ref{ass:lip} and~\ref{ass:bound} hold. Define constants
\begin{small}
\begin{align}\label{eq:defs}
\lambda=&  \Big(\frac{L-\mu}{L+\mu}\Big)^{2D} \Big(2+  \frac{4\beta^2L^2}{\mu^2}  \Big( L+\frac{L^2}{\mu} + \frac{M\tau}{\mu}+\frac{LM\rho}{\mu^2}\Big)^2  \Big) \nonumber
\\\Delta =&\frac{4L^2M^2}{\mu^2D_g} + \Big(\frac{8L^2}{\mu^2} + 2\Big) \frac{M^2}{D_f}+ \frac{16\eta^2  L^4M^2}{\mu^2} \frac{1}{B}+\frac{16 L^2M^2(1-\eta \mu)^{2Q}}{\mu^2} 
\nonumber
\\\omega = &\frac{4\beta^2L^2}{\mu^2} \Big(\frac{L-\mu}{L+\mu}\Big)^{2D}. 
\end{align}
\end{small}
\hspace{-0.12cm}Choose $D$ such that $\lambda<1$ and set inner-loop stepsize $\alpha=\frac{2}{L+\mu}$. Then, we have 
\begin{align*}
\mathbb{E}\|y_k^{D}&-y^*(x_k) \|^2 \nonumber
\\\leq&\lambda^{k} \left( \left(\frac{L-\mu}{L+\mu}\right)^{2D}\|y_0-y^*(x_0)\|^2 + \frac{\sigma^2}{L\mu S}\right) + \omega\sum_{j=0}^{k-1}\lambda^{k-1-j} \mathbb{E}\|\nabla \Phi(x_j)\|^2 + \frac{\omega\Delta +\frac{\sigma^2}{L\mu S}}{1-\lambda}.
\end{align*}
\end{lemma}
\begin{proof}[\bf Proof of~\Cref{tra_error}] First note that for an integer $t\leq D$
\begin{align}\label{eq:initss}
\|y_k^{t+1}-y^*(x_k) \|^2= &\|y_k^{t+1}-y_k^t\|^2 + 2\langle y_k^{t+1}-y_k^t,  y_k^t-y^*(x_k) \rangle + \| y_k^t-y^*(x_k)\|^2 \nonumber
\\=& \alpha^2\| \nabla_y G(x_k,y_k^{t}; \gS_t) \|^2 -  2\alpha\langle  \nabla_y G(x_k,y_k^{t}; \gS_t),  y_k^t-y^*(x_k) \rangle +\| y_k^t-y^*(x_k)\|^2.
\end{align}
Conditioning on $y_k^t$ and taking expectation in~\cref{eq:initss}, we have 
\begin{align}\label{eq:traerr}
\mathbb{E}\|y_k^{t+1}&-y^*(x_k) \|^2  \nonumber
\\\overset{(i)}\leq& \alpha^2 \Big(\frac{\sigma^2}{S} + \|\nabla_y g(x_k,y_k^t)\|^2 \Big) - 2\alpha \langle  \nabla_y g(x_k,y_k^{t}),  y_k^t-y^*(x_k) \rangle  \nonumber
\\&+\| y_k^t-y^*(x_k)\|^2 \nonumber
\\\overset{(ii)}\leq&\frac{\alpha^2\sigma^2}{S} + \alpha^2\|\nabla_y g(x_k,y_k^t)\|^2 - 2\alpha\left(  \frac{L\mu}{L+\mu} \|y_k^t-y^*(x_k)\|^2+\frac{\|\nabla_y g(x_k,y_k^t)\|^2}{L+\mu} \right) \nonumber
\\&+\| y_k^t-y^*(x_k)\|^2  \nonumber
\\=& \frac{\alpha^2\sigma^2}{S} - \alpha\left(\frac{2}{L+\mu}-\alpha\right)\|\nabla_y g(x_k,y_k^t)\|^2 + \left(1-\frac{2\alpha L\mu}{L+\mu} \right)\|y_k^t-y^*(x_k)\|^2
\end{align}
where $(i)$ follows from the third item in Assumption~\ref{ass:lip}, and $(ii)$ follows from the strong-convexity and smoothness of the function $g$. Since $\alpha=\frac{2}{L+\mu}$, we obtain from~\cref{eq:traerr} that  
\begin{align}\label{eq:eyk}
\mathbb{E}\|y_k^{t+1}&-y^*(x_k) \|^2\leq \left(\frac{L-\mu}{L+\mu} \right)^2\|y_k^t-y^*(x_k)\|^2 + \frac{4\sigma^2}{(L+\mu)^2S}.
\end{align}
Unconditioning on $y^t_k$ in \cref{eq:eyk} and telescoping~\cref{eq:eyk} over $t$ from $0$ to $D-1$ yield
\begin{align}\label{eq:yjtt}
\mathbb{E}\|y_k^{D}-y^*(x_k) \|^2 \leq &  \left(\frac{L-\mu}{L+\mu}\right)^{2D}\mathbb{E}\|y^0_k-y^*(x_k)\|^2 + \frac{\sigma^2}{L\mu S} \nonumber
\\ = & \left(\frac{L-\mu}{L+\mu}\right)^{2D}\mathbb{E}\|y^D_{k-1}-y^*(x_k)\|^2 + \frac{\sigma^2}{L\mu S},
\end{align} 
where the last inequality follows from \Cref{alg:main} that $y_k^0=y^{D}_{k-1}$. Note that 
\begin{align}\label{eq:midone}
\mathbb{E}\|y^D_{k-1}-y^*(x_k)\|^2 \leq &2\mathbb{E}\|y^D_{k-1}-y^*(x_{k-1})\|^2 + 2\mathbb{E}\|y^*(x_{k-1})-y^*(x_k)\|^2 \nonumber
\\\overset{(i)}\leq& 2\mathbb{E}\|y^D_{k-1}-y^*(x_{k-1})\|^2 +\frac{2L^2}{\mu^2} \mathbb{E}\|x_k-x_{k-1}\|^2 \nonumber
\\ \leq & 2\mathbb{E}\|y^D_{k-1}-y^*(x_{k-1})\|^2 +\frac{2\beta^2L^2}{\mu^2} \mathbb{E}\|\widehat \nabla \Phi(x_{k-1})\|^2 \nonumber
\\\leq & 2\mathbb{E}\|y^D_{k-1}-y^*(x_{k-1})\|^2  + \frac{4\beta^2L^2}{\mu^2} \mathbb{E}\|\nabla \Phi(x_{k-1})\|^2 \nonumber
\\&+\frac{4\beta^2L^2}{\mu^2} \mathbb{E}\|\widehat \nabla \Phi(x_{k-1})-\nabla \Phi(x_{k-1})\|^2,
\end{align}
where $(i)$ follows from Lemma 2.2 in \citealt{ghadimi2018approximation}. Using~\Cref{le:variancc} in~\cref{eq:midone} yields
\begin{align}\label{eq:seceq}
\mathbb{E}\|&y^D_{k-1}-y^*(x_k)\|^2  \nonumber
\\\leq& \left(2+  \frac{4\beta^2L^2}{\mu^2}  \Big( L+\frac{L^2}{\mu} + \frac{M\tau}{\mu}+\frac{LM\rho}{\mu^2}\Big)^2  \right)\mathbb{E}\|y^D_{k-1}-y^*(x_{k-1})\|^2+ \frac{4\beta^2L^2}{\mu^2} \mathbb{E}\|\nabla \Phi(x_{k-1})\|^2  \nonumber
\\&+  \frac{4\beta^2L^2}{\mu^2} \left( \frac{4L^2M^2}{\mu^2D_g} + \Big(\frac{8L^2}{\mu^2} + 2\Big) \frac{M^2}{D_f}+ \frac{16\eta^2  L^4M^2}{\mu^2} \frac{1}{B}+\frac{16 L^2M^2(1-\eta \mu)^{2Q}}{\mu^2}  \right).
\end{align}
Combining~\cref{eq:yjtt} and~\cref{eq:seceq} yields
\begin{align}\label{eq:enroll}
\mathbb{E}\|y_k^{D}&-y^*(x_k) \|^2  \nonumber
\\\leq&  \Big(\frac{L-\mu}{L+\mu}\Big)^{2D} \Big(2+  \frac{4\beta^2L^2}{\mu^2}  \Big( L+\frac{L^2}{\mu} + \frac{M\tau}{\mu}+\frac{LM\rho}{\mu^2}\Big)^2  \Big)\mathbb{E}\|y^D_{k-1}-y^*(x_{k-1})\|^2 \nonumber
\\&+ \Big(\frac{L-\mu}{L+\mu}\Big)^{2D} \frac{4\beta^2L^2}{\mu^2} \left( \frac{4L^2M^2}{\mu^2D_g} + \Big(\frac{8L^2}{\mu^2} + 2\Big) \frac{M^2}{D_f}+ \frac{16\eta^2  L^4M^2}{\mu^2} \frac{1}{B}+\frac{16 L^2M^2(1-\eta \mu)^{2Q}}{\mu^2}  \right)\nonumber
\\&+\frac{4\beta^2L^2}{\mu^2} \Big(\frac{L-\mu}{L+\mu}\Big)^{2D} \mathbb{E}\|\nabla \Phi(x_{k-1})\|^2  + \frac{\sigma^2}{L\mu S}. 
\end{align}
Based on the definitions of $\lambda,\omega,\Delta$ in~\cref{eq:defs}, we obtain from~\cref{eq:enroll} that 
\begin{align}\label{eq:readytote}
\mathbb{E}\|y_k^{D}-y^*(x_k) \|^2 \leq& \lambda \mathbb{E}\|y^D_{k-1}-y^*(x_{k-1})\|^2 + \omega\Delta +\frac{\sigma^2}{L\mu S}  +\omega \mathbb{E}\|\nabla \Phi(x_{k-1})\|^2. 
\end{align}
Telescoping~\cref{eq:readytote} over $k$ yields
\begin{align*}
\mathbb{E}\|y_k^{D}&-y^*(x_k) \|^2 \nonumber
\\\leq & \lambda^{k} \mathbb{E}\|y_0^D-y^*(x_0)\|^2 + \omega\sum_{j=0}^{k-1}\lambda^{k-1-j} \mathbb{E}\|\nabla \Phi(x_j)\|^2 + \frac{\omega\Delta +\frac{\sigma^2}{L\mu S}}{1-\lambda} \nonumber
\\\leq&\lambda^{k} \left( \left(\frac{L-\mu}{L+\mu}\right)^{2D}\|y_0-y^*(x_0)\|^2 + \frac{\sigma^2}{L\mu S}\right) + \omega\sum_{j=0}^{k-1}\lambda^{k-1-j} \mathbb{E}\|\nabla \Phi(x_j)\|^2 + \frac{\omega\Delta +\frac{\sigma^2}{L\mu S}}{1-\lambda}, 
\end{align*}
which completes the proof. 
\end{proof}

\subsection{Proof of~\Cref{th:nonconvex}}\label{mianshisimida}
In this subsection, we provide the proof for~\Cref{th:nonconvex}, based on the supporting lemmas we develop in~\Cref{se:supplemma}. 

Based on the smoothness of the function $\Phi(x)$ in~\Cref{le:lipphi}, we have 
\begin{align*}
\Phi(x_{k+1}) \leq & \Phi(x_k)  + \langle \nabla \Phi(x_k), x_{k+1}-x_k\rangle + \frac{L_\Phi}{2} \|x_{k+1}-x_k\|^2 \nonumber
\\\leq& \Phi(x_k)  - \beta \langle \nabla \Phi(x_k),\widehat \nabla \Phi(x_k)\rangle + \beta^2 L_\Phi \|\nabla\Phi(x_k)\|^2+\beta^2 L_\Phi\|\nabla\Phi(x_k)-\widehat \nabla\Phi(x_k)\|^2.
\end{align*}
For simplicity, let $\mathbb{E}_k = \mathbb{E}(\cdot\,| \,x_k,y_k^D)$. Note that we choose $\beta=\frac{1}{4L_\phi}$. Then, 
taking expectation over the above inequality, we have
\begin{align} \label{eq:jiayou}
\mathbb{E}\Phi(x_{k+1}) \leq &\mathbb{E}\Phi(x_k)  - \beta \mathbb{E}\langle \nabla \Phi(x_k),\mathbb{E}_k\widehat \nabla \Phi(x_k)\rangle + \beta^2 L_\Phi \mathbb{E}\|\nabla\Phi(x_k)\|^2 \nonumber
\\&+\beta^2 L_\Phi\mathbb{E}\|\nabla\Phi(x_k)-\widehat \nabla\Phi(x_k)\|^2 \nonumber
\\\overset{(i)}\leq& \mathbb{E}\Phi(x_k)  +\frac{\beta}{2}\mathbb{E}\|\mathbb{E}_k\widehat \nabla \Phi(x_k)-\nabla \Phi(x_k) \|^2 -\frac{\beta}{4} \mathbb{E}\|\nabla\Phi(x_k)\|^2+\frac{\beta}{4}\mathbb{E}\|\nabla\Phi(x_k)-\widehat \nabla\Phi(x_k)\|^2 \nonumber
\\\overset{(ii)}\leq& \mathbb{E}\Phi(x_k) -\frac{\beta}{4}\mathbb{E}\|\nabla\Phi(x_k)\|^2 +\frac{\beta L^2M^2(1-\eta \mu)^{2Q}}{\mu^2} 
\nonumber
\\&+\frac{\beta}{4}
\left(   \frac{4L^2M^2}{\mu^2D_g} + \Big(\frac{8L^2}{\mu^2} + 2\Big) \frac{M^2}{D_f}+ \frac{16\eta^2  L^4M^2}{\mu^2} \frac{1}{B}+\frac{16 L^2M^2(1-\eta \mu)^{2Q}}{\mu^2}\right) \nonumber
\\&+  \frac{5\beta}{4} \Big( L+\frac{L^2}{\mu} + \frac{M\tau}{\mu}+\frac{LM\rho}{\mu^2}\Big)^2\mathbb{E}\|y_k^D-y^*(x_k)\|^2 
\end{align}
where $(i)$ follows from Cauchy-Schwarz inequality, and $(ii)$ follows from \Cref{le:first_m} and~\Cref{le:variancc}. 
For simplicity,  let
\begin{align}\label{def:nu}
\nu= \frac{5}{4}\Big( L+\frac{L^2}{\mu} + \frac{M\tau}{\mu}+\frac{LM\rho}{\mu^2}\Big)^2.
\end{align} 
Then, applying~\Cref{tra_error} in~\cref{eq:jiayou} and using the definitions of $\omega,\Delta,\lambda$ in~\cref{eq:defs}, we have 
\begin{align}
\mathbb{E}\Phi(x_{k+1})\leq& \mathbb{E}\Phi(x_k) -\frac{\beta}{4} \mathbb{E}\|\nabla\Phi(x_k)\|^2 +\frac{\beta L^2M^2(1-\eta \mu)^{2Q}}{\mu^2} 
\nonumber
\\&+\frac{\beta}{4} \Delta + \beta\nu\lambda^{k} \left( \left(\frac{L-\mu}{L+\mu}\right)^{2D}\|y_0-y^*(x_0)\|^2 + \frac{\sigma^2}{L\mu S}\right) \nonumber
\\&+ \beta\nu\omega\sum_{j=0}^{k-1}\lambda^{k-1-j} \mathbb{E}\|\nabla \Phi(x_j)\|^2 + \frac{\beta\nu(\omega\Delta +\frac{\sigma^2}{L\mu S})}{1-\lambda}.\nonumber
\end{align}
Telescoping the above inequality over $k$ from $0$ to $K-1$ yields
\begin{align*}
\mathbb{E}\Phi(x_{K}) \leq \Phi(x_0) - &\frac{\beta}{4} \sum_{k=0}^{K-1}\mathbb{E}\|\nabla\Phi(x_k)\|^2 + \beta\nu\omega\sum_{k=1}^{K-1}\sum_{j=0}^{k-1}\lambda^{k-1-j} \mathbb{E}\|\nabla \Phi(x_j)\|^2 \nonumber
\\&+\frac{K\beta\Delta}{4} + \Big(\Big(\frac{L-\mu}{L+\mu}\Big)^{2D}\|y_0-y^*(x_0)\|^2 + \frac{\sigma^2}{L\mu S}\Big)\frac{\beta\nu}{1-\lambda} \nonumber
\\&+ \frac{K\beta L^2M^2(1-\eta \mu)^{2Q}}{\mu^2}  + \frac{K\beta\nu(\omega\Delta +\frac{\sigma^2}{L\mu S})}{1-\lambda},
\end{align*}
which, using the fact that $$\sum_{k=1}^{K-1}\sum_{j=0}^{k-1}\lambda^{k-1-j} \mathbb{E}\|\nabla \Phi(x_j)\|^2\leq \left(\sum_{k=0}^{K-1}\lambda^k\right)\sum_{k=0}^{{K-1}}\mathbb{E}\|\nabla\Phi(x_k)\|^2<\frac{1}{1-\lambda}\sum_{k=0}^{{K-1}}\mathbb{E}\|\nabla\Phi(x_k)\|^2,$$ yields
\begin{align}\label{eq:opsac}
 \Big(\frac{1}{4} -&\frac{\nu\omega}{1-\lambda}\Big) \frac{1}{K}\sum_{k=0}^{K-1}\mathbb{E}\|\nabla\Phi(x_k)\|^2 \nonumber
 \\\leq &\frac{\Phi(x_0)-\inf_x\Phi(x)}{\beta K}+\frac{\nu\big((\frac{L-\mu}{L+\mu})^{2D}\|y_0-y^*(x_0)\|^2 + \frac{\sigma^2}{L\mu S}\big)}{K(1-\lambda)}+\frac{\Delta}{4} +  \frac{ L^2M^2(1-\eta \mu)^{2Q}}{\mu^2}  \nonumber
 \\&+ \frac{\nu(\omega\Delta +\frac{\sigma^2}{L\mu S})}{1-\lambda}.
\end{align}
We choose the number $D$ of inner-loop steps as 
$$D\geq \max\bigg\{\frac{\log \big(12+  \frac{48\beta^2L^2}{\mu^2} ( L+\frac{L^2}{\mu} + \frac{M\tau}{\mu}+\frac{LM\rho}{\mu^2})^2\big)}{2\log (\frac{L+\mu}{L-\mu})},\frac{\log \big(\sqrt{\beta}(L+\frac{L^2}{\mu} + \frac{M\tau}{\mu}+\frac{LM\rho}{\mu^2})\big)}{\log (\frac{L+\mu}{L-\mu})}\bigg\}.$$ Then,  since $\beta=\frac{1}{4L_\Phi}$ and $D\geq \frac{\log \big(12+  \frac{48\beta^2L^2}{\mu^2} ( L+\frac{L^2}{\mu} + \frac{M\tau}{\mu}+\frac{LM\rho}{\mu^2})^2\big)}{2\log (\frac{L+\mu}{L-\mu})}$, we have $\lambda\leq \frac{1}{6}$, and  \cref{eq:opsac} is further simplified to 
\begin{align}\label{eq:ephi}
 \Big(\frac{1}{4} -&\frac{6}{5}\nu\omega\Big) \frac{1}{K}\sum_{k=0}^{K-1}\mathbb{E}\|\nabla\Phi(x_k)\|^2 \nonumber
 \\\leq &\frac{\Phi(x_0)-\inf_x\Phi(x)}{\beta K}+\frac{2\nu\big((\frac{L-\mu}{L+\mu})^{2D}\|y_0-y^*(x_0)\|^2 + \frac{\sigma^2}{L\mu S}\big)}{K}+\frac{ \Delta}{4} +  \frac{ L^2M^2(1-\eta \mu)^{2Q}}{\mu^2}  \nonumber
 \\&+ 2\nu\Big(\omega\Delta +\frac{\sigma^2}{L\mu S}\Big).
\end{align}
By the definitions of $\omega$ in~\cref{eq:defs} and $\nu$ in~\cref{def:nu} and $D\geq \frac{\log \big(12+  \frac{48\beta^2L^2}{\mu^2} ( L+\frac{L^2}{\mu} + \frac{M\tau}{\mu}+\frac{LM\rho}{\mu^2})^2\big)}{2\log (\frac{L+\mu}{L-\mu})}$, we have
\begin{align}\label{eq:opccsa}
\nu\omega = &\frac{5\beta^2L^2}{\mu^2} \Big(\frac{L-\mu}{L+\mu}\Big)^{2D}\Big( L+\frac{L^2}{\mu} + \frac{M\tau}{\mu}+\frac{LM\rho}{\mu^2}\Big)^2\nonumber
\\<& \frac{\frac{5\beta^2L^2}{\mu^2} \Big( L+\frac{L^2}{\mu} + \frac{M\tau}{\mu}+\frac{LM\rho}{\mu^2}\Big)^2}{12+  \frac{48\beta^2L^2}{\mu^2} ( L+\frac{L^2}{\mu} + \frac{M\tau}{\mu}+\frac{LM\rho}{\mu^2})^2} \leq \frac{5}{48}.
\end{align}
In addition, since $D>\frac{\log \big(\sqrt{\beta}\big(L+\frac{L^2}{\mu} + \frac{M\tau}{\mu}+\frac{LM\rho}{\mu^2}\big)\big)}{\log (\frac{L+\mu}{L-\mu})}$, we have
\begin{align}\label{eq:opccsa2}
 \nu\Big(\frac{L-\mu}{L+\mu}\Big)^{2D} =  \frac{5}{4}\Big(\frac{L-\mu}{L+\mu}\Big)^{2D}\Big( L+\frac{L^2}{\mu} + \frac{M\tau}{\mu}+\frac{LM\rho}{\mu^2}\Big)^2 <\frac{5}{4\beta}.
\end{align}
Substituting~\cref{eq:opccsa} and~\cref{eq:opccsa2} in~\cref{eq:ephi} yields 
\begin{align*}
\frac{1}{K}\sum_{k=0}^{K-1}\mathbb{E}\|\nabla\Phi(x_k)\|^2 \leq & \frac{8(\Phi(x_0)-\inf_x\Phi(x)+\frac{5}{2}\|y_0-y^*(x_0)\|^2) }{\beta K}+\Big(1+\frac{1}{K}\Big)\frac{16\nu\sigma^2}{L\mu S} \nonumber
\\&+\frac{11}{3} \Delta+  \frac{8L^2M^2}{\mu^2}(1-\eta \mu)^{2Q}, 
\end{align*}
which, in conjunction with~\cref{eq:defs} and \cref{def:nu}, yields~\cref{eq:main_nonconvex} in~\Cref{th:nonconvex}.  


Then, based on \cref{eq:main_nonconvex},  in order to achieve an $\epsilon$-accurate stationary point, i.e., $\mathbb{E}\|\nabla\Phi(\bar x)\|^2\leq \epsilon$ with $\bar x$ chosen from $x_0,...,x_{K-1}$ uniformly at random, it suffices to choose
\begin{align*}
K =&  \frac{32L_\Phi(\Phi(x_0)-\inf_x\Phi(x)+\frac{5}{2}\|y_0-y^*(x_0)\|^2) }{\epsilon}=\mathcal{O}\Big(\frac{\kappa^3}{\epsilon}\Big), D=\Theta(\kappa)\nonumber
\\Q = & \kappa\log \frac{\kappa^2}{\epsilon}, S = \mathcal{O}\Big(\frac{\kappa^5}{\epsilon}\Big),D_g =\mathcal{O}\left(\frac{\kappa^2}{\epsilon}\right), D_f =\mathcal{O}\left(\frac{\kappa^2}{\epsilon}\right), B =\mathcal{O}\left(\frac{\kappa^2}{\epsilon}\right).
\end{align*}
Note that the above choices of $Q$ and $B$ satisfy the condition that $B\geq \frac{1}{Q(1-\eta\mu)^{Q-1}}$ required in~\Cref{prop:hessian}. 

Then, the gradient complexity is given by $\mbox{\normalfont Gc}(F,\epsilon)=KD_f=\mathcal{O}(\kappa^5\epsilon^{-2}), \mbox{\normalfont Gc}(G,\epsilon)=KDS=\mathcal{O}(\kappa^9\epsilon^{-2}).$
In addition, the Jacobian- and Hessian-vector product complexities are given by $ \mbox{\normalfont JV}(G,\epsilon)=KD_g=\mathcal{O}(\kappa^5\epsilon^{-2})$ and 
\begin{align*}
\mbox{\normalfont HV}(G,\epsilon) = K \sum_{j=1}^Q BQ(1-\eta\mu)^{j-1}=\frac{KBQ}{\eta\mu} \leq\mathcal{O}\left( \frac{\kappa^6}{\epsilon^{2}}\log \frac{\kappa^2}{\epsilon} \right).
\end{align*}
Then, the proof is complete. 

\section{Proof of \Cref{th:meta_learning}}
To prove \Cref{th:meta_learning}, we first establish the following lemma to  characterize the estimation variance $\mathbb{E}_\gB\big\|\frac{\partial \gL_{\gD} (\phi_k,\widetilde w^D_k;\gB)}{\partial \phi_k} - \frac{\partial \gL_{\gD} (\phi_k,\widetilde w^D_k)}{\partial \phi_k} \big\|^2$, where $\widetilde w_k^{D}$ is the output of $D$ inner-loop steps of gradient descent at the $k^{th}$ outer loop. \begin{lemma}\label{le:bvarinace}
Suppose Assumptions~\ref{ass:lip} and \ref{high_lip} are satisfied and suppose each task loss $\gL_{\gS_i}(\phi,w_i)$ is $\mu$-strongly-convex w.r.t. $w_i$.  Then, we have
\begin{align*}
\mathbb{E}_\gB\Big\|\frac{\partial \gL_{\gD} (\phi_k, \widetilde w^D_k;\gB)}{\partial \phi_k} - \frac{\partial \gL_{\gD} (\phi_k, \widetilde w^D_k)}{\partial \phi_k} \Big\|^2 \leq \Big(1+\frac{L}{\mu}\Big)^2\frac{M^2}{|\gB|}.
\end{align*}
\end{lemma}
\begin{proof}
Let $\widetilde w_k^{D}=(w_{1,k}^D,...,w_{m,k}^D)$ be the output of $D$ inner-loop steps of gradient descent at the $k^{th}$ outer loop. 
Using \Cref{deter:gdform}, we have, for task $\gT_i$,
\begin{align}\label{maoxianssa}
\Big\|\frac{\partial \gL_{\gD_i}(\phi_k,w^D_{i,k})}{\partial \phi_k} \Big\|\leq & \|\nabla_\phi  \gL_{\gD_i}(\phi_k,w_{i,k}^D)\| \nonumber
\\+ \Big\|\alpha&\sum_{t=0}^{D-1}\nabla_\phi\nabla_{w_i}  \gL_{\gS_i}(\phi_k,w_{i,k}^{t})\prod_{j=t+1}^{D-1}(I-\alpha  \nabla^2_{w_i}  \gL_{\gS_i}(\phi_k,w_{i,k}^{j}))\nabla_{w_i}  \gL_{\gD_i}(\phi_k,w_{i.k}^D)\Big\| \nonumber
\\\overset{(i)}\leq& M + \alpha LM\sum_{t=0}^{D-1} (1-\alpha\mu)^{D-t-1} = M+\frac{LM}{\mu},
\end{align} 
where $(i)$ follows from Assumptions~\ref{ass:lip} and strong-convexity of $ \gL_{\gS_i}(\phi,\cdot)$. Then, using the definition of $ \gL_{\gD} (\phi,\widetilde w;\gB) = \frac{1}{|\gB|}\sum_{i\in\gB}\gL_{\gD_i}(\phi,w_i)$, we have 
\begin{align}
\mathbb{E}_\gB\Big\|\frac{\partial \gL_{\gD} (\phi_k, \widetilde w^D_k;\gB)}{\partial \phi_k} - \frac{\partial \gL_{\gD} (\phi_k, \widetilde w^D_k)}{\partial \phi_k} \Big\|^2 =& \frac{1}{|\gB|}\mathbb{E}_i\Big\|\frac{\partial \gL_{\gD_i}(\phi_k,w^D_{i,k})}{\partial \phi_k}  - \frac{\partial \gL_{\gD} (\phi_k, \widetilde w^D_k)}{\partial \phi_k} \Big\|^2 \nonumber
\\\overset{(i)}\leq&\frac{1}{|\gB|} \mathbb{E}_i \Big\|\frac{\partial \gL_{\gD_i}(\phi_k,w^D_{i,k})}{\partial \phi_k} \Big\|^2 \nonumber
\\\overset{(ii)}\leq& \Big(1+\frac{L}{\mu}\Big)^2\frac{M^2}{|\gB|}.
\end{align}
where $(i)$ follows from $\mathbb{E}_i \frac{\partial \gL_{\gD_i}(\phi_k,w^D_{i,k})}{\partial \phi_k}=\frac{\partial \gL_{\gD} (\phi_k, \widetilde w^D_k)}{\partial \phi_k}  $ and $(ii)$ follows from~\cref{maoxianssa}. Then, the proof is complete.
\end{proof}

\begin{proof}[\bf Proof of~\Cref{th:meta_learning}] 
Recall $\Phi(\phi):=\gL_{\gD} (\phi,\widetilde w^{*}(\phi))$ be the objective function, and let $\widehat \nabla\Phi(\phi_k) = \frac{\partial \gL_{\gD} (\phi_k, \widetilde w^D_k)}{\partial \phi_k} $.
Using an approach similar to \cref{eq:intimidern}, we have 
\begin{align}\label{eq:starteq}
\Phi(\phi_{k+1}) \leq & \Phi(\phi_k)  + \langle \nabla \Phi(\phi_k), \phi_{k+1}-\phi_k\rangle + \frac{L_\Phi}{2} \|\phi_{k+1}-\phi_k\|^2 \nonumber
\\\leq& \Phi(\phi_k)  - \beta \Big\langle \nabla \Phi(\phi_k), \frac{\partial \gL_{\gD} (\phi_k, \widetilde w^D_k;\gB)}{\partial \phi_k}\Big\rangle + \frac{\beta^2L_\Phi}{2} \Big\| \frac{\partial \gL_{\gD} (\phi_k, \widetilde w^D_k;\gB)}{\partial \phi_k} \Big\|^2.
\end{align}
Taking the expectation of~\cref{eq:starteq}  yields
\begin{align}\label{eq:uijks}
\mathbb{E}\Phi(\phi_{k+1}) \overset{(i)}\leq& \mathbb{E}\Phi(\phi_k)  - \beta \mathbb{E}\big\langle \nabla \Phi(\phi_k),\widehat \nabla\Phi(\phi_k)\big\rangle +\frac{\beta^2L_\Phi}{2}\mathbb{E} \|\widehat \nabla\Phi(\phi_k)\|^2\nonumber
\\&+ \frac{\beta^2L_\Phi}{2}\mathbb{E} \Big\| \widehat \nabla\Phi(\phi_k)-\frac{\partial \gL_{\gD} (\phi_k, \widetilde w^D_k;\gB)}{\partial \phi_k} \Big\|^2  \nonumber
\\\overset{(ii)}\leq &\mathbb{E}\Phi(\phi_k)  - \beta \mathbb{E}\big\langle \nabla \Phi(\phi_k),\widehat \nabla\Phi(\phi_k)\big\rangle +\frac{\beta^2L_\Phi}{2}\mathbb{E} \|\widehat \nabla\Phi(\phi_k)\|^2 +\frac{\beta^2L_\Phi}{2}  \Big(1+\frac{L}{\mu}\Big)^2\frac{M^2}{|\gB|}\nonumber
\\\leq &\mathbb{E}\Phi(\phi_k) -\Big(\frac{\beta}{2}-\beta^2 L_\Phi \Big)\mathbb{E}\| \nabla \Phi(\phi_k)\|^2 +\Big(\frac{\beta}{2}+\beta^2 L_\Phi\Big)\mathbb{E}\|\nabla\Phi(\phi_k)-\widehat \nabla\Phi(\phi_k)\|^2 \nonumber
\\&+\frac{\beta^2L_\Phi}{2}  \Big(1+\frac{L}{\mu}\Big)^2\frac{M^2}{|\gB|},
\end{align}
where $(i)$ follows from $\mathbb{E}_{\gB}\gL_{\gD} (\phi_k, \widetilde w^D_k;\gB)=\gL_{\gD} (\phi_k, \widetilde w^D_k)$ and $(ii)$ follows from~\Cref{le:bvarinace}. Using  \Cref{prop:partialG} in \cref{eq:uijks} and rearranging the terms, we have 
\begin{align*}
\frac{1}{K}\sum_{k=0}^{K-1}&\Big(\frac{1}{2}-\beta L_\Phi \Big)\mathbb{E}\| \nabla \Phi(\phi_k)\|^2   \nonumber
\\\leq& \frac{ \Phi(\phi_0)-\inf_\phi\Phi(\phi)}{\beta K} +3\Big(\frac{1}{2}+\beta L_\Phi\Big)\frac{L^2M^2(1-\alpha\mu)^{2D}}{\mu^2}+\frac{\beta L_\Phi}{2}  \Big(1+\frac{L}{\mu}\Big)^2\frac{M^2}{|\gB|} \nonumber
\\&+ 3\Delta\Big(\frac{1}{2}+\beta L_\Phi\Big)\Big( \frac{L^2(L+\mu)^2}{\mu^2} (1-\alpha\mu)^{D} +\frac{4M^2\left(  \tau\mu+ L\rho \right)^2}{\mu^4}(1-\alpha\mu)^{D-1} \Big),
\end{align*}
where $\Delta=\max_{k}\|\widetilde w^0_k-\widetilde w^*(\phi_k)\|^2<\infty$. Choose the same parameters $\beta,D$ as in~\Cref{th:determin}. Then, we have
\begin{align*}
\frac{1}{K}\sum_{k=0}^{K-1}\mathbb{E}\| \nabla \Phi(\phi_k)\|^2 \leq \frac{16 L_\Phi (\Phi(\phi_0)-\inf_\phi\Phi(\phi))}{K} + \frac{2\epsilon}{3}+ \Big(1+\frac{L}{\mu}\Big)^2\frac{M^2}{8|\gB|}.
\end{align*}
Then, the proof is complete. 
\end{proof}

\end{document}